\documentclass[11pt]{article}
\pdfoutput=1
\usepackage[margin=1in]{geometry} 

\usepackage{etoolbox}

\usepackage[round]{natbib}
\usepackage{array,amsmath,amsthm,amssymb,amsfonts,titling,bbm, titlesec, bm, physics, enumitem, accents, setspace, fancyhdr, natbib, geometry, pdflscape}
\usepackage{graphicx}
\usepackage{float}
\usepackage[font=footnotesize,labelfont=bf]{caption}
\usepackage{array}
\usepackage[title]{appendix}
\usepackage{afterpage}
\usepackage{listings}
\usepackage{prodint}
\usepackage{algorithm}
\usepackage{comment}
 \usepackage{booktabs}
\usepackage{multirow,booktabs}
\usepackage{subcaption}
\newlength{\continueindent}
\usepackage{etoolbox}
\usepackage{setspace}
\usepackage[dvipsnames]{xcolor}
\definecolor{Bleu}{RGB}{0,0,204}
\usepackage[hypertexnames=false]{hyperref}
\hypersetup{
colorlinks,
  citecolor=Bleu,
  linkcolor=Bleu,
  urlcolor=Bleu}
\newcommand{\Reg}{\mathrm{Reg}}
\newcommand{\Rem}{\mathrm{Rem}}

\newcolumntype{P}[1]{>{\centering\arraybackslash}p{#1}}

\usepackage{amsmath}
\usepackage{amssymb}
\usepackage{amsthm,thmtools}
\usepackage{enumitem}
\usepackage{amsfonts}
\RequirePackage{algorithm}
\RequirePackage{algorithmic}
\usepackage{bm,upgreek}
\usepackage{mathrsfs}  
\usepackage{calc}
\usepackage{subcaption}
\usepackage{authblk}

\theoremstyle{plain}
\theoremstyle{definition} 
\newtheorem{example}{Example} 

\newtheorem{theorem}{Theorem}
\newtheorem*{theorem*}{Theorem}    
\newtheorem{corollary}{Corollary}

\newtheorem{lemma}{Lemma}
\newtheorem{remark}{Remark}
\newtheorem*{remark*}{Remark}
\DeclareMathOperator{\E}{E}

\newcommand{\Var}{\operatorname{Var}}

\usepackage[utf8]{inputenc}
\usepackage{textcomp}

\DeclareMathOperator*{\argmin}{argmin}

\DeclareMathOperator*{\esssup}{ess\,sup}

\usepackage[capitalize,noabbrev]{cleveref}

\usepackage{authblk}

\usepackage[most]{tcolorbox}

\DeclareFontFamily{U}{jkpmia}{}
\DeclareFontShape{U}{jkpmia}{m}{it}{<->s*jkpmia}{}
\DeclareFontShape{U}{jkpmia}{bx}{it}{<->s*jkpbmia}{}
\DeclareMathAlphabet{\mathfrak}{U}{jkpmia}{m}{it}
\SetMathAlphabet{\mathfrak}{bold}{U}{jkpmia}{bx}{it}

\usepackage[hang,flushmargin]{footmisc}

\usepackage{tikz}
\usetikzlibrary{arrows.meta,positioning,fit}
\usetikzlibrary{calc}

\title{A Researcher’s Guide to Empirical Risk Minimization}

\author{Lars van der Laan}

\affil{Department of Statistics, University of Washington}

 \usepackage{setspace}
\newcommand{\slightspacing}{\setstretch{1.175}}

\begin{document}

\maketitle

\slightspacing

\begin{abstract}
This guide provides a reference for high-probability regret bounds in empirical risk minimization (ERM). The presentation is modular: we begin with intuition and general proof strategies, then state broadly applicable guarantees under high-level conditions and provide tools for verifying them for specific losses and function classes. We emphasize that many ERM rate derivations can be organized around a three-step recipe---a basic inequality, a uniform local concentration bound, and a fixed-point argument---which yields regret bounds in terms of a critical radius, defined via localized Rademacher complexity, under a mild Bernstein-type variance--risk condition \citep{bartlett2005local,wainwright2019high}. To make these bounds concrete, we upper bound the critical radius using local maximal inequalities and metric-entropy integrals, thereby recovering familiar rates for VC-subgraph, Sobolev/H\"older, and bounded-variation classes.

We also study ERM with nuisance components---including weighted ERM and Neyman-orthogonal losses---as they arise in causal inference, missing data, and domain adaptation. Following \citet{foster2023orthogonal}, we highlight that these problems often admit regret-transfer bounds linking regret under an estimated loss to population regret under the target loss. These bounds typically decompose the regret into (i) statistical error under the estimated loss and (ii) approximation error due to nuisance estimation. Under sample splitting or cross-fitting, the first term can be controlled using standard fixed-loss ERM regret bounds, while the second depends only on nuisance-estimation accuracy. As a novel contribution, we also treat the in-sample regime, in which the nuisances and the ERM are fit on the same data, deriving regret bounds and showing that fast oracle rates remain attainable under suitable smoothness and Donsker-type conditions. 
\end{abstract}

\begingroup
\setcounter{tocdepth}{2} 
\renewcommand{\contentsname}{Table of contents}
\tableofcontents
\endgroup
\bigskip

\section{Introduction}\label{sec:intro}

Empirical risk minimization (ERM) is a central tool in modern statistics and machine learning. In a typical problem, we choose \(\hat f_n\) by minimizing the empirical risk $$P_n \ell(\cdot, f) := \frac{1}{n}\sum_{i=1}^n \ell(Z_i,f)$$
over a function class \(\mathcal F\). The corresponding population risk is $R(f) := E_P[\ell(Z,f)],$
and a basic goal is to obtain guarantees on the \emph{regret} (or excess risk) $$R(\hat f_n)-R(f_0), \quad f_0 \in \arg\min_{f\in\mathcal F} R(f)$$
where $f_0$
is a population risk minimizer. Although the ERM principle is simple, deriving sharp regret rates in new settings can be technically involved. Proofs typically rely on tools from empirical process theory, such as uniform local maximal inequalities and concentration bounds \citep{van1996weak,geer2000empirical,gine2006concentration,wainwright2019high}. To avoid re-deriving regret bounds separately for each loss and function class, a common goal is to express them in terms of standard complexity measures, such as covering numbers and entropy integrals, for which a well-developed calculus is available.

This guide presents a high-probability regret analysis of ERM. We take a modular
approach that isolates proof patterns and complexity bounds with broad
applicability. The main message is that many ERM rate derivations can be
organized around a three-step recipe
\citep{geer2000empirical,abbeel2004apprenticeship,bartlett2005local,koltchinskii2011oracle,wainwright2019high}:
(i) a deterministic basic inequality; (ii) a uniform local concentration bound
for the empirical-process term; and (iii) a fixed-point argument that converts
the local bound into a regret rate. Under a mild Bernstein-type variance
condition on the loss, we give a general theorem that bounds ERM regret in
terms of the \emph{critical radius}---equivalently, the \emph{localized
Rademacher complexity}---of the loss-difference class
\[
\mathcal F_\ell := \{\ell(\cdot,f)-\ell(\cdot,f_0): f\in\mathcal F\}.
\]
This critical-radius viewpoint separates the \emph{statistical} task of
controlling the local complexity of \(\mathcal F_\ell\) from the
\emph{algebraic} task of solving the associated fixed-point inequality to
derive a regret rate.

To carry out step~(ii) and compute critical radii in practice, we develop maximal
inequalities that upper bound the local complexity of \(\mathcal F_\ell\), and hence its
critical radius, in terms of metric-entropy integrals
\citep{van2011local,lei2016local,wainwright2019high}. These results reduce many rate
calculations to bounding covering numbers for \(\mathcal F_\ell\). Covering numbers are
well understood for many classical function classes and behave stably under basic
operations, such as forming star-shaped hulls or applying Lipschitz transformations,
which streamlines these calculations. This approach recovers familiar critical radii
and regret rates for VC-subgraph classes, Sobolev/H\"older classes, and bounded-variation
classes.

A second focus of this guide is ERM with \emph{nuisance components}. Many modern estimators minimize empirical risks of the form \(P_n \ell_{\hat g}(\cdot,f)\), where the loss depends on a nuisance component \(\hat g\) estimated from the data, for example through inverse-probability weighting or pseudo-outcome regression. Following \citet{foster2023orthogonal},
nuisance components often do \emph{not} require a new regret analysis: one can
apply a standard regret bound to the ERM under the \emph{estimated} loss and
then use a \emph{regret-transfer} inequality to control the error induced by
using \(\hat g\) in place of \(g\), namely, the discrepancy between the population
risks (and minimizers) under \(g\) and \(\hat g\). Building on
\citet{foster2023orthogonal}, we develop regret-transfer bounds for weighted ERM
with estimated weights. Under sample splitting or cross-fitting, these bounds
combine with fixed-loss regret results to yield high-probability guarantees for
nuisance-dependent ERM. We also treat the \emph{in-sample} setting, in which the
nuisances are estimated on the same data used to compute \(\hat f_n\), and which
therefore requires additional uniform concentration arguments. For suitably smooth
optimization classes, such as H\"older or Sobolev classes, we show that oracle
rates are attainable under Donsker-type conditions on the nuisance class.

\paragraph{Scope and intent.}
These notes are not intended to be a comprehensive survey of ERM, nor do they aim to provide the sharpest possible regret bounds. Instead, they collect proof patterns and complexity bounds that I have found useful for deriving regret rates across a range of problems. This guide is intended as a technical reference on regret analysis for researchers with some familiarity with empirical process theory \citep{van1996weak,sen2018gentle} and empirical risk minimization \citep{wainwright2019high}. More specifically, the goal is to bridge two complementary literatures: the generality of localized Rademacher complexity arguments
\citep{bartlett2002rademacher,bousquet2003introduction,bartlett2005local,wainwright2019high}
and the practical convenience of uniform-entropy and maximal-inequality arguments
\citep{van1996weak,van2011local}. Similarly, Section~\ref{sec:nuis} is intended to complement the orthogonal statistical learning framework of \citet{foster2023orthogonal}. These are working notes, and comments and corrections are very welcome.

\paragraph{Useful resources.}
For introductory background on ERM, helpful lecture-note treatments include
\citet{bartlett2013stat210b}, \citet{guntuboyina2018stat210b}, and
\citet{sen2018gentle}, while book-length introductions include
\citet{shalev2014understanding}, \citet{mohri2018foundations},
\citet{wainwright2019high}, and \citet{koltchinskii2011oracle}. For
empirical-process treatments of ERM and M-estimation, see
\citet{van1996weak}, \citet{geer2000empirical}, and \citet{sen2018gentle}.
I found \citet{sen2018gentle} and \citet{wainwright2019high} especially
helpful as complementary resources.

 \paragraph{Organization.}
Section~2 introduces the ERM setup and the high-probability language used throughout, and recalls how regret bounds translate into error bounds under curvature. Section~3 presents a high-level blueprint and intuition for regret proofs. Section~4.1 develops a general regret theorem in terms of localized Rademacher complexity and critical radii, while Section~4.2 provides entropy-based bounds that make critical-radius calculations concrete for common classes. Section~5 introduces ERM with nuisance components, including weighted ERM and orthogonal losses with sample splitting and in-sample nuisance estimation. The appendix collects several local concentration inequalities for empirical processes and empirical inner products that support the main results and may be useful more broadly in ERM analyses.

\section{Preliminaries}
\label{sec:prelim}
\subsection{Problem setup and notation}

We observe i.i.d.\ data \(Z_1,\ldots,Z_n\in\mathcal Z\) drawn from an unknown
distribution \(P\), and write \(Z\sim P\) for an independent draw. The empirical
distribution of the sample is \(P_n := \frac{1}{n}\sum_{i=1}^n \delta_{Z_i}\),
where \(\delta_z\) denotes a point mass at \(z\in\mathcal Z\). For any measurable
function \(g:\mathcal Z\to\mathbb R\), we use the shorthand
\[
Pg := E\{g(Z)\},
\qquad
P_n g := \frac{1}{n}\sum_{i=1}^n g(Z_i).
\]
We write \(\|g\|:=\|g\|_{L^2(P)}:=\{Pg^2\}^{1/2}\) and
\(\|g\|_\infty:=\esssup_{z\in\mathcal Z}|g(z)|\), and use
\(\|g\|_n:=\|g\|_{L^2(P_n)}\) as shorthand for the empirical \(L^2\) norm. For simplicity, we let
\(\mathcal F\) denote a class of real-valued functions on \(\mathcal Z\). If
\(Z=(X,Y)\) and the functions in \(\mathcal F\) depend only on \(X\), we adopt
the convention that for \(z=(x,y)\), \(f(z):=f(x)\) for all \(f\in\mathcal F\). Throughout, \(\lesssim\) and \(\gtrsim\) denote inequalities that hold up to
universal constants, unless stated otherwise. We denote the minimum and maximum operators by
\(x \wedge y := \min\{x,y\}\) and \(x \vee y := \max\{x,y\}\).

Our goal is to learn a (constrained) population risk minimizer
\[
f_0 \in \arg\min_{f\in\mathcal{F}} R(f), \qquad
R(f) := E\{\ell(Z,f)\} = P \ell(\cdot,f),
\]
where \(\mathcal{F}\) is a class of candidate predictors and \(R(f)\) is the population risk induced by a loss function \(\ell(z,f)\in\mathbb{R}\). The loss function \(\ell\) specifies the notion of predictive error being optimized, while \(\mathcal{F}\) may encode structural assumptions such as sparsity, smoothness, or low dimensionality. In practice, \(f_0\) is often viewed as an approximation to a global minimizer
\(f_0^\star \in \arg\min_{f\in L^2(P)} R(f)\) (i.e., the minimizer over \(\mathcal{F}=L^2(P)\)).

Empirical risk minimization (ERM) is a basic principle for approximating population risk minimizers such as \(f_0\). The idea is simple: replace the unknown distribution \(P\) in \(R(f)=P\ell(\cdot,f)\) with its empirical counterpart \(P_n\), and then minimize the resulting objective. Concretely, an empirical risk minimizer is any solution
\[
\hat f_n \in \arg\min_{f\in\mathcal F} R_n(f), \qquad
R_n(f) := P_n\ell(\cdot,f)=\frac{1}{n}\sum_{i=1}^n \ell(Z_i,f),
\]
where \(R_n(f)\) is the empirical risk, computed over the sample $\{Z_i\}_{i=1}^n$.

Examples of loss functions and their purposes in machine learning can be found,
for example, in \citet{hastie2009elements}, \citet{james2013introduction},
\citet{shalev2014understanding}, \citet{wainwright2019high}, and
\citet{foster2023orthogonal}. For instance, in regression with \(Z=(X,Y)\), a
common choice is the squared loss \(\ell\{(x,y),f\}:=\{y-f(x)\}^2\). In this
case, the population risk is globally minimized over \(L^2(P_X)\) by the
regression function \(f_0^\star(x)=E_0[Y\mid X=x]\). The function class
\(\mathcal F\) specifies a working model for \(f_0^\star\) and need not contain
\(f_0^\star\) exactly. For example, one may assume that \(f_0^\star\) is
contained in, or well approximated by, the linear class
\(\mathcal F:=\{x\mapsto x^\top\beta:\ \beta\in\mathbb R^p\}\) (or a constrained
variant such as \(\mathcal F:=\{x\mapsto x^\top\beta:\ \|\beta\|_1\le B\}\) in
high-dimensional settings).

In binary classification with \(Y\in\{0,1\}\), one often models \(f(x)\in(0,1)\)
as a class probability and uses the cross-entropy (negative log-likelihood) loss
\[
\ell\{(x,y),f\}:= -y\log f(x) - (1-y)\log\{1-f(x)\},
\]
for which a global minimizer is \(f_0^\star(x)=P(Y=1\mid X=x)\). Different
parameterizations of \(P(Y=1\mid X=x)\) lead to different but equivalent loss
representations. For example, if \(\mathcal F\) parameterizes the log-odds (the
logit), so that \(f(x)\in\mathbb R\) and
\(p_f(x):=\{1+\exp(-f(x))\}^{-1}\) is the induced conditional probability, then
the cross-entropy loss can be written as the logistic regression loss
\[
\ell\{(x,y),f\}:=\log\bigl(1+\exp\{f(x)\}\bigr)-y f(x).
\]
The benefit of this parameterization is that \(f\) is unconstrained, so
\(\mathcal F\) can be taken to be a convex set (or a linear space), which often
simplifies optimization, for example via gradient-based methods. The squared and
logistic losses are special cases of a broad class of losses derived from
(working or quasi-)log-likelihoods for exponential family models, which also
includes the Poisson (log-linear) loss and many others \citep{hastie2009elements}.

\subsection{High-probability guarantees (PAC-style bounds)}

A central goal of this guide is to understand how quickly \(\hat f_n\) approaches \(f_0\) as \(n\) grows. We derive convergence rates under general conditions, focusing on two accuracy measures: the regret (or excess risk) \(R(\hat f_n)-R(f_0)\) and the \(L^2(P)\) estimation error \(\|\hat f_n-f_0\|\).

To state finite-sample guarantees, we will often use \textit{probably approximately correct} (\emph{PAC})--style bounds \citep{vapnik1999overview,geer2000empirical,shalev2014understanding,mohri2018foundations,wainwright2019high}. Concretely, a PAC-style bound specifies a function \(\varepsilon(n,\eta)\) of the sample size \(n\) and a failure probability \(\eta\in(0,1)\) such that, for every \(n\in\mathbb{N}\) and every \(\eta\in(0,1)\),
\[
R(\hat f_n)-R(f_0)\le \varepsilon(n,\eta) \qquad\text{with probability at least }1-\eta.
\]
The parameter $\eta$ is user-chosen: taking $\eta$ smaller makes the event more likely, at the cost of a larger (more conservative) bound.

For example, consider linear prediction with squared loss, where $\mathcal F:=\{x\mapsto x^\top\beta:\ \|\beta\|_1\le B\}$ and $f_0(x)=x^\top\beta_0$ for an $s$-sparse vector $\beta_0\in\mathbb R^p$ (i.e., $\|\beta_0\|_0=s$). Let $\hat f$ be the corresponding $\ell^1$-constrained ERM (equivalently, the lasso). Under standard design conditions, one has the high-probability regret bound \citep{wainwright2019high}
\begin{equation}
\label{eqn::highdim}
R(\hat f_n)-R(f_0)
\ \lesssim\
\frac{\sigma^2\, s\,\log(p/\eta)}{n}
\qquad\text{with probability at least }1-\eta,
\end{equation}
where $\sigma^2$ is a noise level (e.g., the variance of the regression errors) and $p$ is the ambient dimension.

PAC-style bounds provide more information than \(O_p\) (big-Oh in probability) statements because they make the dependence on the failure probability and sample size explicit.
The statement \(R(\hat f_n)-R(f_0)=O_p(\varepsilon_n)\) means that, for every \(\eta\in(0,1)\), there exist constants \(C_\eta<\infty\) and \(N_\eta\) such that, for all \(n\ge N_\eta\),
\[
R(\hat f_n)-R(f_0)\le C_\eta\,\varepsilon_n
\qquad\text{with probability at least }1-\eta.
\]
The constant \(C_\eta\) is left unspecified and may depend poorly on \(\eta\).
Moreover, this bound need not hold until the sample size exceeds \(N_\eta\), which may be arbitrarily large.
In contrast, a PAC-style bound gives an explicit finite-sample guarantee by providing an \(\eta\)-dependent error function \(\varepsilon(n,\eta)\) (and, if needed, an explicit \(N_{\eta}\)). For example, the high-dimensional linear regression bound in \eqref{eqn::highdim} yields the \(O_p\) statement
\(R(\hat f_n)-R(f_0)=O_p\bigl(\sigma^2 s\log p/n\bigr)\).
It also shows how the guarantee degrades as \(\eta\) decreases: \(\eta\) enters only through \(\log(1/\eta)\), which grows logarithmically as \(\eta\downarrow 0\). PAC-style bounds can also be used to prove almost-sure convergence. In particular, if
\[
\sum_{n=1}^\infty \Pr\!\left(R(\hat f_n)-R(f_0) > \varepsilon(n,\eta_n)\right) < \infty
\]
for some sequence \(\{\eta_n\}_{n\in\mathbb N}\), then the Borel--Cantelli lemma implies that \(R(\hat f_n)-R(f_0) \le \varepsilon(n,\eta_n)\) eventually almost surely.

Another advantage of PAC-style bounds over \(O_p\) bounds is that they facilitate simultaneous control of many events. For example, suppose we wish to control the errors of a collection of ERMs \(\{\hat f_{n,k}: k\in \mathcal{K}_n\}\), where the size \(K(n) := |\mathcal{K}_n|\) of the index set \(\mathcal{K}_n\) grows with \(n\). An \(O_p\) statement for each \(k\) is inherently pointwise: knowing \(R(\hat f_{n,k})-R(f_{0,k}) = O_p(\varepsilon_{n,k})\) for every fixed \(k\) does not, by itself, imply that all bounds hold simultaneously when \(K(n)\to\infty\). In contrast, if we have PAC-style bounds
\[
\Pr\left(R(\hat f_{n,k})-R(f_{0,k}) \le \varepsilon_k(n,\eta)\right)\ge 1-\eta
\quad\text{for each }k,
\]
then a union bound yields
\[
\Pr\left(\forall k\in \mathcal{K}_n:\ R(\hat f_{n,k})-R(f_{0,k}) \le \varepsilon_k\left(n,\frac{\eta}{K(n)}\right)\right)\ge 1-\eta.
\]
Thus, PAC-style bounds let us tune \(\eta\) to control the probability of \emph{any} failure across a growing family of estimators, typically at the cost of an additional \(\log K(n)\) term.

In regression, a common approach is to partition the covariate space \(\mathcal X\) into \(K(n)\) disjoint bins and perform empirical risk minimization within each bin \citep{wasserman2006all, takezawa2005introduction}. For example, if \(\mathcal F\) consists of constant predictors and we use the squared-error loss, then \(\hat f_{n,k}\) is the constant function equal to the empirical mean of the \(\approx n/K(n)\) responses in bin \(k\). In this case, the maximum squared error over \(K(n)\) independent empirical means, each based on \(n/K(n)\) observations, is typically of order \(K(n)\log K(n)/n\), a fact that can be used to derive regret bounds for histogram regression \citep{wasserman2006all}. Similar needs for uniform control arise beyond classical regression, for instance in sequential or iterative regression procedures in reinforcement learning, such as fitted value iteration and fitted \(Q\)-iteration \citep{munos2005error, munos2008finite, antos2007fitted, van2025inverse, van2025fitted, van2025stationary}. In these settings, regret is incurred at each iteration, and one must control the cumulative regret with high probability; see, for example, Lemma~3 and Theorem~2 of \citet{van2025fitted}.

\paragraph{Switching between different ways of writing high-probability bounds}
We state our results in PAC-style form, whereas some related works, such as \cite{bousquet2002bennett}, \cite{bartlett2005local}, and \cite{wainwright2019high}, express analogous results as exponential tail bounds. For example, PAC-style bounds often take the form
\[
R(\hat f_n)-R(f_0)\lesssim \varepsilon_n+\frac{\log(1/\eta)}{n}
\]
with probability at least \(1-\eta\), where \(\varepsilon_n\) is the leading rate term and \(\log(1/\eta)/n\) is the deviation term. Setting \(\Delta=\log(1/\eta)/n\) yields the equivalent formulation that, with probability at least \(1-e^{-n\Delta}\), the regret exceeds \(\varepsilon_n\) by at most \(\Delta\):
\[
R(\hat f_n)-R(f_0)\lesssim \varepsilon_n+\Delta.
\]
Thus, the PAC-style form immediately implies an exponential tail bound. Alternatively, following \cite{wainwright2019high}, we may choose \(\eta=e^{-t n\varepsilon_n}\) for some \(t>0\). Then \(\log(1/\eta)/n=t\varepsilon_n\), so the above display implies that, with probability at least \(1-e^{-t n\varepsilon_n}\),
\[
R(\hat f_n)-R(f_0)\lesssim (1+t)\varepsilon_n.
\]
In particular, \(n\varepsilon_n\) can be viewed as an effective sample size for learning \(f_0\), since it governs the exponent in the tail probability. More generally, if \(n\varepsilon_n=O(1)\), then for any fixed confidence level \(1-\alpha\in(0,1)\), we can choose \(t\) large enough that \(1-e^{-t n\varepsilon_n}\ge 1-\alpha\), and hence \(R(\hat f_n)-R(f_0)=O_p(\varepsilon_n)\). Integrating the tail bound over \(t\) yields the mean regret bound \(\mathbb{E}\{R(\hat f_n)-R(f_0)\}=O(\varepsilon_n+1/n)\) \citep{wainwright2019high}.

Each formulation parametrizes the same high-probability bound in a different way and may be more convenient in different applications. The first uses the failure probability \(\eta\), the second uses the additive deviation \(\Delta\), and the third uses the multiplicative inflation parameter \(t\). The first is often convenient when combining many events via a union bound. The second and third more directly quantify the size of the deviation, and the third makes it especially easy to read off \(O_p\) convergence rates.

\subsection{From regret to $L^2(P)$ error: curvature and strong convexity}
\label{sec::ell2}

It is common to derive high-probability bounds on the regret
\(R(\hat f_n)-R(f_0)\). In many problems, the regret is locally quadratic in
\(f-f_0\), so \(\{R(\hat f_n)-R(f_0)\}^{1/2}\) behaves like a distance between
\(\hat f_n\) and \(f_0\). Nevertheless, the quantity of primary interest is
often the \(L^2(P)\) estimation error \(\|\hat f_n-f_0\|\), which is typically
more interpretable and is often needed for downstream analyses. In this
section, we show how regret bounds can be translated into \(L^2(P)\) error
bounds.

A standard sufficient condition is a local quadratic-growth, or curvature,
inequality: there exists \(\kappa>0\) such that, for all \(f\in\mathcal F\),
\[
R(f)-R(f_0)\ge \kappa\,\|f-f_0\|^2.
\]
When this holds, any high-probability bound of the form
\(R(\hat f_n)-R(f_0)\lesssim \varepsilon\) immediately implies
\(\|\hat f_n-f_0\| \lesssim \kappa^{-1/2}\sqrt{\varepsilon}\) at the same
confidence level. This condition is natural when \(R\) is twice
differentiable with positive curvature near \(f_0\), since a second-order
Taylor expansion around \(f_0\) yields local quadratic growth in
\(\|f-f_0\|\).

We now formalize this intuition. A function class \(\mathcal F\) is
\emph{convex} if, for any \(f,g\in\mathcal F\) and any \(t\in[0,1]\), the
convex combination \(tf+(1-t)g\) also belongs to \(\mathcal F\). We say that
the risk function \(R:\mathcal F\to\mathbb R\) is \emph{strongly convex at
\(f_0\)} if there exists a curvature constant \(\kappa\in(0,\infty)\) such
that, for all \(f\in\mathcal F\),
\[
R(f)-R(f_0)-\dot R_{f_0}(f-f_0) \ge \kappa\,\|f-f_0\|^2,
\]
where \(\dot R_{f_0}(h):=\left.\frac{d}{dt}R(f_0+th)\right|_{t=0}\) denotes
the directional derivative of \(R\) at \(f_0\) in direction \(h\); see
Equation~14.42 of \citealp{wainwright2019high}.

\begin{lemma} \label{lemma::strongconvexity}
Suppose $\mathcal F$ is convex and $R:\mathcal F\to\mathbb R$ is strongly convex around $f_0$ with curvature constant $\kappa\in(0,\infty)$. Then, for all $f\in\mathcal F$,
\[
R(f)-R(f_0)\ge \kappa\,\|f-f_0\|^2.
\]
\end{lemma}
\begin{proof}
By strong convexity of $R$, we have
\[
R(f)-R(f_0)\ge \dot R_{f_0}(f-f_0) + \kappa\,\|f-f_0\|^2.
\]
Since $\mathcal F$ is convex, the line segment $f_0+t(f-f_0)$ lies in $\mathcal F$ for all $t\in[0,1]$. Because $f_0$ minimizes $R$ over $\mathcal F$, we have
\[
R\{f_0+t(f-f_0)\}-R(f_0)\ge 0
\qquad\text{for all }t\in[0,1].
\]
Dividing by $t>0$ and taking the limit as $t\downarrow 0$ yields
\[
0 \le \lim_{t\downarrow 0}\frac{R\{f_0+t(f-f_0)\}-R(f_0)}{t}
= \dot R_{f_0}(f-f_0).
\]
Thus $\dot R_{f_0}(f-f_0)\ge 0$, and dropping this term from the strong convexity bound gives $R(f)-R(f_0)\ge \kappa\,\|f-f_0\|^2,$
as claimed.
\end{proof}

We illustrate Lemma~\ref{lemma::strongconvexity} by verifying the curvature condition explicitly for the least-squares loss using properties of projections onto convex sets.

\begin{example}[Strong convexity for least squares via projection]
Let \(\ell(z,f):=\{y-f(x)\}^2\) and \(R(f):=E\{(Y-f(X))^2\}\). Let \(\mathcal F\subset L^2(P_X)\) be convex, and let
\(f_0\in\arg\min_{f\in\mathcal F} R(f)\). Write \(f^\star(x):=E(Y\mid X=x)\) for the (unconstrained) risk minimizer. Conditioning on \(X\) yields
\[
R(f)
=E\Bigl[E\bigl[(Y-f(X))^2\mid X\bigr]\Bigr]
=E\bigl[\Var(Y\mid X)+\{f^\star(X)-f(X)\}^2\bigr].
\]
Therefore, for any \(f\in\mathcal F\),
\begin{align*}
R(f)-R(f_0)
&=
E\bigl[\{f^\star(X)-f(X)\}^2-\{f^\star(X)-f_0(X)\}^2\bigr]\\
&=
\|f-f_0\|^2
-2E\bigl[(f-f_0)(f^\star-f_0)\bigr].
\end{align*}
Since \(f_0\) is the least-squares projection of \(f^\star\) onto the convex set \(\mathcal F\), the first-order optimality condition gives
\(E[(f-f_0)(f^\star-f_0)]\le 0\) for all \(f\in\mathcal F\). Hence
\[
R(f)-R(f_0)\ \ge\ \|f-f_0\|^2,\qquad f\in\mathcal F.
\]
If \(\mathcal F\) is a linear subspace, then \(E[(f-f_0)(f^\star-f_0)]=0\), and the inequality becomes an equality:
\(R(f)-R(f_0)=\|f-f_0\|^2\).
\end{example}

In the remainder of the guide, we focus primarily on bounding the regret, since the lemma above lets us translate regret bounds directly into \(L^2(P)\) error bounds under mild curvature conditions.

\section{Proof blueprint for regret bound}

\label{sec:overview}

\subsection{Step 1: the basic inequality (deterministic regret bound)}

The starting point of any ERM analysis is a deterministic upper bound on the
regret \(R(\hat f_n)-R(f_0)\). The following theorem establishes the so-called
``basic inequality,'' which provides such a bound.

\begin{theorem}[Basic inequality for constrained ERM]
\label{theorem::basic}
For any empirical risk minimizer $\hat f_n\in\arg\min_{f\in\mathcal F} R_n(f)$ and any population risk minimizer $f_0\in\arg\min_{f\in\mathcal F} R(f)$,
\[
R(\hat f_n)-R(f_0) \le (P_n-P)\bigl\{\ell(\cdot,f_0)-\ell(\cdot,\hat f_n)\bigr\}.
\]
\end{theorem}
\begin{proof}
Since $\hat f_n$ minimizes $R_n$, we have $R_n(\hat f_n)\le R_n(f_0)$, i.e.,
$0\le R_n(f_0)-R_n(\hat f_n)$. Add and subtract $R(f_0)$ and $R(\hat f_n)$ to obtain
\begin{align*}
R(\hat f_n)-R(f_0)
&= \bigl\{R(\hat f_n)-R_n(\hat f_n)\bigr\} + \bigl\{R_n(\hat f_n)-R_n(f_0)\bigr\} + \bigl\{R_n(f_0)-R(f_0)\bigr\} \\
&\le \bigl\{R(\hat f_n)-R_n(\hat f_n)\bigr\} + \bigl\{R_n(f_0)-R(f_0)\bigr\}\\
&= \bigl\{R_n(f_0)-R(f_0)\bigr\}-\bigl\{R_n(\hat f_n)-R(\hat f_n)\bigr\}\\
&= (P_n-P)\bigl\{\ell(\cdot,f_0)-\ell(\cdot,\hat f_n)\bigr\}.
\end{align*}
\end{proof}

The basic inequality in Theorem~\ref{theorem::basic} reduces the regret
analysis to controlling the empirical-process fluctuation
\((P_n-P)\{\ell(\cdot,f_0)-\ell(\cdot,\hat f_n)\}\). The remainder of the
argument is to bound this term and translate the resulting bound into a rate
for \(R(\hat f_n)-R(f_0)\).

\subsection{Warm-up: turning the basic inequality into rates}

This section illustrates, at a heuristic level, how the basic inequality yields
regret rates. The goal is to build intuition for the empirical-process
fluctuation term
\((P_n-P)\{\ell(\cdot,f_0)-\ell(\cdot,\hat f_n)\}\) and to preview the main
proof techniques used in ERM analyses.

It is helpful to distinguish the two sources of randomness in the empirical-process
fluctuation
\((P_n-P)\{\ell(\cdot,f_0)-\ell(\cdot,\hat f_n)\}\).
First, \(P_n\) averages over an i.i.d.\ sample. Second, the function being
averaged is itself random because it depends on the data through the ERM
\(\hat f_n\). Any ERM analysis must account for both effects. The first source
alone can be handled using standard concentration inequalities for i.i.d.\
averages of the form \((P_n-P)h\) with \(h\) fixed. The second source is the
main complication: because \(\hat f_n\) is data dependent, this fixed-function
viewpoint is no longer sufficient, and one instead needs \emph{uniform}
control of the empirical process
\[
\bigl\{(P_n-P)\{\ell(\cdot,f_0)-\ell(\cdot,f)\} : f\in\mathcal F\bigr\}
\]
over the class \(\mathcal F\). Indeed,
\((P_n-P)\{\ell(\cdot,f_0)-\ell(\cdot,\hat f_n)\}\) is obtained by evaluating
this process at the data-dependent index \(f=\hat f_n\). We therefore begin
with a warm-up in the fixed-function setting, isolating source~(i), to
illustrate how regret bounds are derived. Handling source~(ii) requires more
sophisticated uniform, typically local, concentration tools, but once such
bounds are available, the overall ERM analysis follows the same proof template
with only minor modifications.

As a thought experiment, suppose we have a deterministic sequence
\(\{f_n\}_{n=1}^\infty \subset \mathcal F\) that happens to satisfy the basic
inequality
\begin{equation}
\label{eqn::basicdeterm}
R(f_n)-R(f_0)\le (P_n-P)\{\ell(\cdot,f_0)-\ell(\cdot,f_n)\}.
\end{equation}
This is generally not possible in practice, since \eqref{eqn::basicdeterm} is
typically obtained only for a data-dependent choice, such as \(f_n=\hat f_n\).
In the deterministic case, however, the fluctuation term reduces to a sample
mean of independent, mean-zero random variables. Indeed, letting
\[
X_{n,i}
:=
\bigl\{\ell(Z_i,f_0)-\ell(Z_i,f_n)\bigr\}
-
P\{\ell(\cdot,f_0)-\ell(\cdot,f_n)\},
\]
we have \((P_n-P)\{\ell(\cdot,f_0)-\ell(\cdot,f_n)\}=n^{-1}\sum_{i=1}^n X_{n,i}\).
Under mild tail conditions, a central limit theorem for triangular arrays
suggests that, for large \(n\), this quantity behaves like a normal random
variable with mean \(0\) and standard deviation \(\sigma_{f_n}/\sqrt{n}\), where
\(\sigma_{f_n}^2:=\Var\bigl(\ell(Z,f_0)-\ell(Z,f_n)\bigr)\). In particular,
\[
(P_n-P)\{\ell(\cdot,f_0)-\ell(\cdot,f_n)\}=O_p(\sigma_{f_n}/\sqrt{n}).
\]
Combining this with \eqref{eqn::basicdeterm} yields
\(R(f_n)-R(f_0)=O_p(\sigma_{f_n}/\sqrt{n})\), and hence the \emph{slow rate}
\(R(f_n)-R(f_0)=O_p(n^{-1/2})\) provided that \(\sigma_{f_n} = O(1)\).

As the following lemma shows, one can in fact obtain an exponential-tail, PAC-style bound for the empirical-process fluctuation, and hence for the regret; see \cite{sridharan2002gentle} for an introduction to concentration inequalities.

\begin{lemma}[Bernstein-type point-wise local concentration bound]
\label{lem:fixed_fn_bernstein}
Fix \(n\) and let \(f\in\mathcal F\) be deterministic. Define
\(\sigma_f^2:=\Var\{\ell(Z,f_0)-\ell(Z,f)\}\) and
\(M_f:=\|\ell(\cdot,f_0)-\ell(\cdot,f)\|_{\infty}\).
Then there exists a universal constant \(C>0\) such that, for all
\(\eta\in(0,1)\),
\[
\left|(P_n-P)\{\ell(\cdot,f_0)-\ell(\cdot,f)\}\right|
\le
C\left\{
\sigma_f\sqrt{\frac{\log(2/\eta)}{n}}
+
M_f\frac{\log(2/\eta)}{n}
\right\}
\qquad\text{with probability at least }1-\eta.
\]
\end{lemma}

If \(b_{f_n}\lesssim 1\), then Lemma~\ref{lem:fixed_fn_bernstein} combined with
\eqref{eqn::basicdeterm} yields the PAC-style guarantee that, whenever
\(n\gtrsim \log(1/\eta)\),
\begin{equation}
\label{eqn::slowrate}
R(f_n)-R(f_0)
\ \lesssim\
\sqrt{\frac{\log(1/\eta)}{n}}
\qquad\text{with probability at least }1-\eta.
\end{equation}
This \emph{slow rate} is often loose; for example, by Section~\ref{sec::ell2}, it typically implies only
\(\|f_n-f_0\|=O_p(n^{-1/4})\).
The looseness stems from the fact that \eqref{eqn::slowrate} ignores the \emph{local} variance term in Bernstein's inequality,
\(\sigma_{f_n}^2=\Var\{\ell(Z,f_0)-\ell(Z,f_n)\}\),
which can shrink when \(f_n\) is close to \(f_0\).
Treating \(\sigma_{f_n}\) as a constant forces the leading term to scale like \(n^{-1/2}\), regardless of how close \(f_n\) is to \(f_0\).
We now show how exploiting the localization of \(f_n\) around \(f_0\) can yield substantially faster behavior in our deterministic thought experiment, namely \(R(f_n)-R(f_0)=O_p(n^{-1})\) and \(\|f_n-f_0\|=O_p(n^{-1/2})\).

To formalize this \emph{localization}, a standard approach is to impose a
Bernstein condition, which relates the variance of the loss difference
\(\ell(\cdot,f)-\ell(\cdot,f_0)\) to the corresponding regret. Specifically, assume there exists \(c_{\mathrm{Bern}}>0\) such that, for all \(f\in\mathcal F\),
\begin{equation}
\label{eq:bernstein_condition}
\Var\bigl(\ell(Z,f)-\ell(Z,f_0)\bigr)\le c_{\mathrm{Bern}}\{R(f)-R(f_0)\}.
\end{equation}
In other words, small regret implies small variance of the loss difference \(\ell(Z,f)-\ell(Z,f_0)\). This condition is mild and holds in many standard settings—for example, when
\(\mathcal F\) is convex and the risk is strongly convex. In particular, it
follows by combining a Lipschitz-type variance bound with local quadratic
curvature of the risk, as in Section~\ref{sec::ell2}.\footnote{For convex risks
over convex classes, strong convexity---and hence the Bernstein condition---can
also be enforced via Tikhonov regularization; see
Appendix~\ref{appendix::regularridge}.}

\begin{lemma}[Sufficient conditions for the Bernstein condition]
\label{lemma::suffbern}
Suppose that, for all \(f\in\mathcal{F}\), \(\Var\bigl(\ell(Z,f)-\ell(Z,f_0)\bigr)\le L\|f-f_0\|^2\) and \(\kappa \|f-f_0\|^2\le \{R(f)-R(f_0)\}\). Then, for all \(f\in\mathcal{F}\), \(\Var\bigl(\ell(Z,f)-\ell(Z,f_0)\bigr)\le c_{\mathrm{Bern}}   \{R(f)-R(f_0)\}\) with $c_{\mathrm{Bern}} := L \kappa^{-1}$.
\end{lemma}

We now illustrate how the Bernstein condition can yield faster rates. Under the
Bernstein condition, the variance
\(\sigma_{f_n}^2:=\Var\{\ell(Z,f_0)-\ell(Z,f_n)\}\) shrinks as the regret shrinks.
Since Lemma~\ref{lem:fixed_fn_bernstein} concentrates
\((P_n-P)\{\ell(\cdot,f_0)-\ell(\cdot,f_n)\}\) at the scale \(\sigma_{f_n}/\sqrt{n}\),
combining it with the basic inequality yields
\begin{equation}
\label{eqn::highprobboundfixed}
R(f_n)-R(f_0)=O_p\left(\frac{\sigma_{f_n}}{\sqrt{n}}\right).
\end{equation}
A crude bound treats \(\sigma_{f_n}\) as constant and gives the slow rate
\(R(f_n)-R(f_0)=O_p(n^{-1/2})\). The Bernstein condition improves this by
linking variance to regret: \(\sigma_{f_n}^2\lesssim R(f_n)-R(f_0)\), so
\eqref{eqn::highprobboundfixed} becomes
\[
R(f_n)-R(f_0)
=
O_p\left(\frac{\{R(f_n)-R(f_0)\}^{1/2}}{\sqrt{n}}\right).
\]
Rearranging yields \(\{R(f_n)-R(f_0)\}^{1/2}=O_p(n^{-1/2})\), and therefore
\(R(f_n)-R(f_0)=O_p(n^{-1})\).
This \(n^{-1}\) rate is generally unattainable in nonparametric classes, because
uniform control of the empirical process introduces additional complexity terms.

\paragraph{From deterministic to random functions via uniform local concentration inequalities.}
Much of the intuition from the deterministic thought experiment carries over to ERM, where the random minimizer \(\hat f_n\) satisfies the basic inequality.
The main complication is that Lemma~\ref{lem:fixed_fn_bernstein} controls the empirical-process fluctuation only for a \emph{fixed} (deterministic) function, and therefore does not apply directly to the data-dependent choice \(\hat f_n\).
Extending the argument requires high-probability bounds that hold \emph{uniformly} over \(f\in\mathcal F\).
As we will see, such uniformity is typically costly: outside low-dimensional parametric models, it rules out the optimistic \(O_p(n^{-1})\) regret rate, and slower rates are the norm.

A simple, though loose, way to handle the randomness in \(\hat f_n\) is to upper bound the empirical-process fluctuation \((P_n-P)\{\ell(\cdot,f_0)-\ell(\cdot,\hat f_n)\}\) by the \emph{global} supremum
\[
\sup_{f \in \mathcal{F}} (P_n - P)\{\ell(\cdot,f_0)-\ell(\cdot,f)\}.
\]
This removes the data dependence of \(\hat f_n\), at the cost of controlling the worst-case fluctuation of the empirical process over the entire class.
When \(\mathcal{F}\) is not too large (for example, when it is Donsker), one can often apply a uniform law of large numbers or central limit theorem \citep{van1996weak} to show that this supremum is \(O_p(n^{-1/2})\), and hence
\[
(P_n-P)\{\ell(\cdot,f_0)-\ell(\cdot,\hat f_n)\}=O_p(n^{-1/2}).
\]
This recovers the ``slow rate'' \(R(\hat f_n)-R(f_0)=O_p(n^{-1/2})\) that we argued heuristically in the deterministic case.
However, as in the deterministic setting, this approach does not exploit the localization of \(\hat f_n\) around \(f_0\), since it replaces \(\hat f_n\) by a supremum over all \(f \in \mathcal{F}\).

The next section addresses this limitation using \emph{uniform local concentration} inequalities.
Unlike the global supremum bound, these results control
\((P_n-P)\{\ell(\cdot,f_0)-\ell(\cdot,f)\}\) \emph{simultaneously} over \(f\in\mathcal F\), but with a
high-probability bound that adapts to \(f\) through the local scale parameter \(\sigma_f = \{\Var\{\ell(Z,f_0)-\ell(Z,f)\}\}^{1/2}\).
For intuition, one can derive these bounds by controlling self-normalized, ratio-type empirical process
suprema, as in \citet{gine2006concentration}.
Concretely, one seeks a bound of the form
\[
\sup_{f\in\mathcal F}\frac{(P_n-P)\{\ell(\cdot,f_0)-\ell(\cdot,f)\}}{\sigma_f \vee \delta_{n,\eta}}
\ \lesssim\ \delta_{n,\eta},
\qquad\text{with probability at least }1-\eta,
\]
where \(\delta_{n,\eta}\) is a deterministic complexity term (often characterized via a critical
radius) that may decay more slowly than \(n^{-1/2}\) in nonparametric settings.
Such a bound implies that, with probability at least \(1-\eta\),
\begin{equation}
    (P_n-P)\{\ell(\cdot,f_0)-\ell(\cdot,\hat f_n)\}
\ \lesssim\ (\sigma_{\hat f_n} \vee \delta_{n,\eta})\,\delta_{n,\eta}
\ \lesssim\ \sigma_{\hat f_n}\,\delta_{n,\eta} + \delta_{n,\eta}^2. \label{example::rateineq}
\end{equation}
Consequently, arguing as in \eqref{eqn::highprobboundfixed}, one can combine this inequality with the basic inequality to obtain the regret bound
\(R(\hat f_n)-R(f_0)=O_p\bigl(\sigma_{\hat f_n}\,\delta_{n,\eta} + \delta_{n,\eta}^2\bigr)\).
If, in addition, a Bernstein-type variance--risk condition holds so that
\(\sigma_{\hat f_n}^2 \lesssim R(\hat f_n)-R(f_0)\),
then similar algebra to the deterministic case yields
\(R(\hat f_n)-R(f_0)=O_p\bigl(\delta_{n,\eta}^2\bigr)\).

\subsection{A three-step template for ERM rates}
\label{sec:template}

In this section, we outline a high-level template for deriving ERM rates. Many
existing ERM analyses follow this blueprint, and later sections instantiate
each step in concrete settings. The presentation here is intentionally
schematic; readers who prefer to start with a fully specified result may skip
ahead to Section~\ref{sec:genregret}.

The template consists of three steps. First, derive a deterministic basic
inequality, as we already did in Theorem~\ref{theorem::basic}. Second, obtain a high probability bound for the empirical-process fluctuation \((P_n-P)\{\ell(\cdot,f_0)-\ell(\cdot,\hat f_n)\}\) appearing in the basic
inequality. To obtain fast rates, this bound should ideally reflect the
\emph{localization} of \(\hat f_n\) around \(f_0\), as illustrated in the
previous section. Third, combine the basic inequality with
the high-probability bound to obtain regret bounds for the ERM.

\noindent \textbf{Step 1. Derive a deterministic inequality.}
Theorem~\ref{theorem::basic} establishes the basic inequality
\begin{equation}
R(\hat f_n)-R(f_0)
\le
(P_n-P)\bigl\{\ell(\cdot,f_0)-\ell(\cdot,\hat f_n)\bigr\}.
\label{eqn::basicgen}
\end{equation}
This is the key inequality for constrained ERM. In other settings, different
basic inequalities are needed; for example, in regularized ERM (e.g., ridge- or
lasso-type penalties) the optimization problem includes a penalty term, and the
corresponding basic inequality features additional terms involving the
regularizer \citep{wainwright2019high}. Likewise, when the loss involves
estimated nuisance components, further empirical-process and remainder terms may
appear (see Section~\ref{sec:insample} for one such inequality). In general, the
goal is to upper bound the regret by a finite sum of empirical-process-type
remainders (to be controlled via high-probability bounds) and population
quantities depending on \(\hat f_n\) and \(f_0\) (e.g., terms that can be bounded
by a fractional power of the regret \(\{R(\hat f_n)-R(f_0)\}^{\gamma}\)).

\noindent \paragraph{Step 2. Control randomness through a high-probability bound.}
The next step is to control the empirical-process fluctuation
\((P_n-P)\bigl\{\ell(\cdot,f_0)-\ell(\cdot,\hat f_n)\bigr\}\).
Concretely, we assume a Bernstein-type uniform local concentration bound on the empirical process
\(\{(P_n-P)\{\ell(\cdot,f_0)-\ell(\cdot,f)\} : f \in \mathcal F\}\) of the following form:
\begin{enumerate}[label=\bfseries E\arabic*), ref={E\arabic*}, series=ex]
  \item \label{cond::highprob} (Uniform local concentration bound.)
  There exists a function \(\varepsilon(n,\sigma,\eta)\), nondecreasing in
  \(\sigma\), such that for all \(f\in\mathcal F\) with
  \(\sigma_f^2:=\Var\bigl(\ell(Z,f_0)-\ell(Z,f)\bigr)<\infty\) and all
  \(\eta\in(0,1)\),
  \[
  (P_n-P)\{\ell(\cdot,f_0)-\ell(\cdot,f)\}\le \varepsilon(n,\sigma_f,\eta)
  \qquad\text{with probability at least }1-\eta.
  \]
\end{enumerate}
This bound ensures that, uniformly over \(f\in\mathcal F\), we can control the empirical-process fluctuation
\((P_n-P)\{\ell(\cdot,f_0)-\ell(\cdot,f)\}\),
and that this control is \emph{local} around \(f_0\) in the sense that it adapts to the standard deviation parameter \(\sigma_f\).
Such uniform concentration bounds are widely available (e.g., Theorem~8 of \citealp{bousquet2003introduction}; Theorem~3.3 of
\citealp{bartlett2005local}; Theorem~2.1 of \citealp{gine2006concentration};
Theorem~14.20(b) of \citealp{wainwright2019high}; Appendix~K of
\citealp{foster2023orthogonal}).
For example, in Section~\ref{sec:genregret}, Theorem~\ref{theorem::localmaxloss} shows that \(\varepsilon(n,\sigma_f,\eta)\)
can be taken proportional to \(\delta_{n,\eta}\sigma_f+\delta_{n,\eta}^2\),
where \(\delta_{n,\eta}\asymp \delta_n+\sqrt{\log(1/\eta)/n}\) and \(\delta_n\)
is a critical radius determined by the complexity of the loss-difference class
\(\mathcal F_\ell:=\{\ell(\cdot,f_0)-\ell(\cdot,f): f\in\mathcal F\}\).

Since the bound in \ref{cond::highprob} holds for all \(f\in\mathcal F\), it also holds for the (random) choice \(f=\hat f_n\).
In particular,
\begin{equation}
\label{eqn::highprobboundgeneric}
(P_n-P)\bigl\{\ell(\cdot,f_0)-\ell(\cdot,\hat f_n)\bigr\}
\le \varepsilon (n,\sigma_{\hat f_n},\eta)
\qquad\text{with probability at least }1-\eta.
\end{equation}

 \noindent \paragraph{Step 3. Bound regret and extract a rate.} Denote the regret by \(\hat d_n^2 := R(\hat f_n)-R(f_0)\). Combining the basic inequality in \eqref{eqn::basicgen} and our high-probability bound in \eqref{eqn::highprobboundgeneric}, we find that
\[
\hat d_n^2  = R(\hat f_n)-R(f_0)  \le \varepsilon (n,\sigma_{\hat f_n},\eta)
\qquad\text{with probability at least }1-\eta.
\] 
Next, as in  the previous section, we exploit the localization of \(\hat f_n\) around \(f_0\) by assuming the following Bernstein condition: there exists \(c_{\mathrm{Bern}}>0\) such that, for all \(f\in\mathcal F\),
\begin{equation*}
\Var\bigl(\ell(Z,f)-\ell(Z,f_0)\bigr)\le c_{\mathrm{Bern}}\{R(f)-R(f_0)\}.
\end{equation*}
We have \(\sigma_{\hat f_n}^2 \le c_{\mathrm{Bern}} \hat{d}_n^2\), and hence
\(\varepsilon(n,\sigma_{\hat f_n},\eta) \le \varepsilon(n, c_{\mathrm{Bern}} \hat{d}_n,\eta)\), since
\(\sigma \mapsto \varepsilon(n,\sigma,\eta)\) is nondecreasing.
It then follows that
\begin{equation}
\label{eqn::fixedpointgen}
\hat d_n^2 \ \le\ \varepsilon(n, c_{\mathrm{Bern}}\hat d_n,\eta)
\qquad\text{with probability at least }1-\eta.
\end{equation}
This is a \emph{fixed-point} (self-bounding) inequality: \(\hat d_n\) appears on
both sides, and the right-hand side typically grows sub-quadratically in
\(\hat d_n\). Consequently, the inequality forces \(\hat d_n\) to be small: if
\(\hat d_n\) were too large, the left-hand side would dominate the right-hand
side, contradicting \eqref{eqn::fixedpointgen}.

We obtain a high-probability bound for \(\hat d_n\) by comparing it to the
largest deterministic \(\delta(n,\eta)\) that satisfies the same inequality.
Define
\[
\delta(n,\eta)
:=
\sup\Bigl\{\delta\ge 0:\ \delta^2 \le \varepsilon\left(n, c_{\mathrm{Bern}} \delta,\eta\right)\Bigr\}.
\]
On any event where \(\hat d_n^2 \le \varepsilon(n, c_{\mathrm{Bern}}\hat d_n,\eta)\) holds, we
have \(\hat d_n \le \delta(n,\eta)\), and hence \(\hat d_n^2 \le \delta^2(n,\eta)\).
Therefore, by \eqref{eqn::fixedpointgen},
\[
R(\hat f_n)-R(f_0) \ \le\ \delta^2(n,\eta)
\qquad\text{with probability at least }1-\eta.
\]
In practice, the worst-case rate \(\delta(n,\eta)\) is rarely computed exactly.
Instead, one typically manipulates the fixed-point inequality in
\eqref{eqn::fixedpointgen} to obtain a closed-form algebraic upper bound on
\(\hat d_n\). For this purpose, Young's inequality is often useful
\citep{hardy1952inequalities}.

\begin{lemma}[Young's inequality]
\label{lem:young}
Let \(x,y\ge 0\) and let \(p,q>1\) satisfy \(1/p+1/q=1\). Then
\[
xy \ \le\ \frac{x^p}{p}+\frac{y^q}{q}.
\]
In particular, for any \(\lambda>0\),
\[
xy \ \le\ \frac{\lambda}{2}x^2+\frac{1}{2\lambda}y^2.
\]
\end{lemma}

To illustrate Lemma~\ref{lem:young}, suppose we have shown that
\(\varepsilon(n,C\hat d_n,\eta)\le c\,\delta_{n,\eta}\hat d_n+\delta_{n,\eta}^2\)
for some \(\delta_{n,\eta}>0\) and \(c>0\), as would follow from \eqref{example::rateineq}. Then \eqref{eqn::fixedpointgen} becomes
\[
\hat d_n^2 \ \le\ \delta_{n,\eta}^2 \;+\; c\,\delta_{n,\eta}\,\hat d_n .
\]
Applying Young's inequality with \(x=\hat d_n\), \(y=c\,\delta_{n,\eta}\), and \(\lambda=1\) yields
\[
c\,\delta_{n,\eta}\hat d_n \ \le\ \tfrac12 \hat d_n^2 \;+\; \tfrac12 c^2\delta_{n,\eta}^2.
\]
Substituting and rearranging give \(\hat d_n^2 \le (2+c^2)\delta_{n,\eta}^2\), i.e.,
\(\hat d_n \lesssim \delta_{n,\eta}\) with probability at least $1 - \eta$.

\paragraph{Remark.}
When the risk is strongly convex and satisfies the curvature bound
\(\kappa\|f-f_0\|^2 \le R(f)-R(f_0)\), it is often simplest to analyze the
\(L^2\) error directly \citep{wainwright2019high}. Indeed, combining this
curvature bound with the basic inequality of Theorem~\ref{theorem::basic} yields
\[
\kappa\|\hat f_n-f_0\|^2
\ \le\
(P_n-P)\bigl\{\ell(\cdot,f_0)-\ell(\cdot,\hat f_n)\bigr\}.
\]
On the event in Condition~\ref{cond::highprob} (which has probability at least \(1-\eta\)),
\[
\kappa\|\hat f_n-f_0\|^2
\ \le\
\varepsilon\bigl(n,\sigma_{\hat f_n},\eta\bigr).
\]
If moreover \(\ell\) is pointwise Lipschitz, so that
\(|\ell(Z,f_0)-\ell(Z,f)|\le L\,|f_0(Z)-f(Z)|\) for all \(f\in\mathcal F\), then
\(\sigma_{\hat f_n}\le L\|\hat f_n-f_0\|\), and hence
\[
\kappa\|\hat f_n-f_0\|^2
\ \le\
\varepsilon\bigl(n,L\|\hat f_n-f_0\|,\eta\bigr).
\]
This is a fixed-point inequality for \(\|\hat f_n-f_0\|\) (rather than for the
regret), from which rates follow by the same algebra as in Step~3. Along this
route, the Bernstein condition is not invoked explicitly; it is implied by the
curvature bound together with the Lipschitz control of \(\sigma_{\hat f_n}\) via
Lemma~\ref{lemma::suffbern}.

\section{Regret via localized Rademacher complexity}

\subsection{A general high-probability regret theorem}
\label{sec:genregret}
In this section, we present a general regret theorem for ERM, which is the fulcrum of this paper. Our bound uses the proof recipe of Section~\ref{sec:template} to obtain a PAC-style guarantee in terms of a function-class complexity measure. The key technical step is to control the empirical-process fluctuation
\[
(P_n-P)\{\ell(\cdot,f_0)-\ell(\cdot,\hat f_n)\},
\]
since this term upper bounds the regret $R(\hat f_n)-R(f_0)$ of the empirical risk minimizer $\hat f_n$. The magnitude of this fluctuation is governed by the size of the loss-difference class
\[
\mathcal F_{\ell} := \{\ell(\cdot,f_0)-\ell(\cdot,f) : f\in\mathcal F\}.
\]
Intuitively, richer classes give ERM more opportunities to fit random noise, so the empirical risk $P_n\ell(\cdot,\hat f_n)$ can be much smaller than the population risk $P\ell(\cdot,\hat f_n)$ at the data-selected $\hat f_n$, unless the class is suitably controlled. Consequently, regret bounds are typically expressed in terms of complexity measures for $\mathcal F_{\ell}$, which can often be related back to those of $\mathcal F$.

We quantify the complexity of a function class \(\mathcal{G}\) through its \emph{Rademacher complexity} \citep{bartlett2002rademacher, bartlett2005local, wainwright2019high}. For a radius \(\delta \in (0,\infty)\), the \emph{localized Rademacher complexity} of \(\mathcal{G}\) is
\[
\mathfrak{R}_n(\mathcal G,\delta)
:=
\mathbb{E}\left[
\sup_{\substack{f \in \mathcal{G}\\ \|f\| \le \delta}}
\frac{1}{n} \sum_{i=1}^n \epsilon_i f(Z_i)
\right],
\]
where \(\epsilon_1,\ldots,\epsilon_n \in \{-1,1\}\) are i.i.d.\ Rademacher random variables, independent of \(Z_1,\ldots,Z_n\). Intuitively, this quantity measures how well functions in \(\mathcal{G}\) can align their values \(f(Z_i)\) with random signs \(\epsilon_i\). If the supremum is large, then \(\mathcal{G}\) is flexible enough to fit noise in the sample and is therefore more prone to overfitting; if it is small, then \(\mathcal{G}\) is more constrained and should generalize better. If the supremum is effectively taken over a single function \(g\) with \(\|g\|\le \delta\), then typically \(\mathfrak{R}_n(\mathcal G,\delta)\asymp \delta n^{-1/2}\). The same scaling arises for low-dimensional parametric classes. For many nonparametric classes, by contrast, the dependence on \(\delta\) is slower, often of the form \(\delta^{\alpha}n^{-1/2}\) for some \(\alpha\in(0,1)\); see Table~\ref{tab:critical_radii_examples}.

Generalization bounds, such as those in \cite{wainwright2019high}, often involve star-shaped classes because localized Rademacher complexity satisfies a useful scaling property on such classes. We define the star hull of \(\mathcal{G}\) as
\[
\mathrm{star}(\mathcal{G}) := \{\, th : h \in \mathcal{G},\ t \in [0,1] \,\}.
\]
Any convex class containing \(0\) is star-shaped and therefore equals its star hull. A key feature of \(\mathrm{star}(\mathcal{G})\), which we use repeatedly in our proofs, is that for all \(t \in [0,1]\),
\[
\mathfrak{R}_n\bigl(\mathrm{star}(\mathcal{G}), t\delta\bigr)
\ \ge\
t\,\mathfrak{R}_n\bigl(\mathrm{star}(\mathcal{G}), \delta\bigr),
\]
and hence \(\delta \mapsto \mathfrak{R}_n(\mathrm{star}(\mathcal{G}),\delta)/\delta\) is nonincreasing; see Lemma~\ref{lem:rad_scaling} in Appendix~\ref{appendix::localmax}. In words, shrinking functions in \(\mathrm{star}(\mathcal{G})\) toward zero by a factor \(t \in [0,1]\) reduces the local complexity at least proportionally.

A key characteristic of a function class \(\mathcal{G}\) is its \emph{critical radius}, which governs local concentration and, consequently, the regret of ERM. The critical radius \(\delta_n(\mathcal{G})\) is defined as the smallest \(\delta>0\) such that \(\mathfrak{R}_n(\mathcal{G},\delta)\le \delta^2\), that is,
\[
\delta_n(\mathcal{G})
:= \inf\Bigl\{\delta>0:\ \mathfrak{R}_n(\mathcal{G},\delta)\le \delta^2\Bigr\}.
\]
Larger classes typically have larger critical radii. Intuitively, \(\delta_n(\mathcal G)\) marks the scale at which the empirical \(L^2(P_n)\) norm and the population \(L^2(P)\) norm begin to be comparable on \(\mathcal G\); see Theorem~14.1 of \citealp{wainwright2019high}. Roughly, when \(g\in\mathcal{G}\) satisfies \(\|g\|\asymp \delta_n(\mathcal{G})\), its empirical norm \(\|g\|_n\) is typically of the same order. Below this scale, sampling noise can dominate, so \(\|g\|_n\) may be much larger than \(\|g\|\) purely by chance. For this reason, critical radii enter uniform local concentration inequalities for empirical processes; see Theorem~\ref{theorem:loc_max_ineq} in Appendix~\ref{appendix::localmax}. Localized complexities and critical radii are known for many function classes, and common examples are given in Table~\ref{tab:critical_radii_examples}. In the next section, we develop general tools for upper bounding the critical radius via metric entropy bounds on localized Rademacher complexity.

\begin{table}[h]
\centering
\caption{Examples of localized Rademacher complexity scalings \(\mathfrak{R}_n(\mathcal{G},\delta)\) and critical radii \(\delta_n\) for common function classes. The reported local complexities are upper bounds, and the displayed scaling may hold only for \(\delta\) larger than the critical radius \(\delta_n\). See Appendix \ref{app:references} for references.}
\label{tab:critical_radii_examples}
\begin{tabular}{|lll|}
\hline
Function class $\mathcal{G}$ & Local scaling & Critical radius $\delta_n$ \\
\hline
Star hull of single function $\mathrm{star}(\{f\})$
& $n^{-1/2}\delta$
& $n^{-1/2}$ \\
$s$-sparse linear predictors in $\mathbb{R}^p$
& $n^{-1/2}\delta \sqrt{s\log(e p/s)}$
& $\sqrt{s\log(e p/s)/n}$
\\
VC-subgraph of dimension $V$
& $n^{-1/2}\delta \sqrt{V\log(1/\delta)}$
& $\sqrt{V\log(n/V)/n}$
\\
H\"older/Sobolev with smoothness $s$ in dimension $d < 2s$
& $n^{-1/2}\delta^{1-d/(2s)}$
& $n^{-s/(2s+d)}$
\\
Bounded Hardy--Krause variation
& $n^{-1/2}\delta^{1/2}(\log(1/\delta))^{d-1}$
& $n^{-1/3}(\log n)^{2(d-1)/3}$
\\
RKHS with eigenvalue decay $\sigma_j \asymp j^{-2\alpha}, \alpha > 1/2$
& $n^{-1/2}\delta^{1-1/(2\alpha)}$
& $n^{-\alpha/(2\alpha+1)}$
\\
\hline
\end{tabular}
\end{table}

To bound the regret, we first establish a uniform local concentration inequality for the empirical process \(\{(P_n-P)g : g\in\mathcal F_\ell\}\) of the form in Condition~\ref{cond::highprob}. The resulting bound is governed by the critical radius of the centered loss-difference class
\[
\overline{\mathcal{F}}_{\ell}
:=
\Bigl\{\ell(\cdot,f_0)-\ell(\cdot,f)-P\{\ell(\cdot,f_0)-\ell(\cdot,f)\}:\ f\in\mathcal F\Bigr\}.
\]
In the next section, we study how to bound the associated critical radius. The result and its proof follow \citet{bartlett2005local} and \citet{bartlett2006empirical}, exploiting the fact that the map  $\delta \mapsto \mathfrak{R}_n(\mathrm{star}(\overline{\mathcal{F}}_{\ell}),\delta) / \delta$
is nonincreasing. Alternative uniform local concentration inequalities are given in Theorem~14.20 of \citet{wainwright2019high}; see also Appendix~K of \citet{foster2023orthogonal}. Appendix~\ref{appendix::localuniform} provides a general template for deriving such bounds from local maximal inequalities.

\begin{theorem}[Uniform local concentration inequality]
\label{theorem::localmaxloss}
Let $\delta_n>0$ satisfy the critical radius condition
$\mathfrak{R}_n(\mathrm{star}(\overline{\mathcal{F}}_{\ell}),\delta_n)\le \delta_n^2$.
Suppose there exists \(B \in [1,\infty)\) such that
\(\sup_{f \in \mathcal F}\|\ell(\cdot,f_0)-\ell(\cdot,f)\|_{\infty} \le B\),
and define \(\sigma_f^2 := \Var\{\ell(Z,f_0)-\ell(Z,f)\}\) for each
\(f \in \mathcal F\).
For all $\eta \in (0,1)$, there exists a universal constant $C>0$ such that, with probability at least $1-\eta$, for every $f\in\mathcal{F}$,
\[
(P_n-P) \{\ell(\cdot, f_0) - \ell(\cdot, f)\}
\le
C\left[
\sigma_f\delta_n
+
\delta_n^2
+
\sigma_f B  \sqrt{\frac{\log(1/\eta)}{n}}
+
B^2\frac{\log(1/\eta)}{n}
\right].
\]
Alternatively, if $\delta_n>0$ satisfies the critical radius condition
$\mathfrak{R}_n(\mathrm{star}(\mathcal{F}_{\ell}),\delta_n)\le \delta_n^2$, then the same bound holds with
$\sigma_f$ replaced by $\|\ell(\cdot, f_0) - \ell(\cdot, f)\|$.
\end{theorem}

Our main result is the following PAC-style regret bound. Its proof follows the template of Section~\ref{sec:template}, using that \ref{cond::highprob} holds with $\varepsilon(n,\sigma_f,\eta)$ given by the right-hand side of Theorem~\ref{theorem::localmaxloss}. We assume the following high-level Bernstein condition; sufficient conditions are provided in Lemma~\ref{lemma::suffbern}.

\begin{enumerate}[label=\bf{A\arabic*)}, ref={A\arabic*}, series = main]
  \item \label{cond::bernstein} (Bernstein variance--risk bound) There exists $c_{\mathrm{Bern}} \in [1,\infty)$ such that, for all $f\in\mathcal F$,
\begin{equation*}
\Var\bigl(\ell(Z,f)-\ell(Z,f_0)\bigr)\le c_{\mathrm{Bern}}\{R(f)-R(f_0)\}.
\end{equation*}
\end{enumerate}

\begin{theorem}[Regret bound for ERM]
\label{theorem:genregret}
Assume \ref{cond::bernstein}. Let \(\delta_n>0\) satisfy the critical radius condition
\(\mathfrak{R}_n(\mathrm{star}(\overline{\mathcal{F}}_{\ell}),\delta_n)\lesssim \delta_n^2\), and suppose
\(\sup_{f \in \mathcal{F}} \|\ell(\cdot, f)\|_{\infty} \leq M\) for some \(M \in [1, \infty)\).
Then, for all \(\eta \in (0, 1)\), there exists a universal constant
\(C\in(0,\infty)\) such that, with probability at least \(1-\eta\),
\[
R(\hat f_n)-R(f_0)
\le
C\,c_{\mathrm{Bern}}\Biggl[
\delta_n^2
+
M^2 \frac{\log(1/\eta)}{n}
\Biggr].
\]
Furthermore, if the right-hand side is at most \(1\), then
\(\Pr\{R(\hat f_n)-R(f_0)\le 1\}\ge 1-\eta\).
In this case, possibly with a different universal constant \(C\), the same conclusion holds
with probability at least \(1-\eta\) when \(\delta_n>0\) is chosen to satisfy
\(\mathfrak{R}_n(\mathrm{star}(\mathcal{F}_{\ell}),\delta_n)\lesssim \delta_n^2\).
\end{theorem}

When working instead with the \(L^2\) error and under a stronger Bernstein condition, applying the second part of Theorem~\ref{theorem::localmaxloss} yields a slightly cleaner bound than Theorem~\ref{theorem:genregret}. In that case, the critical radius may be taken with respect to the uncentered loss-difference class \(\mathrm{star}(\mathcal F_{\ell})\) at any sample size.

\begin{enumerate}[label=\bf{A\arabic*)}, ref={A\arabic*}, resume = main]
  \item \label{cond::bernstein::ell2} (Bernstein-type MSE--risk bound) There exists $L\in [1,\infty)$ and $\kappa \in (0,\infty)$ such that, for all $f\in\mathcal F$,
\begin{equation*}
\kappa \|f - f_0\|^2 \leq R(f) - R(f_0), \qquad \|\ell(\cdot,f)-\ell(\cdot,f_0)\| \le L \|f- f_0\|.
\end{equation*}
\end{enumerate}

\begin{theorem}[$L^2$ error bound for ERM]
\label{theorem:genregret::ell2}
Assume \ref{cond::bernstein::ell2}. Let \(\delta_n>0\) satisfy the critical radius condition
\(\mathfrak{R}_n(\mathrm{star}(\mathcal{F}_{\ell}),\delta_n)\lesssim \delta_n^2\), and suppose
\(\sup_{f \in \mathcal{F}} \|\ell(\cdot, f)\|_{\infty} \leq M\) for some \(M \in [1, \infty)\).
Then, for all \(\eta \in (0, 1)\), there exists a universal constant
\(C\in(0,\infty)\) such that, with probability at least \(1-\eta\),
\[
 \|\hat f_n - f_0\|^2 
\le
C\,\kappa^{-1} L^2 \Biggl[
\delta_n^2
+
M^2 \frac{\log(1/\eta)}{n}
\Biggr].
\]
\end{theorem}

Theorem~\ref{theorem:genregret} reduces the problem of bounding the regret \(R(\hat f_n)-R(f_0)\) to controlling the localized Rademacher complexity—and the corresponding critical radius—of the star hull of the centered loss-difference class \(\overline{\mathcal F}_{\ell}\). Consequently, the regret converges to zero at rate \(O_p(\delta_n^2)\), up to a typically second-order \(O_p(n^{-1})\) term. For strongly convex losses over convex classes, Lemma~\ref{lemma::strongconvexity} further yields the \(L^2\) error bound \(\|\hat f_n-f_0\|=O_p(\delta_n)+O_p(n^{-1/2})\). When working directly with \(L^2\) error, Theorem~\ref{theorem:genregret::ell2} shows that \(\|\hat f_n-f_0\|=O_p(\delta_n)+O_p(n^{-1/2})\), where \(\delta_n\) satisfies \(\mathfrak{R}_n(\mathrm{star}(\mathcal F_{\ell}),\delta_n)\lesssim \delta_n^2\), with the critical radius now taken with respect to the uncentered loss-difference class. Alternatively, if \(R(\hat f_n)-R(f_0)=o_p(1)\), the second part of Theorem~\ref{theorem:genregret} permits this choice of \(\delta_n\) for all sufficiently large \(n\), yielding \(R(\hat f_n)-R(f_0)=O_p(\delta_n^2)+O_p(n^{-1})\).

 A general strategy for selecting $\delta_n$ to satisfy the critical radius condition is as follows:
(i) construct a nondecreasing envelope $\phi_n:[0,\infty)\to[0,\infty)$ such that, for all $\delta$ in the relevant range,
\begin{equation}
\label{eqn::env}
\mathfrak{R}_n\left(\mathrm{star}(\overline{\mathcal{F}}_{\ell}),\delta\right)
\ \lesssim\ 
\frac{1}{\sqrt{n}}\,\phi_n(\delta);
\end{equation}
and (ii) choose $\delta_n$ such that
\[
\phi_n(\delta_n) \ \lesssim\ \sqrt{n}\,\delta_n^2.
\]
Then $\delta_n$ is, up to universal constants, a valid upper bound on the critical radius of
$\mathrm{star}(\overline{\mathcal{F}}_{\ell})$, and therefore satisfies the critical radius condition in
Theorem~\ref{theorem:genregret}. A convenient way to construct such envelopes $\phi_n$ is via metric-entropy
integrals, which we develop next.

\paragraph{Aside: eventual almost-sure bounds}

In some applications, one is mainly interested in an asymptotic rate rather than a sharp finite-sample probability bound. The following corollaries yield events that hold eventually almost surely, meaning that there exists a (potentially random) \(N<\infty\) such that, for all \(n\ge N\), the events hold almost surely; this can simplify proofs. The idea is simple: take \(\eta \asymp n^{-2}\) in the high-probability bound and apply the Borel--Cantelli lemma. This introduces only a \(\log n\) factor.

\begin{corollary}[Eventual almost-sure local concentration bound]
\label{cor::localmaxloss_as}
Let \(\delta_n>0\) satisfy the critical radius condition
\(\mathfrak{R}_n(\mathrm{star}(\overline{\mathcal{F}}_{\ell}),\delta_n)\lesssim \delta_n^2\). Suppose the conditions of Theorem~\ref{theorem::localmaxloss} hold. Then there exists a universal constant \(C>0\) such that, with probability one, for all sufficiently large \(n\), every \(f\in\mathcal F\) satisfies
\[
(P_n-P)\{\ell(\cdot,f_0)-\ell(\cdot,f)\}
\le
C\left[
\sigma_f\delta_n
+
\delta_n^2
+
\sigma_f B\sqrt{\frac{\log n}{n}}
+
B^2\frac{\log n}{n}
\right].
\]
The same conclusion holds under the alternative critical-radius condition
\(\mathfrak{R}_n(\mathrm{star}(\mathcal{F}_{\ell}),\delta_n)\le \delta_n^2\), with
\(\sigma_f\) replaced by \(\|\ell(\cdot,f_0)-\ell(\cdot,f)\|\).
\end{corollary}

\begin{corollary}[Eventual almost-sure regret bound]
\label{cor:genregret_as}
Assume Condition~\ref{cond::bernstein}. Suppose \(\sup_{f \in \mathcal{F}} \|\ell(\cdot, f)\|_{\infty} \leq M\) for some \(M \in [1, \infty)\), and let \(\delta_n>0\) satisfy
$\mathfrak{R}_n(\mathrm{star}(\overline{\mathcal{F}}_{\ell}),\delta_n)\lesssim \delta_n^2.$
Then there exists a universal constant \(C\in(0,\infty)\) such that, with probability one, for all sufficiently large \(n\),
\[
R(\hat f_n)-R(f_0)
\le
C\,c_{\mathrm{Bern}}\Biggl[
\delta_n^2
+
M^2\frac{\log n}{n}
\Biggr].
\]
\end{corollary}

\subsection{From entropy to critical radii: bounding localized complexity}
\label{sec:entropy}

This section develops upper bounds on localized Rademacher complexities in terms
of metric entropies of function classes. These bounds are useful because they
reduce the task of deriving ERM rates to controlling covering numbers for
\(\mathcal{F}_{\ell}\), which are well understood for many common classes. In
particular, they yield workable envelopes \(\phi_n(\delta)\) satisfying \eqref{eqn::env},
thereby reducing critical-radius calculations to bounding suitable metric-entropy
integrals \citep{lei2016local, wainwright2019high}.

Let \(\mathcal{G}\) be a class of measurable functions. Let
\(N\left(\varepsilon,\ \mathcal{G},\ L^2(Q)\right)\) denote the \(\varepsilon\)--covering number of \(\mathcal{G}\)
with respect to the \(L^2(Q)\) metric \(d_{2}(f,g) := \|f-g\|_{L^2(Q)}\). That is,
\(N\left(\varepsilon,\ \mathcal{G},\ L^2(Q)\right)\) is the smallest number of \(d_2\)--balls of radius
\(\varepsilon\) needed to cover \(\mathcal{G}\); for background, see Chapter~2 of \cite{sen2018gentle},
Chapter~2 of \cite{van1996weak}, and Chapter~5 of \cite{wainwright2019high}. Define the uniform $L^2$-entropy integral
\[
\mathcal{J}_2(\delta,\mathcal{G})
:= \sup_Q \int_{0}^{\delta} \sqrt{\log N\left(\varepsilon,\ \mathcal{G},\ L^2(Q)\right)}\, d\varepsilon,
\]
where the supremum is over all discrete distributions \(Q\) supported on \(\mathrm{supp}(P)\). We also define the $L^\infty$-entropy integral
\[
\mathcal{J}_{\infty}(\delta, \mathcal{G})
:= \int_0^{\delta} \sqrt{\log N_{\infty}(\varepsilon, \mathcal{G})}\,d\varepsilon,
\]
where \(N_{\infty}(\varepsilon, \mathcal{G})\) is the \(\varepsilon\)--covering number with respect to the
\(P\)-essential supremum metric. Note that \(\mathcal{J}_2(\delta,\mathcal{G}) \le \mathcal{J}_{\infty}(\delta,\mathcal{G})\). Entropy integrals enjoy a useful monotonicity: \(\delta \mapsto \mathcal{J}_2(\delta,\mathcal{G})\) is nondecreasing, while \(\delta \mapsto \mathcal{J}_2(\delta,\mathcal{G})/\delta\) is nonincreasing. Here the lower limit of integration is \(0\), so the integral may diverge; for
example, it diverges for H\"older classes when \(d>2s\). We do not pursue this here, but a standard workaround is to use a truncated
entropy integral, as in Corollary~14.3 of \citealp{wainwright2019high} and Section 4.1 of \cite{foster2023orthogonal}, which
excludes a small neighborhood of \(0\). The ERM analysis can then be modified
to control the regret only on the event that \(\hat f_n\) is not already
sufficiently close to \(f_0\).

A key advantage of entropy integrals is that they behave well under Lipschitz transformations of
function classes. This allows us to control the complexity of composite classes by reducing the task
to entropy bounds for simpler, ``primitive'' classes. In particular, the following Lipschitz preservation bound is often useful; it is
a special case of Theorem~2.10.20 in \citealp{van1996weak}.

\begin{theorem}[Lipschitz preservation for uniform entropy]
\label{theorem:vw_lipschitz_entropy}
Let \(\mathcal F_1,\ldots,\mathcal F_k\) be classes of measurable functions on \(\mathcal Z\), and let
\(\varphi:\mathcal Z\times\mathbb R^k\to\mathbb R\) satisfy, for some \(L_1,\ldots,L_k\ge 0\),
\[
|\varphi(z,x)-\varphi(z,y)| \le \sum_{j=1}^k L_j |x_j-y_j|,
\qquad z\in\mathcal Z,\ x,y\in I,
\]
on a set \(I\subseteq\mathbb R^k\) containing
\(\{(f_1(z),\ldots,f_k(z)):\ f_j\in\mathcal F_j,\ z\in\mathcal Z\}\). Define
\[
\varphi(\mathcal F_1,\ldots,\mathcal F_k)
:=
\bigl\{z\mapsto \varphi\bigl(z,f_1(z),\ldots,f_k(z)\bigr):\ f_j\in\mathcal F_j\bigr\}.
\]
Then, for every \(\delta>0\), up to universal constants,
\[
\mathcal J_2\left(\delta,\ \varphi(\mathcal F_1,\ldots,\mathcal F_k)\right)
\ \lesssim\
\sum_{j=1}^k \mathcal J_2\left(\frac{\delta}{L_j},\ \mathcal F_j\right).
\]
If \(|\varphi(z,x)-\varphi(z,y)| \le L\|x-y\|_2\) for all \(z\in\mathcal Z\) and all \(x,y\in I\), then $\mathcal J_2\left(\delta,\ \varphi(\mathcal F_1,\ldots,\mathcal F_k)\right)
\ \lesssim\
\sum_{j=1}^k \mathcal J_2\left(\frac{\delta}{L},\ \mathcal F_j\right).$
\end{theorem}

Uniform entropy bounds are available for many function classes (Table \ref{tab:uniform_entropy_examples}). For example, for \(d\)-variate Sobolev or H\"older classes with smoothness exponent \(s>d/2\), one typically has
\(\log N_{\infty}(\varepsilon,\mathcal{G}) \lesssim \varepsilon^{-\alpha}\) with \(\alpha := d/s\), and therefore
\(\mathcal{J}_{\infty}(\delta,\mathcal{G}) \lesssim \delta^{\,1-\frac{\alpha}{2}}\)
\citep{van1996weak,nickl2007bracketing}. Classes of bounded Hardy--Krause variation satisfy
\(\mathcal{J}_{2}(\delta,\mathcal{G}) \lesssim \delta^{1/2}\{\log(1/\delta)\}^{d-1}\)
(Proposition 2 of \citealp{bibaut2019fast}). VC-subgraph classes with VC dimension $V$ satisfy the uniform entropy bound $\mathcal{J}_2(\delta,\mathcal{G})
\ \lesssim\
\delta\,\sqrt{V\log\left(\frac{1}{\delta}\right)} $ (Theorem~2.6.7 of \citealp{van1996weak}). In particular, if $\mathcal{G}$ is a linear space of real-valued functions of (finite) dimension $p$, then its VC-subgraph dimension satisfies $V\leq p + 2$ (see, e.g., Lemma~2.6.15 of \citealp{van1996weak}).

\begin{table}[h]
\centering
\caption{Examples of uniform entropy bounds for common function classes.}
\label{tab:uniform_entropy_examples}
\begin{tabular}{|ll|}
\hline
Function class $\mathcal{G}$ & Entropy bound \\
\hline
H\"older/Sobolev (smoothness $s>d/2$ in dimension $d$)
& $\mathcal{J}_{\infty}(\delta,\mathcal{G}) \lesssim \delta^{\,1-d/2s}$;\\
Monotone functions on $[0,1]$  
& $\mathcal{J}_2(\delta,\mathcal{G}) \lesssim \delta^{1/2}$
\\
Bounded Hardy--Krause variation (dimension $d$)
& $\mathcal{J}_{2}(\delta,\mathcal{G}) \lesssim \delta^{1/2}\{\log(1/\delta)\}^{d-1}$
\\
$s$-sparse linear predictors in $\mathbb{R}^p$ (support size $\le s$)
& $\mathcal{J}_2(\delta,\mathcal{G}) \lesssim \delta\sqrt{s\log(ep/(s\delta))}$ \\
VC-subgraph (VC dimension $V$)
& $\mathcal{J}_2(\delta,\mathcal{G}) \lesssim \delta\sqrt{V\log(1/\delta)}$
\\
\hline
\end{tabular}
\end{table}

We now present local maximal inequalities for Rademacher complexities. Our first result bounds the localized Rademacher complexity of \(\mathcal{G}\) in terms of the uniform entropy integral and is adapted from Theorem~2.1 of \cite{van2011local}. It shows that, up to universal constants, a solution to the critical inequality \(\mathfrak{R}_n(\mathcal{G},\delta)\le \delta^2\) can be obtained by solving the entropy-based inequality \(\mathcal{J}_2(\delta,\mathcal{G}) \le \sqrt{n}\,\delta^2\). Related extensions that allow for unbounded function classes or envelopes appear in \cite{chernozhukov2014gaussian} and \cite{vaart2023empirical}.

In what follows, our bounds are stated in terms of the localized class
\(\mathcal{G}(\delta) := \{f \in \mathcal{G}: \|f\|_2 \le \delta\}\). The same conclusion holds if \(\mathcal{G}(\delta)\) is replaced by \(\mathcal{G}\), since \(\mathcal{J}_2(\delta,\mathcal{G}(\delta)) \le \mathcal{J}_2(\delta,\mathcal{G})\).

\begin{theorem}[Local maximal inequality under
uniform $L^2$-entropy]
\label{theorem::maximaluniform}
Let $\delta >0$. Assume $\delta_{\infty}
:= \sup_{f \in \mathcal{G}(\delta)}\|f\|_{\infty} < \infty$
Then, there exists a universal constant $C \in (0,\infty)$ such that
\[
\mathfrak{R}_n(\mathcal{G},\delta)
\ \le\
\frac{C}{\sqrt{n}}\,\mathcal{J}_2\bigl(\delta,\mathcal{G}(\delta)\bigr)
\left\{
1
+
\frac{\delta_{\infty}}{\sqrt{n}\,\delta^2}\,
\mathcal{J}_2\bigl(\delta,\mathcal{G}(\delta)\bigr)
\right\}.
\]
\end{theorem}
Theorem~\ref{theorem::maximaluniform} yields an envelope bound of the form
\(\mathfrak{R}_n(\mathcal{G},\delta)\lesssim n^{-1/2}\phi_n(\delta)\), where
\[
\phi_n(\delta)
:=
\mathcal{J}_2\bigl(\delta,\mathcal{G}\bigr)
\left\{
1+\frac{\delta_{\infty}}{\sqrt{n}\,\delta^2}\,\mathcal{J}_2\bigl(\delta,\mathcal{G}\bigr)
\right\}.
\]
If $\delta_n$ satisfies \(\sqrt{n}\,\delta_n^2 \gtrsim \mathcal{J}_2\bigl(\delta_n,\mathcal{G}\bigr)\), then the bracketed factor is $O(1)$ at $\delta=\delta_n$, so \(\phi_n(\delta_n)\lesssim \mathcal{J}_2\bigl(\delta_n,\mathcal{G}\bigr)\lesssim \sqrt{n}\,\delta_n^2\).
Consequently, $\delta_n$ is, up to constants, a solution to the critical inequality \(\mathfrak{R}_n(\mathcal{G},\delta_n)\lesssim \delta_n^2\) and therefore upper bounds the critical radius of $\mathcal{G}$.

The next theorem sharpens Theorem~\ref{theorem::maximaluniform} when metric entropy is measured in the supremum norm. It shows that \(\mathfrak{R}_n(\mathcal{G},\delta)\lesssim n^{-1/2}\phi_n(\delta)\) with \(\phi_n(\delta):=\mathcal{J}_\infty\bigl(\delta\vee n^{-1/2},\mathcal{G}\bigr)\); in particular, \(\mathfrak{R}_n(\mathcal{G},\delta)\) is of order \(\frac{1}{\sqrt{n}}\,\mathcal{J}_\infty\bigl(\delta,\mathcal{G}\bigr)\) whenever \(\delta \gtrsim n^{-1/2}\). In most applications, however, Theorem~\ref{theorem::maximaluniform} already suffices to derive ERM rates via Theorem~\ref{theorem:genregret}. The proof invokes Theorem~2.2 of \citet{van2014uniform} on concentration of empirical \(L^2\) norms under $L^\infty$-entropy.

\begin{theorem}[Local maximal inequality under $L^\infty$-entropy]
\label{theorem::maximalsup}
Let \(\delta > 0\). Define $\mathcal{G}(\delta) := \{ f \in \mathcal{G}: \|f\| \le \delta\}$ and assume
\(B := 1 \vee \mathcal{J}_\infty\left(\infty, \mathcal{G}\right) < \infty\).
Then there exists a universal constant \(C\in(0,\infty)\) such that
\begin{equation*}
\mathfrak{R}_n(\mathcal{G},\delta)
\leq
\frac{C B}{\sqrt{n}}\,\mathcal{J}_\infty(\delta \vee n^{-1/2},\mathcal{G}(\delta)).
\end{equation*}
\end{theorem}

The bounds above on localized Rademacher complexities can be used to select radii satisfying the critical inequality in Theorem~\ref{theorem:genregret}. In particular, an upper bound on the critical radius of \(\mathrm{star}(\overline{\mathcal{F}}_{\ell})\) can be obtained by choosing \(\delta_n>0\) to satisfy the corresponding entropy-based critical inequality,
\[
\mathcal{J}_2\bigl(\delta_n,\mathrm{star}(\overline{\mathcal{F}}_{\ell})\bigr) \lesssim \sqrt{n}\,\delta_n^2
\qquad\text{or}\qquad
\mathcal{J}_\infty\bigl(\delta_n,\mathrm{star}(\overline{\mathcal{F}}_{\ell})\bigr) \lesssim \sqrt{n}\,\delta_n^2.
\]
It is sharper to solve the former, since \(\mathcal{J}_2\bigl(\delta,\mathrm{star}(\overline{\mathcal{F}}_{\ell})\bigr)\le \mathcal{J}_\infty\bigl(\delta,\mathrm{star}(\overline{\mathcal{F}}_{\ell})\bigr)\) for all \(\delta>0\). We summarize this in the following lemma.

\begin{lemma}[Critical radii via entropy integrals]
\label{lemma::critentropy}
Let \(\phi:(0,\infty)\to(0,\infty)\) satisfy
\(\mathcal{J}_2\bigl(\delta,\mathrm{star}(\overline{\mathcal{F}}_{\ell})\bigr)\lesssim \phi(\delta)\) for all \(\delta>0\).
If \(\delta_n>0\) satisfies \(\phi(\delta_n)\lesssim \sqrt{n}\,\delta_n^2\), then
\(\mathfrak{R}_n\bigl(\mathrm{star}(\overline{\mathcal{F}}_{\ell}),\delta_n\bigr)\lesssim \delta_n^2\).
\end{lemma}

The regret bound in Theorem~\ref{theorem:genregret} is stated in terms of the critical radius of
\(\mathrm{star}(\overline{\mathcal{F}}_{\ell})\), the star hull of the centered loss-difference class.
At first glance, taking a star hull may appear to substantially enlarge the class, inflating the critical radius and degrading rates.
The following lemma alleviates this concern by showing that centering and star-hull closure increase the relevant entropy integral only mildly.
Moreover, when the loss is pointwise Lipschitz, the entropy integral of \(\mathcal{F}_{\ell}\) can be controlled by that of the original ERM class \(\mathcal{F}\).
Together with Lemma~\ref{lemma::critentropy}, this implies that ERM rates can be read off directly from the entropy integrals for \(\mathcal{F}\).

\begin{enumerate}[label=\bfseries A\arabic*), ref={A\arabic*}, resume=main]
  \item \label{cond::lipschitz} (Lipschitz loss.)
  There exists a constant \(L\in(0,\infty)\) such that, for all \(f\in\mathcal F\),
  \[
  \bigl|\ell(Z,f)-\ell(Z,f_0)\bigr|
  \ \le\ 
  L\,\bigl|f(Z)-f_0(Z)\bigr|
  \qquad\text{almost surely.}
  \]
\end{enumerate}

\begin{lemma}[Uniform entropy for star hulls and Lipschitz transformations]
\label{lemma::starentropybounds}
Let $\delta > 0$ and assume $M :=  \sup_{f \in \mathcal{F}}\|\ell(\cdot, f)\|_{\infty} < \infty$. Then, up to universal constants,
\[
\mathcal J_2\left(\delta,\ \mathrm{star}(\overline{\mathcal{F}}_{\ell})\right)
\lesssim
\mathcal J_2\left(\delta,\ \mathcal{F}_{\ell}\right)
+
\delta\,\sqrt{\log\left(1+\frac{M}{\delta}\right)}.
\]
If \ref{cond::lipschitz} holds, we may further upper bound
$\mathcal J_2(\delta,\mathcal F_\ell)$ by $\mathcal J_2(\delta/L,\mathcal F)$ in the display above.
\end{lemma}

It is useful to note that \(\delta\sqrt{\log(1+1/\delta)} \lesssim \sqrt n\,\delta^2\) whenever \(\delta \gtrsim \{\log(e n)/n\}^{1/2}\). Consequently, if \(\delta_n\) satisfies \(\mathcal J_2(\delta_n,\mathcal F_{\ell}) \lesssim \sqrt n\,\delta_n^2\), then, by Lemma~\ref{lemma::starentropybounds}, the enlarged radius \(\widetilde{\delta}_n := \delta_n \vee \{\log(en)/n\}^{1/2}\) satisfies the critical inequality \(\mathcal J_2(\widetilde{\delta}_n,\mathrm{star}(\overline{\mathcal F}_{\ell})) \lesssim \sqrt n\,\widetilde{\delta}_n^2\), up to constants depending on $M$. Combining Theorem~\ref{theorem:genregret} with Lemma~\ref{lemma::critentropy} and Lemma~\ref{lemma::starentropybounds} yields the following corollary.

\begin{corollary}[Regret bound under uniform entropy]
\label{cor:regretentropy}
Assume \ref{cond::bernstein} and \ref{cond::lipschitz}, and suppose that
\(\sup_{f \in \mathcal{F}} \|\ell(\cdot, f)\|_{\infty} \leq M\) for some \(M \in [1, \infty)\).
Suppose there exists an envelope function \(\phi\) such that
\(\mathcal{J}_2(\delta,\mathcal{F}) \lesssim \phi(\delta)\) for all \(\delta>0\), and let \(\delta_n>0\)
satisfy \(\phi(\delta_n)\lesssim \sqrt{n}\,\delta_n^2\). Then, for all \(\eta \in (0,1)\), there exists
a universal constant \(C\in(0,\infty)\) such that, with probability at least \(1-\eta\),
\[
R(\hat f_n)-R(f_0)
\le
C\,c_{\mathrm{Bern}}\Biggl[
L^2 \delta_n^2
+ {\frac{\log (e Mn)}{n}}
+ M^2\frac{\log(1/\eta)}{n}
\Biggr].
\]
\end{corollary}
 
The corollary is stated in terms of the entropy-based critical radius
\(\delta_n\) of the original ERM class \(\mathcal F\), rather than the
critical radius of \(\mathrm{star}(\overline{\mathcal F}_{\ell})\). This
simplification comes at a cost: by Lemma~\ref{lemma::starentropybounds}, it
introduces an additional \(\log(eMn)/n\) term in the regret bound. This term
is typically loose and can often be improved with more refined arguments; see,
for example, Theorem~14.20(b) of \cite{wainwright2019high},
Appendix~K of \cite{foster2023orthogonal}; Chapter 8 of \cite{sen2018gentle}. In particular, Appendix~\ref{appendix::locallip} gives a local concentration inequality,
derived from Theorem~14.20 of \cite{wainwright2019high}, that may be preferred to
Theorem~\ref{theorem::localmaxloss} when the loss \(\ell\) satisfies
Condition~\ref{cond::lipschitz}. When \(\mathcal F\) is convex, it replaces
\(\log(eMn)\) by \(\log\log(eMn)\), and it allows the entropy integral to be
replaced by the localized Rademacher complexity of \(\mathcal F\). Alternatively, if one only seeks the
in-probability rate \(R(\hat f_n)-R(f_0)=O_p(\delta_n^2)\), it suffices to
combine Theorem~3.4.1 of \cite{van1996weak} with
Theorem~\ref{theorem::maximaluniform} applied directly to
\(\mathcal{F}_{\ell}\), thereby avoiding star-shaped hulls and the resulting
extra logarithmic factors. Nonetheless, the corollary suffices for most
nonparametric applications. Outside low-dimensional parametric settings---in
which \(\delta_n^2\) can scale as \(n^{-1}\) and \(\log n / n\) may
dominate---it yields the usual sharp rates.

\subsection{Illustrative examples}

In the next example, we illustrate the standard workflow for computing critical radii and applying the regret theorems: we bound the uniform entropy integral of the loss-difference class via Theorem~\ref{theorem:vw_lipschitz_entropy} in terms of that of \(\mathcal F\), evaluate the resulting integral, and verify the critical-radius condition.

\begin{example}[Least-Squares Regression]
Consider the regression setting where \(Z=(X,Y)\), with \(X\) used to predict \(Y\), and let \(\mathcal F\) be a convex class.
Assume \(|Y|\le B\) almost surely and \(\sup_{f\in\mathcal F}\|f\|_\infty \le B\).
The least-squares loss is \(\ell(z,f):=\{y-f(x)\}^2\) with \(z=(x,y)\). For least squares, the population risk \(R(f):=E\{(Y-f(X))^2\}\) is strongly convex in \(f\) and satisfies
\[
R(f)-R(f_0)
\ \ge\
E\bigl\{(f(X)-f_0(X))^2\bigr\}
\;=\;
\|f-f_0\|^2.
\]
Moreover, uniform boundedness of \(Y\) and \(\mathcal F\) implies pointwise Lipschitzness: for all \(f,f'\in\mathcal F\),
\[
\bigl|\ell\{(X,Y),f\}-\ell\{(X,Y),f'\}\bigr|
=
\bigl|(Y-f(X))^2-(Y-f'(X))^2\bigr|
\le
4B\,|f(X)-f'(X)|
\qquad\text{a.s.}
\]
In particular, Theorem~\ref{theorem:vw_lipschitz_entropy} yields
\(\mathcal J_2(\delta,\mathcal F_{\ell}) \lesssim \mathcal J_2(\delta/(4B),\mathcal F)\). Thus Condition~\ref{cond::lipschitz} holds and, by Lemma~\ref{lemma::suffbern}, so does Condition~\ref{cond::bernstein}.
We therefore apply Corollary~\ref{cor:regretentropy}.

For illustration, we bound the uniform entropy integral of the loss-difference class explicitly. For least squares, $\mathcal F_{\ell}
:=
\{z \mapsto f_0(x)^2 - f(x)^2 - 2y\{f_0(x)-f(x)\}:\ f \in \mathcal F\}.$
This class is contained in \(\widetilde{\mathcal F}_{\ell}-\widetilde{\mathcal F}_{\ell}\), where
\(\widetilde{\mathcal F}_{\ell}:=\{z \mapsto f(x)^2 - 2y f(x):\ f \in \mathcal F\}\). Hence
\[
\mathcal J_2(\delta,\mathcal F_{\ell})
\ \le\
\mathcal J_2\bigl(\delta,\widetilde{\mathcal F}_{\ell}-\widetilde{\mathcal F}_{\ell}\bigr)
\ \lesssim\
\mathcal J_2(\delta,\widetilde{\mathcal F}_{\ell}),
\]
where the last step uses the standard difference-class bound
\(\mathcal J_2(\delta,\mathcal H-\mathcal H)\lesssim \mathcal J_2(\delta,\mathcal H)\).

Next, the class \(\widetilde{\mathcal F}_{\ell}\) is a Lipschitz transformation of \(\mathcal F^2\times \mathcal F\), where
\(\mathcal F^2:=\{f^2:\ f\in\mathcal F\}\), via the map \(\varphi:\mathbb R\times\mathbb R^2\to\mathbb R\) defined by
\(\varphi(y;u,v):=u-2y v\). Indeed, for any \(|y|\le B\) and any \(u,u'\in\mathbb R\), \(v,v'\in\mathbb R\),
\[
|\varphi(y;u,v)-\varphi(y;u',v')|
=
|(u-u')-2y(v-v')|
\le
|u-u'| + 2|y|\,|v-v'|
\le
|u-u'| + 2B\,|v-v'|.
\]
Thus \(\varphi(y;\cdot,\cdot)\) is Lipschitz uniformly over \(|y|\le B\) with respect to the weighted \(\ell_1\) metric
\((u,v)\mapsto |u|+2B|v|\). Consequently, by Lemma~\ref{theorem:vw_lipschitz_entropy}, $\mathcal J_2(\delta,\widetilde{\mathcal F}_{\ell})
\ \lesssim\
\mathcal J_2(\delta,\mathcal F^2)\;+\;\mathcal J_2\left(\frac{\delta}{2B},\mathcal F\right).$
Moreover, since \(t\mapsto t^2\) is \(2B\)-Lipschitz on \([-B,B]\), $\mathcal J_2(\delta,\mathcal F^2)
\ \lesssim\
\mathcal J_2\left(\frac{\delta}{2B},\mathcal F\right).$
Combining the displays yields \(\mathcal J_2(\delta,\mathcal F_{\ell})\lesssim \mathcal J_2\left(\delta/(2B),\mathcal F\right)\).

To apply Corollary~\ref{cor:regretentropy}, it remains to choose \(\delta_n\) satisfying the critical-radius condition
\(\mathcal J_2(\delta_n,\mathcal F_\ell)\lesssim \sqrt n\,\delta_n^2\). Suppose the metric entropy of \(\mathcal F\) grows polynomially:
for some \(\alpha\in(0,2)\),
\[
\sup_Q \log N\left(\varepsilon,\mathcal F,L^2(Q)\right)\ \lesssim\ \varepsilon^{-\alpha},
\qquad \varepsilon\in(0,1].
\]
For H\"older or Sobolev classes on \([0,1]^d\) with smoothness \(s>d/2\), this holds with \(\alpha=d/s\). Then
\(\mathcal J_2(\delta,\mathcal F)\lesssim \int_0^\delta \varepsilon^{-\alpha/2}\,d\varepsilon \asymp \delta^{1-\alpha/2}\),
up to constants depending only on \(\alpha\). Consequently,
\[
\mathcal J_2(\delta,\mathcal F_\ell)
\ \lesssim\
\mathcal J_2\left(\frac{\delta}{B},\mathcal F\right)
\ \lesssim \delta^{1-\alpha/2},
\]
up to constants depending on \(\alpha\) and \(B\). The critical-radius condition holds provided $\delta_n^{1-\alpha/2}\ \lesssim\ \sqrt n\,\delta_n^2,$
equivalently \(\delta_n^{1+\alpha/2}\gtrsim n^{-1/2}\). Thus one may take $\delta_n \asymp n^{-\frac{1}{2+\alpha}}.$
Corollary~\ref{cor:regretentropy}, together with strong convexity, yields
\(\|\hat f_n-f_0\| = O_p(\delta_n) = O_p(n^{-1/(2 + \alpha)})\).
Alternatively, one may apply Theorem~\ref{theorem:genregret::ell2}, using that \(\delta_n\) satisfies the critical inequality via Lemma~\ref{lemma::critentropy} with envelope \(\phi(\delta)\asymp \delta^{1-\alpha/2}\).

\end{example}

\section{ERM with nuisance components}
 \label{sec:nuis}
\subsection{Weighted ERM and regret transfer}

 \label{sec:weightedERM}

Weighted empirical risk minimization arises in many settings where the empirical distribution of the observed sample differs from the target distribution of interest, including missing data, causal inference, and domain adaptation. In these problems, the target risk can often be written as a weighted expectation
\[
R(f;w_0)=E_P\{w_0(Z)\,\ell(Z,f)\},
\]
for a potentially unknown weight function $w_0:\mathcal Z\to\mathbb R$. A natural estimator of $R(f;w_0)$ is the weighted empirical risk $P_n\{\hat w(\cdot)\,\ell(\cdot,f)\}$, where $\hat w$ estimates $w_0$.
The weighted empirical risk minimizer is defined as any solution
\[
\hat f_n \in \arg\min_{f \in \mathcal{F}} \sum_{i=1}^n \hat{w}(Z_i)\,\ell(Z_i,f).
\]
This section develops regret bounds for weighted ERM, highlighting how weight estimation modifies the usual ERM guarantees.

The following theorem characterizes the excess regret incurred by estimating
the weight function. In what follows, let
\(f_0 \in \arg\min_{f\in\mathcal F} R(f;w_0)\) denote a population risk
minimizer, and let \(\hat f_n\) be any random element of \(\mathcal F\). For
any weight function \(w\), define the regret under weight \(w\) by
\[
\Reg(f;w) := R(f;w) - \inf_{f\in\mathcal F} R(f;w),
\]
where $R(f;w) := P\{w(\cdot)\,\ell(\cdot,f)\}.$
We also impose a weighted variant of the Bernstein condition.

\begin{enumerate}[label=\bfseries B\arabic*), ref={B\arabic*}, series=nuis]
    \item \label{cond::weightedberm} (Weighted Bernstein bound.) There exists a constant $c \in (0,\infty)$ such that, for all $f \in \mathcal{F}$,
\[
\Var\left\{w_0(Z)\bigl(\ell(Z,f)-\ell(Z,f_0)\bigr)\right\}
\le
c\bigl\{R(f;w_0)-R(f_0;w_0)\bigr\}.
\]
\end{enumerate}

\begin{theorem}[Regret transfer under weight error]
\label{theorem:weightedERM}
Assume \ref{cond::weightedberm} and that $\left\|1-\frac{\hat w}{w_0}\right\| < 1/2$. Then
\[
\Reg(\hat f_n; w_0)
\le
\Reg(\hat f_n;\hat w)
+
4c^2\left\|1-\frac{\hat w}{w_0}\right\|^2.
\]
\end{theorem}

The quantity \(\Reg(\hat f_n;\hat w)\) is the regret incurred with respect to the
estimated weights \(\hat w\). When \(\hat w\) is trained on data independent of
\(\{Z_i\}_{i=1}^n\), for example via sample splitting, \(\Reg(\hat f_n;\hat w)\)
can be bounded by applying Theorem~\ref{theorem:genregret} conditional on the
weight-training data. Sample splitting is a standard technique for obtaining
fast regret bounds: conditional on the weight-training data, the weighted loss
\(\hat w(\cdot)\,\ell(\cdot,f)\) is fixed and deterministic, which enables the
use of regret-analysis tools for known losses \citep{foster2023orthogonal}. In
this case, Theorem~\ref{theorem:weightedERM} shows that the additional regret
incurred by weight estimation, relative to an oracle procedure that uses the
true weights, is of order \(\|1-\hat w/w_0\|^2\), which is the chi-squared
divergence between \(\hat w\) and \(w_0\).

In the next sections, we study how this dependence on weight-estimation error can be mitigated using Neyman-orthogonal loss functions, and how in-sample nuisance estimation affects regret rates through $\Reg(\hat f_n;\hat w)$.  

\begin{remark}[Weighted losses are orthogonal under realizability]
Consider the regression setting in which $R(f;w_0)=E_P\{w_0(X)\,\ell((X,Y),f)\}$ and $\mathcal{F}$ consists of functions $f(x)$. Appendix~G.2 of \cite{foster2023orthogonal} shows that, under a global realizability condition—typically requiring that the global minimizer lies in $\mathcal{F}$—the weighted loss is orthogonal. Consequently, the quadratic remainder term $\bigl\|1 - \tfrac{\hat w}{w_0}\bigr\|^2$ in Theorem~\ref{theorem:weightedERM} can be replaced by a higher-order term, such as $\bigl\|1 - \tfrac{\hat w}{w_0}\bigr\|_{L^4(P)}^4$. Intuitively, this robustness arises because, when the global minimizer is realizable within $\mathcal{F}$, the choice of weights does not change the population minimizer. As a result, small estimation errors in the weights affect the population risk only at higher order.
\end{remark}

\subsection{Orthogonal losses and nuisance-robust learning}

In many modern learning problems, the loss depends on unknown nuisance components,
especially in causal inference and missing-data settings
\citep{rubin2006doubly, van2003unified, Tlearner, nie2021quasi, kennedy2023towards, foster2023orthogonal, van2023causal, yang2023forster, van2024combining}.
The weighted ERM setup from the previous section is a simple instance, with the
weight function playing the role of the nuisance. A general treatment of
learning with nuisance components, and the role of orthogonality in obtaining
fast rates, is given in \citet{foster2023orthogonal}. Here we summarize the main
ideas and refer the reader there for a comprehensive analysis.

Concretely, suppose we are given a family of losses
$\{\ell_{g}(z,f): g\in\mathcal{G},\ f\in\mathcal{F}\}$, where $g$ denotes nuisance
functions (e.g., regression functions, propensity scores, or density ratios) that
are not themselves the primary learning target. Let $g_0$ denote the true
nuisance, and define the population risk
\[
R(f; g) := P\,\ell_{g}(\cdot,f), \qquad
f_0 \in \arg\min_{f\in\mathcal{F}} R(f; g_0).
\]
In practice, one replaces $g_0$ with an estimator $\hat g$ (often fit using
flexible regression methods and typically cross-fitted), and then computes an
ERM under the plug-in loss,
\[
\hat f_n \in \arg\min_{f\in\mathcal{F}} P_n\,\ell_{\hat g}(\cdot,f).
\]
A key difficulty is that the plug-in objective can be \emph{first-order}
sensitive to nuisance error: even if \(\hat f_n\) nearly minimizes the estimated
risk $R(f; \hat g)$, the regret under the \emph{true} risk,
\[
\Reg(\hat f_n;g_0) := R(\hat f_n;g_0)-R(f_0;g_0),
\]
can depend on \(\hat g-g_0\) through a leading term of order \(\|\hat g-g_0\|^2\)
(rather than a higher-order term, such as \(\|\hat g-g_0\|^4\)). Equivalently,
the \(L^2\) error \(\|\hat f_n-f_0\|\) may depend on the nuisance through the
first-order error \(\|\hat g-g_0\|\), rather than \(\|\hat g-g_0\|^2\). This issue
appears, for example, in weighted ERM, where Theorem~\ref{theorem:weightedERM}
shows that weight-estimation error enters at first order.

\textit{Orthogonal statistical learning} \citep{foster2023orthogonal} addresses
this issue by working with \emph{Neyman-orthogonal} losses. Specifically, for all
$f'\in\mathcal{F}$ and $g'\in\mathcal{G}$,
\[
\partial_{g}\partial_f R(f_0; g_0)\,[f'-f_0,g'-g_0]
\ :=\
\left.\frac{d}{dt}\right|_{t=0}\left.\frac{d}{ds}\right|_{s=0}
R\left(f_0+s (f'-f_0);\, g_0+t(g'-g_0)\right)
\ =\ 0,
\]
so the first-order optimality conditions at $(f_0,g_0)$ are locally insensitive
to perturbations of the nuisance around $g_0$. In many problems, one can obtain
an orthogonal loss from a non-orthogonal criterion by adding a bias-correction
term (often derived from an influence-function expansion) that cancels
the leading effect of nuisance error; see Appendix~D of
\citet{foster2023orthogonal} and Section~2.3 of \citet{van2024combining}.

In the framework of \citet{foster2023orthogonal}, one first controls the regret
under the \emph{estimated} loss,
\[
\Reg(\hat f_n;\hat g) := R(\hat f_n; \hat g) - \inf_{f\in\mathcal{F}} R(f; \hat g),
\]
and then relates it to regret under the \emph{true} loss. In particular,
Theorem~1 of \citet{foster2023orthogonal} gives a a regret-transfer bound of the form
\[
\Reg(\hat f_n;g_0)
\ \le\
C\,\Reg(\hat f_n;\hat g)
\ +\
\Rem(\hat g,g_0),
\]
where $\Rem(\hat g,g_0)$ is a \emph{second-order} remainder in the nuisance error
(e.g., $\Rem(\hat g,g_0)\lesssim \|\hat g-g_0\|^4$ under appropriate conditions), and $C$ is a problem-dependent constant. Thus, if $\hat f_n$
achieves small regret under the plug-in loss and $\hat g$ converges
sufficiently quickly, then $\Reg(\hat f_n;g_0)$ can be controlled without a
first-order nuisance penalty. This decomposition is useful because it separates
(i) statistical/optimization error for the main learner under the estimated loss
from (ii) higher-order error due to nuisance estimation.

To bound the plug-in regret \(\Reg(\hat f_n;\hat g)\), one can reduce to a fixed-loss
analysis via sample splitting: split the data, estimate \(\hat g\) on the first
part, and compute \(\hat f_n\) on the second using \(\ell_{\hat g}\), so that,
conditional on the nuisance-training sample (and hence on \(\hat g\)), standard
high-probability ERM regret bounds for deterministic losses (e.g.,
Theorem~\ref{theorem:genregret}) apply; see the proof of Theorem~3 in
\cite{foster2023orthogonal} for a formal argument.  To avoid the inefficiency of using only
part of the sample at each stage, one can instead use $K$-fold cross-fitting
\citep{van2011cross, chernozhukov2018double}: partition $\{1,\ldots,n\}$ into
folds $I_1,\ldots,I_K$, fit $\hat g^{(-k)}$ on $I_k^c$, evaluate
$\ell_{\hat g^{(-k)}}$ on $I_k$, and minimize the resulting cross-fitted risk
\[
P_n^{\mathrm{cf}}\ell_{\hat g}(\cdot,f)
:=
\frac{1}{K}\sum_{k=1}^K \frac{1}{|I_k|}\sum_{i\in I_k} \ell_{\hat g^{(-k)}}(Z_i,f),
\qquad
\hat f_n \in \arg\min_{f\in\mathcal{F}} P_n^{\mathrm{cf}}\ell_{\hat g}(\cdot,f).
\]
Conditional on $\{\hat g^{(-k)}\}_{k=1}^K$, each observation is evaluated under a
loss fit on independent data, so the same conditional-independence argument as
in sample splitting typically yields comparable ERM regret bounds while using
all observations for both nuisance estimation and learning.

\subsection{In-sample nuisance estimation: ERM without sample splitting}
\label{sec:insample}
With nuisance components, the regret under the true nuisance,
\(\Reg(\hat f_n; g_0)\), can often be decomposed into the regret
\(\Reg(\hat f_n;\hat g)\) under the estimated nuisance plus an approximation
error term reflecting nuisance estimation error. The main additional challenge,
therefore, is to understand how replacing \(g_0\) with \(\hat g\) affects the
regret. In sample-splitting settings, \(\hat g\) is estimated on an independent
subsample and can therefore be treated as fixed when analyzing
\(\Reg(\hat f_n;\hat g)\), so standard regret bounds for fixed losses apply
directly. Sample splitting, however, is data-inefficient, and although
cross-fitting can recover statistical efficiency, it may be computationally
expensive. Moreover, in-sample nuisance estimation is intrinsic to certain
calibration procedures \citep{van2024stabilized, van2024doubly} and debiasing
procedures \citep{van2024combining}. For example, in the Efficient Plug-in (EP)
learning framework of \cite{van2024combining}---a plug-in alternative to
orthogonal learning---initial cross-fitted nuisance estimators are calibrated
post hoc through an in-sample, sieve-based adjustment designed to remove
first-order bias in the nuisance components.

In this section, we bound \(\Reg(\hat f_n;\hat g)\) when the nuisance is
estimated in sample. The main message is that, for suitably smooth
optimization classes \(\mathcal F\), ERM without sample splitting can achieve
the same rate as with sample splitting, provided the nuisance class
\(\mathcal G\) satisfies a Donsker-type condition.  Intuitively, this
condition requires that the nuisance estimator \(\hat g\) belong to a class
that is not too complex. That such conditions
suffice is well known in semiparametric statistics and debiased machine
learning for scalar parameters \citep{van2007empirical}. To our knowledge, however, prior work has not
shown that this extends to functional parameters, as arise in ERM with nuisances. That said, we generally recommend sample splitting in practice,
particularly when using highly adaptive machine learning methods (e.g., neural
networks or boosted trees), for which such Donsker-type conditions are
unlikely to hold.

Let \(\mathcal{F}\) be a class of candidate functions and \(\mathcal{G}\) a
nuisance class. For \(g\in\mathcal{G}\) and \(f\in\mathcal{F}\), let
\(\ell_g(\cdot,f)\) denote the loss evaluated at nuisance value \(g\), and
define the population risk \(R(f;g):=P\ell_g(\cdot,f)\). Denote the empirical
and population risk minimizers under the estimated nuisance \(\hat g\) by
\[
\hat f_n \ :=\ \arg\min_{f\in\mathcal{F}} P_n \ell_{\hat g}(\cdot, f),
\qquad
\hat f_0 \ :=\ \arg\min_{f\in\mathcal{F}} R(f;\hat g).
\]
Our goal is to bound the regret \(\Reg(\hat f_n;\hat g):=R(\hat f_n;\hat g)-R(\hat f_0;\hat g)\)
with respect to the estimated nuisance \(\hat g\). We make no sample-splitting
or cross-fitting assumptions on \(\hat g\). We impose the following mild conditions on the loss.

\begin{enumerate}[label=\bfseries B\arabic*), ref={B\arabic*}, resume=nuis]
\item \textit{(Strong convexity)} \label{cond::strongconvexnuis}
The class $\mathcal{F}$ is convex, and there exists \(\kappa\in(0,\infty)\) such that
\[
R(f;\hat g)-R(\hat f_0;\hat g)\ \ge\ \kappa\,\|f-\hat f_0\|^2
\qquad\text{for all } f\in\mathcal F.
\]

\item \textit{(Product-structured loss)} \label{cond::lipcross}
Assume that the following hold:
\begin{enumerate}[label=(\roman*), ref={\ref{cond::lipcross}-\roman*}]
\item \textit{(Product structure)} \label{cond::productloss}
There exist functions \(m_1,r_1:\mathcal Z\times\mathcal F\to\mathbb R\) and
\(m_2,r_2:\mathcal Z\times\mathcal G\to\mathbb R\) such that, for all
\(z\in\mathcal Z\),
\[
\ell_g(z,f) \ =\ m_1(z,f)\,m_2(z,g) \ +\ r_1(z,f) \ +\ r_2(z,g).
\]

\item \textit{(Pointwise Lipschitz)} \label{cond::Liploss}
There exists \(L\in(0,\infty)\) such that, for all \(z\in\mathcal Z\),
\[
|r_1(z,f_1)-r_1(z,f_2)| \le L |f_1(z)-f_2(z)|,\qquad
|m_1(z,f_1)-m_1(z,f_2)| \le L |f_1(z)-f_2(z)|,
\]
\[
|m_2(z,g_1)-m_2(z,g_2)| \le L |g_1(z)-g_2(z)|.
\]
\end{enumerate}
\end{enumerate}
Condition~\ref{cond::strongconvexnuis} holds under the sufficient conditions of
Lemma~\ref{lemma::strongconvexity}. Condition~\ref{cond::Liploss} ensures that
\(f\) and \(g\) enter the loss \(\ell_g(\cdot,f)\) in a pointwise Lipschitz
manner. The product structure in Condition~\ref{cond::productloss} is satisfied
by many losses. For example, it holds for pseudo-outcome regression losses
\citep{rubin2006doubly, yang2023forster}, such as the DR-learner
\citep{TMLEOptTrt, superLearnOptRule, kennedy2023towards}, which take the form
\(\ell_g(z,f)=\{g(z)-f(z)\}^2\) for a pseudo-outcome function \(g\). In this case,
one may set \(m_1(z,f):=f(z)\), \(m_2(z,g):=-2g(z)\), \(r_1(z,f):=f(z)^2\), and
\(r_2(z,g):=g(z)^2\). The condition also holds for the weighted loss in
Section~\ref{sec:weightedERM}, \(\ell_g(z,f):=g(z)\,\ell(z,f)\), with
\(m_1(z,f):=\ell(z,f)\) and \(m_2(z,g):=g(z)\). More generally, our analysis extends to losses that can be expressed as finite sums of terms with this product structure, potentially involving multiple nuisance components, including the loss class studied by \cite{van2024combining}.

Our analysis of ERM with in-sample nuisance estimation follows the template of
Section~\ref{sec:template}, but uses a refined basic inequality. In particular,
starting from \eqref{eqn::basicgen}, we further decompose
\begin{align*}
R(\hat f_n; \hat g)-R(\hat f_0; \hat g)
&\le
(P_n-P)\bigl\{\ell_{\hat g}(\cdot,\hat f_0)-\ell_{\hat g}(\cdot,\hat f_n)\bigr\}\\
&=
(P_n-P)\bigl\{\ell_{g_0}(\cdot,\hat f_0)-\ell_{g_0}(\cdot,\hat f_n)\bigr\}\\
&\quad+
(P_n-P)\Bigl[
\bigl\{\ell_{\hat g}(\cdot,\hat f_0)-\ell_{\hat g}(\cdot,\hat f_n)\bigr\}
-
\bigl\{\ell_{g_0}(\cdot,\hat f_0)-\ell_{g_0}(\cdot,\hat f_n)\bigr\}
\Bigr],
\end{align*}
where we may take $g_0\in\mathcal{G}$ to be the true nuisance, or more generally any $L^2$ limit of $\hat g$, i.e., $\|\hat g-g_0\|=o_p(1)$, assuming such a limit exists.
The first term is handled exactly as in Section~\ref{sec:template}, since it involves the fixed loss \(\ell_{g_0}\). The second term captures the
data-dependence of \(\ell_{\hat g}\) and is the main difficulty in our analysis,
since \(\ell_{\hat g}\) is random when \(\hat g\) is estimated in-sample.
Nonetheless, the term is localized in both \(\hat f_n\) and \(\hat g\), and we
control it using specialized maximal inequalities that exploit this double
localization and its difference structure. In particular,
Condition~\ref{cond::productloss} shows that it reduces to bounding an empirical
inner product of the form
\[
(P_n-P)\Bigl\{m_1(\cdot,\hat f_n)-m_1(\cdot,\hat f_0)\Bigr\}
\Bigl\{m_2(\cdot,\hat g)-m_2(\cdot,g_0)\Bigr\},
\]
for which the required concentration tools are developed in
Appendix~\ref{appendix:innerproduct}.

With in-sample nuisance estimation, the regret typically depends on the local
complexity of a \emph{loss-difference difference} class,
\[
\Bigl\{\bigl[\ell_{g}(\cdot,f)-\ell_{g}(\cdot,f')\bigr]
-
\bigl[\ell_{g_0}(\cdot,f)-\ell_{g_0}(\cdot,f')\bigr]
:\ f,f'\in\mathcal{F},\ g\in\mathcal{G}\Bigr\},
\]
rather than only on the complexity of \(\mathcal{F}_{\ell_{g_0}}\) under the fixed nuisance
\(g_0\). When \(\mathcal{G}\) is substantially more complex than \(\mathcal{F}\),
oracle rates for ERM under \(\ell_{g_0}\) may be unattainable. In the worst
case, the regret can scale with the square of the larger of the critical radii
associated with the classes used to learn \(\hat f_n\) and \(\hat g\). In
particular, if \(\hat g\) is an ERM over \(\mathcal{G}\), then
\(\Reg(\hat f_n;\hat g)\), and hence \(\Reg(\hat f_n; g_0)\), may be dominated by the estimation error of \(\hat g\),
regardless of whether the loss is orthogonal.

Nonetheless, under suitable smoothness assumptions on \(\mathcal F\)—namely, that
the supremum norm \(\|\cdot\|_{\infty}\) is controlled by a fractional power
\(\|\cdot\|^{\alpha}\) of the \(L^2(P)\) norm over the pairwise difference class
\(\mathcal F-\mathcal F\)—the sensitivity of ERM to nuisance complexity can be
mitigated. Under additional conditions, this yields oracle rates even without
sample splitting. Our main result relies on the following additional
conditions.

\begin{enumerate}[label=\bfseries B\arabic*), ref={B\arabic*}, resume=nuis]

\item \textit{\textit{(\(L^2\)-to-\(L^\infty\) interpolation inequality)}} \label{cond::holdersup}
There exist \(c_{\infty} \in (0,\infty)\) and \(\beta\ \geq 1/2\) such that
\[
\|f-f'\|_\infty \le c_{\infty}\,\|f-f'\|^{1 - \tfrac{1}{2\beta}}
\qquad\text{for all } f,f'\in\mathcal F.
\]

\item \textit{(PAC-style error bound for the nuisance)} \label{cond::nuispac}
For all \(\eta\in(0,1)\),
\[
\|\hat g-g_0\|_{\mathcal G}\ \le\ \varepsilon_{\mathrm{nuis}}(n,\eta)
\qquad\text{with probability at least } 1-\eta/2.
\]
\end{enumerate}
Condition~\ref{cond::holdersup} always holds with \(\beta=1/2\) and \(c_\infty := 2\sup_{f\in\mathcal F}\|f\|_\infty\), but for smoothness classes it often holds with \(\beta>1/2\). For finite-dimensional linear models of dimension \(p\), one may take \(c_\infty\asymp \sqrt{p}\) and \(\beta \rightarrow \infty\). It also holds for reproducing kernel Hilbert space balls with eigenvalue decay \(\lambda_j\asymp j^{-2\beta}\) \citep[Lemma~5.1]{mendelson2010regularization}, for \(\ell^1\)-balls (signed convex hulls) in suitable high- or infinite-dimensional linear spaces \citep[Lemma~2.1]{van2014uniform}, and for \(d\)-variate H\"older and Sobolev classes of order \(s>d/2\) on bounded domains, where \(\beta=s/d\) (see \citealt[Lemma~4]{bibaut2021sequential} and \citealt{adams2003sobolev,triebel2006theory}). Condition~\ref{cond::nuispac} requires that the nuisance estimator satisfy a high-level PAC-style guarantee. For example, this follows from Theorem~\ref{theorem:genregret} when the nuisance is itself obtained by empirical risk minimization over \(\mathcal G\), with \(\varepsilon_{\mathrm{nuis}}(n,\eta)\) scaling with the critical radius of \(\mathcal G\).

The following theorem characterizes the \(L^2(P)\) error of \(\hat f_n\) in
terms of the critical radius of the optimization
class \(\mathcal F\) and the critical radius of the nuisance class
\(\mathcal G\). Its proof relies on a local concentration bound for empirical
inner products of the form
\(\{(P_n-P)(fg):\ f\in\mathcal F,\ g\in\mathcal G\}\)
(Theorem~\ref{lem:uniform_local_conc_lipschitz_product_increments_pointwise}
in Appendix~\ref{appendix:innerproduct}) and on the localized Rademacher
complexity
\(\mathfrak{R}_n\bigl(\mathcal G,\varepsilon_{\mathrm{nuis}}(n,\eta)\bigr)\)
of \(\mathcal G\). We write \(\mathcal F-\mathcal F\) for the pairwise
difference class \(\{f-f' : f,f'\in\mathcal F\}\), and similarly for
\(\mathcal G-\mathcal G\).

\begin{theorem}[$L^2$-error bounds without sample splitting]
\label{theorem::insamplenuis}
Assume Conditions~\ref{cond::strongconvexnuis}--\ref{cond::holdersup}. Let
\(\delta_{n,\mathcal F}>0\) and \(\delta_{n,\mathcal G}>0\) satisfy the critical inequalities
\[
\delta_{n,\mathcal F}^2
\ \gtrsim\
\mathfrak{R}_n\bigl(\delta_{n,\mathcal F},\,\mathcal F-\mathcal F\bigr), \qquad \delta_{n,\mathcal G}^2
\ \gtrsim\
\mathfrak{R}_n\bigl(\delta_{n,\mathcal G},\,\mathrm{star}(\mathcal G-\mathcal G)\bigr).
\]
Fix \(\eta\in(0,1/2)\) and assume further that
\[
\delta_{n,\mathcal F} \wedge \delta_{n,\mathcal G}
\ \gtrsim\
M \sqrt{\frac{\log(1/\eta)+\log\log(eMn)}{n}}.
\]
where \(M:=1 \vee \sup_{g \in \mathcal{G}} \|g\|_\infty \vee \sup_{f \in \mathcal{F}} \|f\|_\infty\). Then, with
probability at least \(1-\eta\),
\[
\|\hat f_n-\hat f_0\|^2
\ \lesssim\
\delta_{n,\mathcal F}^2
+
\Bigl\{\delta_{n,\mathcal G}^2 + \delta_{n,\mathcal G} \, \varepsilon_{\mathrm{nuis}}(n,\eta)\Bigr\}^{4\beta /(2\beta+1)}.
\]
where the implicit constant depends only on \(\kappa, L, c_{\infty}, \alpha\),
\(\sup_{f\in\mathcal F}\|f\|_{\infty}\), and \(\sup_{g\in\mathcal G}\|g\|_{\infty}\).
\end{theorem}
As a special case, taking $\hat g = g_0$ and $\mathcal{G} := \{g_0\}$, our proof permits setting $\delta_{n,\mathcal G} = 0$, thereby recovering Theorem~14.20 in \cite{wainwright2019high} for the regret under fixed (known) losses: $\|\hat f_n - f_0\| \lesssim \delta_{n,\mathcal F}$. In contrast to Theorem~\ref{theorem:genregret::ell2}, Theorem~\ref{theorem::insamplenuis} leverages  \ref{cond::Liploss} to express the regret bound in terms of the critical radius of $\mathcal F - \mathcal F$, rather than that of the loss-difference class.

For simplicity, suppose that \(\hat g\) is obtained by empirical risk
minimization over \(\mathcal G\). In this case,
Theorem~\ref{theorem:genregret} typically yields the PAC-style bound
\(\varepsilon_{\mathrm{nuis}}(n,\eta)\asymp \delta_{n,\mathcal G}\), where
\(\delta_{n,\mathcal G}\) is the critical radius of \(\mathcal G\). Under
this specialization, Theorem~\ref{theorem::insamplenuis} makes explicit how
the critical radii of \(\mathcal F\) and \(\mathcal G\) jointly control the
\(L^2(P)\) error of \(\hat f_n\). The leading term,
\(\delta_{n,\mathcal F}^2\), matches the oracle rate that would be obtained
if the nuisance were known. The additional price of estimating the nuisance
in sample appears through
\(\{\delta_{n,\mathcal G}^2\}^{4\beta/(2\beta+1)}\). The H\"older exponent \(\beta\) in Condition~\ref{cond::holdersup}
quantifies how effectively \(L^2(P)\) localization translates into sup-norm
control, and therefore governs how strongly the nuisance terms are
attenuated. More generally, the theorem highlights two regimes. In the oracle
regime, if $\delta_{n,\mathcal G}
= O(\delta_{n,\mathcal F}^{(2\beta+1)/(4\beta)}),$
then the nuisance terms are \(O(\delta_{n,\mathcal F}^2)\), and the overall
rate matches that of an oracle that knows \(g_0\); in particular,
\(\|\hat f_n-\hat f_0\|^2 \lesssim \delta_{n,\mathcal F}^2\) up to constants.
By contrast, if \(\delta_{n,\eta,\mathcal G}\) exceeds
\(\delta_{n,\mathcal F}^{(2\beta+1)/(4\beta)}\), then the nuisance term
dominates \(\delta_{n,\mathcal F}^2\), and the rate is governed by the
complexity of the nuisance class.

The following corollary shows that the oracle rate can be attained under the
mild Donsker-type requirement that \(\delta_{n,\mathcal G}^2=O(n^{-1/2})\)
(and likewise \(\varepsilon_{\mathrm{nuis}}^2(n,\eta)=O(n^{-1/2})\)) for
optimization classes satisfying the critical-radius scaling
\(\delta_{n,\mathcal F}\asymp n^{-\beta/(2\beta+1)}\).
\begin{corollary}[Sufficient condition for the oracle rate]
\label{cor::oracleeff}
Fix
\(\eta\in(0,1/2)\), and assume the setup and conditions of Theorem~\ref{theorem::insamplenuis}.  Suppose, in addition, that
\[
\delta_{n,\mathcal G}^2 \ \vee\ \varepsilon^2(n,\eta) 
\;=\;
O(n^{-1/2})
\qquad\text{and}\qquad
\delta_{n,\mathcal F}\asymp n^{-\beta/(2\beta+1)}
\]
for the same \(\beta \geq 1/2\) as in Condition~\ref{cond::holdersup}. Then, with
probability at least \(1-\eta\),
\[
\|\hat f_n-\hat f_0\|^2
\ \lesssim\
\delta_{n,\mathcal F}^2 \asymp n^{-2\beta/(2\beta+1)},
\]
where the implicit constant depends only on \(\kappa, L, c_{\infty}, \beta\),
\(\sup_{f\in\mathcal F}\|f\|_{\infty}\), and \(\sup_{g\in\mathcal G}\|g\|_{\infty}\).
\end{corollary}

When \(\delta_{n,\mathcal G} \asymp \varepsilon_{\mathrm{nuis}}(n,\eta)\),
Corollary~\ref{cor::oracleeff} yields a simple sufficient condition for oracle
performance \emph{without} sample splitting. Specifically, the oracle rate
holds if (i) \(\delta_{n,\mathcal G}^2 = O(n^{-1/2})\) and (ii) there exists
\(\beta>1/2\) such that \(\delta_{n,\mathcal F}\asymp
n^{-\beta/(2\beta+1)}\) and the sup-norm bound in
Condition~\ref{cond::holdersup} holds with the same \(\beta\). Condition~(i),
equivalently \(\delta_{n,\mathcal G}=O(n^{-1/4})\), is a Donsker-type
requirement on the nuisance class. For example, this condition holds whenever the global Rademacher complexity
satisfies
\(\mathfrak{R}_n(\infty,\mathrm{star}(\mathcal G-\mathcal G))=O(n^{-1/2})\),
or, more classically, whenever \(\mathcal J_2(\infty,\mathcal G)<\infty\),
which is sufficient for \(\mathcal G\) to be Donsker
\citep{van1996weak}. In
turn, the scaling in (ii) is implied if the local Rademacher complexity%
\footnote{The scaling
\(\sqrt{n}\,\mathfrak{R}_n(\delta,\mathcal F-\mathcal F)\asymp
\delta^{1-\tfrac{1}{2\beta}}\) need only hold for
\(\delta \gtrsim \delta_{n,\mathcal F}\) to ensure that
\(\delta_{n,\mathcal F}\asymp n^{-\beta/(2\beta+1)}\).}
and a local \(L^2(P)\)-to-\(L^\infty\) bound both hold at the same scale:
\begin{equation}
\sqrt{n}\,\mathfrak{R}_n(\delta,\mathcal F-\mathcal F)\asymp \delta^{1-\tfrac{1}{2\beta}}
\qquad\text{and}\qquad
\sup_{f \in \mathcal{F}-\mathcal{F}:\ \|f\|\le \delta}\|f\|_{\infty}\ \lesssim\ \delta^{1-\tfrac{1}{2\beta}}.
\label{eqn::scaling}
\end{equation}

Requirement~(ii) (and \eqref{eqn::scaling}) holds for several standard function classes, including H\"older and Sobolev classes, as well as certain RKHS classes. It typically holds for finite-dimensional linear spaces of dimension \(k(n)\), with scaling proportional to \(\sqrt{k(n)}\,\delta\) \citep{van2014uniform}. Theorem~2.3 of \cite{van2014uniform} further suggests that this scaling is natural for \(\ell^1\)-balls in potentially high- or infinite-dimensional linear spaces. For H\"older and Sobolev classes on \(\mathbb{R}^d\) with smoothness \(s>d/2\), one typically has \eqref{eqn::scaling} with \(\beta=s/d\) (see \citealt[Lemma~4]{bibaut2021sequential} and \citealt{adams2003sobolev,triebel2006theory}). The same scaling holds for a broad class of reproducing kernel Hilbert space balls (RKHSs) whose kernel-operator eigenvalues satisfy \(\lambda_j \asymp j^{-2\beta}\). In particular, the sup-norm bound in Condition~\ref{cond::holdersup} follows from Lemma~5.1 of \citet{mendelson2010regularization}, whereas the corresponding critical-radius scaling is established in Theorem~4.7 of \citet{mendelson2010regularization} and Chapter~13 of \citet{wainwright2019high}.

The \(n^{-1/4}\)-rate condition \(\varepsilon_{\mathrm{nuis}}(n,\eta)=O(n^{-1/4})\) imposed by~(i), together with the scalings in~(ii), also appears in the R-learner analyses of \citet{nie2021quasi}, which study ERM of the CATE with cross-fitted nuisance estimation over RKHSs (see also Example~1 on p.~19 of \cite{foster2023orthogonal}). This suggests that, at least for RKHSs, the requirements for oracle rates are similar with and without cross-fitting, provided that \(\mathcal G\) satisfies a Donsker-type condition and that \(\delta_{n,\mathcal G}\asymp \varepsilon_{\mathrm{nuis}}(n,\eta)\), as is typical when \(\hat g\) is obtained by ERM over \(\mathcal G\). Finally, our analysis establishes the oracle rate only up to multiplicative constants; a sharper argument, as in Theorem~5 of \citet{van2024combining}, also recovers the oracle constant.

\section{Conclusion}

The goal of this guide is to provide a practical reference and proof template for analyzing empirical risk minimization through high-probability regret bounds. The central message is that many ERM analyses can be organized around a basic inequality, a uniform local concentration bound, and a fixed-point argument, yielding rates governed by localized complexity and a critical radius. We also discuss how these critical radii can be related to entropy integrals in common settings, and how the same template extends to nuisance-dependent losses under both sample splitting and in-sample nuisance estimation. Overall, we view these notes as a compact guide to this proof strategy, with room for additional examples and extensions in future revisions.

For simplicity, this guide focuses mainly on uniformly bounded loss functions and function classes. These boundedness assumptions can be relaxed under alternative conditions on moments, envelopes, or tail behavior, including sub-Gaussian, sub-exponential, and more general heavy-tailed settings
\citep{van1996weak,lecue2012general,bibaut2019fast,grunwald2020fast}. Analyses in these settings are typically more involved and technical, and often require more specialized concentration and empirical-process tools
\citep{chernozhukov2014gaussian,balazs2016chaining,wainwright2019high}. Natural directions for future work include extending the analysis to unbounded losses and function classes, as well as to dependent data.

\paragraph{Acknowledgements.}
This guide reflects my understanding of empirical risk minimization as it developed while learning the topic, beginning in my master's program. I first encountered ERM through \citet{wainwright2019high} in the theoretical statistics course STAT~210B at UC Berkeley. My understanding of ERM---particularly in settings with nuisance components and without sample splitting---was further shaped during my PhD at the University of Washington, in part through a dissertation project on Efficient Plug-in learning \citep{van2024combining}. I thank my PhD advisors, Marco Carone and Alex Luedtke, whose guidance, together with Chapter~3 of \citet{van1996weak} and \citet{van2011local}, helped clarify the structure of ERM analyses. I also thank Aur\'elien Bibaut and Nathan Kallus, whose collaboration introduced me to uniform local concentration inequalities and PAC-style analyses of ERM.

\bibliography{ref}

\appendix

\section{Additional details}

\subsection{References for critical radii in Table \ref{tab:critical_radii_examples}}
\label{app:references}

\begin{table}[h]
\centering
\caption{Examples of critical radii for common function classes.}
\label{tab:critical_radii_examples2}
\begin{tabular}{|lll|}
\hline
Function class $\mathcal{G}$ & Critical radius $\delta_n$ & Reference \\
\hline
Star hull of single function $\mathrm{star}(\{f\})$
& $n^{-1/2}$
& \citet{wainwright2019high} \\
$s$-sparse linear predictors in $\mathbb{R}^p$
& $\sqrt{s\log(e p/s)/n}$
& \citet{wainwright2019high}
\\
VC-subgraph of dimension $V$
& $\sqrt{V\log(n/V)/n}$
& \citet{van1996weak}
\\
H\"older/Sobolev with smoothness $s$ in dimension $d$
& $n^{-s/(2s+d)}$
& \citet{nickl2007bracketing}
\\
Bounded Hardy--Krause variation
& $n^{-1/3}(\log n)^{2(d-1)/3}$
& \citet{bibaut2019fast}
\\
RKHS with eigenvalue decay $\sigma_j \asymp j^{-2\alpha}, \alpha > 1/2$
& $n^{-\alpha/(2\alpha+1)}$
& \citet{mendelson2010regularization}
\\
\hline
\end{tabular}
\end{table}

\subsection{Some facts about Reproducing Kernel Hilbert Spaces (RKHS)}
\label{app::rkhs}

RKHS classes provide a canonical example in which approximation error,
local complexity, and critical radii can be characterized explicitly
in terms of kernel eigenvalues. Because of this structure, they serve
as a useful benchmark throughout the guide. In this section, we briefly
review RKHSs and their basic properties; for further details, see
\cite{mendelson2010regularization} and Chapters~13--14 of
\cite{wainwright2019high}.

Consider the regression setting where \(Z=(X,Y)\) with
\(\mathcal Z:=\mathcal X\times \mathcal Y\). Let \(\mu_X\) be a measure
on \(\mathcal X\) such that \(P_X\ll \mu_X\).
Let \((\varphi_j)_{j\ge 1}\) be an orthonormal basis of \(L^2(\mu_X)\),
and let \((\lambda_j)_{j\ge 1}\) be a corresponding sequence of positive
numbers. For example, \((\varphi_j)\) may be the eigenbasis of a compact,
self-adjoint, positive operator on \(L^2(\mu_X)\), with \((\lambda_j)\)
the associated eigenvalues.

\noindent\paragraph{Definition.}
Given an orthonormal basis \((\varphi_j)_{j\ge 1}\) of \(L^2(\mu_X)\) and a sequence
of positive numbers \((\lambda_j)_{j\ge 1}\), the associated RKHS is
\[
\mathcal H
=
\Bigl\{
f=\sum_{j\ge 1}\theta_j \varphi_j
:\ 
\sum_{j\ge 1}\theta_j^2/\lambda_j<\infty
\Bigr\},
\qquad
\|f\|_{\mathcal H}^2
=
\sum_{j\ge 1}
\frac{|\langle f,\varphi_j\rangle_{L^2(\mu_X)}|^2}{\lambda_j}
=
\sum_{j\ge 1}\frac{\theta_j^2}{\lambda_j}.
\]
Equivalently, \(\mathcal H\) is generated by the kernel
\(K(x,x')=\sum_{j\ge 1}\lambda_j \varphi_j(x)\varphi_j(x')\),
which satisfies the reproducing property
\(f(x)=\langle f, K(x,\cdot)\rangle_{\mathcal H}\).
Indeed, if \(f=\sum_{j\ge 1}\theta_j\varphi_j\), then
\[
\langle f, K(x,\cdot)\rangle_{\mathcal H}
=
\sum_{j\ge 1}\theta_j\varphi_j(x)
=
f(x).
\]
A common optimization class is the RKHS ball
\[
\mathcal F:=\{f\in\mathcal H:\|f\|_{\mathcal H}\le R\},
\]
where where radius $R \in (0,\infty)$ plays the role of a regularization parameter.
RKHSs can be viewed as infinite-dimensional analogues of finite-dimensional
linear spaces and retain many convenient structural properties,
such as continuity of the evaluation functional.  

\noindent \paragraph{Finite-dimensional approximation.} Functions in an RKHS ball can be approximated by the finite-dimensional
subspace \(\mathcal H_{k(n)}:=\mathrm{span}\{\varphi_j:1\le j\le k(n)\}\).
If \(h=\sum_{j\ge 1}\theta_j\varphi_j\), its \(L^2(\mu_X)\)-projection
onto \(\mathcal H_{k(n)}\) is
\(\Pi_{k(n)}h=\sum_{j=1}^{k(n)}\theta_j\varphi_j\), and
\[
\|h-\Pi_{k(n)}h\|_{L^2(\mu_X)}^2
=
\sum_{j>k(n)}\theta_j^2
\le
\lambda_{k(n)+1}\|h\|_{\mathcal H}^2.
\]
Thus, \(\lambda_{k(n)+1}\) controls the approximation error of the truncated
expansion; faster eigenvalue decay corresponds to smoother RKHS classes.

\noindent \paragraph{Relation between supremum norm and $L^2$ norm.} RKHSs also admit a convenient \(L^\infty\)–\(L^2\) interpolation inequality
(Lemma~5.1 of \cite{mendelson2010regularization}): every \(f\in\mathcal H\)
satisfies
\[
\|f\|_\infty
\le
\Bigl(\sup_{x\in\mathcal X}\sum_{j\ge 1}\lambda_j \varphi_j(x)^2\Bigr)^{1/2}
\|f\|_{\mathcal H}.
\]
If, in addition, the eigenfunctions are uniformly bounded and
\(\lambda_j\lesssim j^{-1/p}\) for some \(p\in(0,1)\), then
\[
\|f\|_\infty
\lesssim
\|f\|_{\mathcal H}^{\,p}\|f\|_{L^2(\mu_X)}^{\,1-p},
\qquad f\in\mathcal H.
\]

\noindent \paragraph{Local Radamacher complexity.} Finally, the localized Rademacher complexity of an RKHS admits a sharp
bound in terms of the kernel eigenvalues. If
\(\mathcal H(R):=\{f\in\mathcal H:\|f\|_{\mathcal H}\le R\}\) and
\((\lambda_j)_{j\ge 1}\), then 
\[
\mathfrak{R}_n\bigl(\mathcal H(R), \delta\bigr)
\lesssim
\left(\frac{1}{n}\sum_{j\ge 1}
\min\{\delta^2,\ R^2\lambda_j\}\right)^{1/2}.
\]
See Chapter~14 of \citet{wainwright2019high}.
This bound makes explicit how eigenvalue decay governs local complexity:
faster decay of \((\lambda_j)\) reduces the effective number of active
directions at scale \(\delta\), leading to a smaller critical radius. For example, if \(\lambda_j\asymp j^{-2\beta}\) for some \(\beta>1/2\),
then solving the associated fixed-point inequality yields
\(\delta_n^2\asymp n^{-2\beta/(2\beta+1)}\), up to universal constants.

\subsection{Enforcing strong convexity via Tikhonov regularization (Ridge)}

 \label{appendix::regularridge}
Strong convexity of the risk is desirable because it yields fast ERM rates over convex classes. In
our setting, it also ensures that the Bernstein condition in Condition~\ref{cond::bernstein} holds, by
Lemma~\ref{lemma::suffbern}. While many risks are strongly convex, some are not. Tikhonov (ridge)
regularization is a general strategy for enforcing strong convexity by adding a quadratic penalty \citep{hoerl1970ridge, tihonov1977solutions}.

For simplicity, we take the regularization penalty to be the squared $L^2(P)$ norm. For $\lambda\ge 0$,
define the Tikhonov-regularized \emph{population} and \emph{empirical} risks
\[
R_\lambda(f)
:=
P\ell(\cdot,f)+\frac{\lambda}{2}\|f\|^2,
\qquad
R_{n,\lambda}(f)
:=
P_n\ell(\cdot,f)+\frac{\lambda}{2}\|f\|_n^2,
\]
where $\|f\|^2 := P(f^2)$ and $\|f\|_n^2 := P_n(f^2)$. Denote the regularized population minimizer and a regularized empirical minimizer by
\[
f_{0,\lambda}\ \in\ \arg\min_{f\in\mathcal F} R_\lambda(f),
\qquad
\hat f_{n,\lambda}\ \in\ \arg\min_{f\in\mathcal F} R_{n,\lambda}(f).
\]
Other penalties are possible (e.g., an RKHS norm for kernel classes or an $\ell_2$ penalty on
coefficients for linear classes), and our arguments extend to these settings with minor
changes.

Suppose that the risk $R$ is convex and $\mathcal{F}$ is convex. Then $\|\cdot\|^2$ is strongly convex
on $\mathcal F$, and the penalty makes $R_\lambda$ strongly convex even when $R$ is not. An argument
similar to Lemma~\ref{lemma::strongconvexity} yields the quadratic growth bound
\[
R_\lambda(f)-R_\lambda(f_{0,\lambda})
\ \gtrsim\
\lambda\,\|f-f_{0,\lambda}\|^2,
\qquad f\in\mathcal F.
\]
Consequently, the Bernstein condition in Condition~\ref{cond::bernstein} holds for the regularized
objective (by Lemma~\ref{lemma::suffbern}), with $f_0$ replaced by $f_{0,\lambda}$ and Bernstein
constant $c_{\mathrm{Bern}} \asymp \lambda^{-1}$.

Next, apply Theorem~\ref{theorem:genregret} to the modified loss
\(\ell_{\lambda}(z,f) := \ell(z,f) + \tfrac{\lambda}{2} f(z)^2\). This yields
\[
R_\lambda(\hat f_{n,\lambda})-R_\lambda(f_{0,\lambda})
=
O_p\left(\lambda^{-1}\delta_n^2\right),
\]
where \(\delta_n\) is the critical radius associated with the loss class
\(\mathcal{F}_{\ell,\lambda} := \{\ell_{\lambda}(\cdot,f): f \in \mathcal{F}\}\). The quadratic growth bound implies
\[
\| \hat f_{n,\lambda}-f_{0,\lambda}\|
=
O_p\left(\lambda^{-1/2}\delta_n\right).
\] The radius $\delta_n$
can be obtained as in Section~\ref{sec:entropy} by combining entropy integral bounds for
$\mathcal{F}_{\ell}$ and $\mathcal{F}$. In particular, by Lipschitz preservation for entropy integrals
(e.g., Theorem~\ref{theorem:vw_lipschitz_entropy}),
\[
\mathcal{J}_2(\delta, \mathcal{F}_{\ell,\lambda})
\ \lesssim\
\mathcal{J}_2(\delta, \mathcal{F}_{\ell})
+
\mathcal{J}_2\left(\frac{\delta}{\lambda B}, \mathcal{F}\right),
\qquad
B := \sup_{f \in \mathcal{F}}\|f\|_{\infty}.
\]
Moreover, if $\lambda B \lesssim 1$, then by monotonicity of $\delta\mapsto\mathcal J_2(\delta,\mathcal F)$,
Moreover, if $\lambda B \lesssim 1$, then by monotonicity of $\delta\mapsto\mathcal J_2(\delta,\mathcal F)$ we have
\(\mathcal{J}_2\left(\tfrac{\delta}{\lambda B}, \mathcal{F}\right)\lesssim \mathcal{J}_2(\delta, \mathcal{F})\).

It remains to bound the regularization bias $R_\lambda(f_{0,\lambda})-R(f_0)$, i.e., the discrepancy between
the regularized population minimizer and the unregularized minimizer. Indeed, by the triangle
inequality,
\begin{equation}
\label{eqn::trianglereg}
    R_\lambda(\hat f_{n,\lambda})-R(f_0)
\le
R_\lambda(f_{0,\lambda})-R(f_0)
+
O_p\left(\lambda^{-1}\delta_n^2\right).
\end{equation}

\begin{lemma}[Regularization bias bound]
\label{lem:tikhonov_bias_variational}
Let $f_0 \in \argmin_{f \in \mathcal{F}} R(f)$. It holds that 
\[
R_\lambda(f_{0,\lambda})-R(f_0)
\le
\inf_{f \in \mathcal{F}} \left\{\{R(f)-R(f_0)\}+\frac{\lambda}{2}\|f\|^2\right\}.
\]
In particular, if $\|f_0\|<\infty$, then $R_\lambda(f_{0,\lambda})-R(f_0)\le \frac{\lambda}{2}\|f_0\|^2.$
\end{lemma}

Substituting the bound the bound $R(f_{0,\lambda})-R(f_0)\le \frac{\lambda}{2}\|f_0\|^2$ from Lemma~\ref{lem:tikhonov_bias_variational} into \eqref{eqn::trianglereg}, we obtain
\[
    R_\lambda(\hat f_{n,\lambda})-R(f_0)
\ \lesssim\
\lambda
+
O_p\left(\lambda^{-1}\delta_n^2\right).
\]
Balancing the two terms suggests taking $\lambda\asymp \delta_n$, which yields
\[
    R_\lambda(\hat f_{n,\lambda})-R(f_0)
\ =\
O_p(\delta_n).
\]

\paragraph{Remark (on the choice of penalty).}
We use an $L^2(P)$ penalty primarily for expositional convenience: it yields strong convexity directly in the same norm used throughout our analysis. More generally, regularization can improve regret rates by trading off estimation error against the regularization bias $R(f_{0,\lambda})-R(f_0)$; when this bias decays rapidly as $\lambda \downarrow 0$, one may take $\lambda$ smaller and obtain a faster overall rate. The limitation of an $L^2(P)$ penalty is that it does not enforce smoothness beyond $L^2$, and therefore does not deliver the classical gains associated with RKHS or Sobolev regularization in ill-posed problems where the risk itself is not strongly convex. Under stronger penalties and additional smoothness assumptions on $f_0$ (e.g., source conditions), the regularization bias can often be controlled more sharply \citep{engl1996regularization,caponnetto2007optimal, steinwart2008support, steinwart2009optimal, zhang2025optimal}.

\section{Uniform local concentration inequalities for empirical processes}

\label{appendix::localmax} 

\subsection{Notation}

We introduce the following notation. Let $Z_1,\ldots,Z_n \in \mathcal{Z}$ be independent random variables. Let $\mathcal F$ be a class of measurable functions $f:\mathcal Z\to\mathbb R$.
For any $f\in\mathcal{F}$, define
\[
\|f\| := \sqrt{\frac{1}{n} \sum_{i=1}^n \mathbb{E}\{f(Z_i)^2\}}.
\]
We define the global and localized Rademacher complexities by
\[
\mathfrak{R}_n(\mathcal{F})
:=
\mathbb{E}\left[
\sup_{f \in \mathcal{F}}
\frac{1}{n} \sum_{i=1}^n \epsilon_i f(Z_i)\right],
\qquad
\mathfrak{R}_n(\mathcal{F},\delta)
:=
\mathbb{E}\left[
\sup_{\substack{f \in \mathcal{F} \\ \|f\| \le \delta}}
\frac{1}{n} \sum_{i=1}^n \epsilon_i f(Z_i)
\right],
\]
where $\epsilon_1,\ldots,\epsilon_n$ are i.i.d.\ Rademacher random variables, independent of $Z_1,\ldots,Z_n$.
We define the \emph{critical radius} of $\mathcal{F}$ as the smallest $\delta_n\in(0,\infty)$ such that $\mathfrak{R}_n(\mathcal{F},\delta_n) \le \delta_n^2$, that is,
\[
\delta_n := \inf \bigl\{\delta>0 : \mathfrak{R}_n(\mathcal{F},\delta) \le \delta^2 \bigr\}.
\]
Finally, define
\[
P_n f := \frac{1}{n}\sum_{i=1}^n f(Z_i),
\qquad
Pf := \frac{1}{n}\sum_{i=1}^n \mathbb{E}\{f(Z_i)\},
\qquad
P_n^\epsilon f := \frac{1}{n}\sum_{i=1}^n \epsilon_i f(Z_i).
\]
Let $\|\cdot\| := \|\cdot\|_{L^2(P)}$ and $\|\cdot\|_n := \|\cdot\|_{L^2(P_n)}$ denote the population and empirical $L^2$ norms. We say that a function class $\mathcal{F}$ is \emph{star-shaped} (with respect to $0$) if, for every $f\in\mathcal{F}$ and every $\lambda\in[0,1]$, we have $\lambda f \in \mathcal{F}$. We define the star hull of $\mathcal{F}$ as
\[
\mathrm{star}(\mathcal F):=\{t f:\ f\in\mathcal F,\ t\in[0,1]\}.
\]


\subsection{Uniform local concentration for empirical processes}

The following theorem is our main result and provides a local maximal inequality.
Related inequalities appear in \citet{bartlett2005local} and \citet{wainwright2019high};
see also \citet[Appendix~K, Lemmas~11--14]{foster2023orthogonal} and
\citet[Lemma~11]{van2025nonparametric}. Our proof largely follows the technique
outlined in Section~3.1.4 of \citet{bartlett2005local}.

\begin{theorem}[Uniform local concentration inequality]\label{theorem:loc_max_ineq}
Let $\mathcal{F}$ be star-shaped and satisfy $\sup_{f\in\mathcal{F}}\|f\|_\infty\le M$.
Let $\delta_n>0$ satisfy the critical radius condition
$\mathfrak{R}_n(\mathcal{F},\delta_n)\le \delta_n^2$.
Then there exists a universal constant $C>0$ such that, with probability at least $1-\eta$, for every $f\in\mathcal{F}$,
\[
(P_n-P)f
\leq
C \left( \|f\|\delta_n
+
\delta_n^2
+
(M \vee 1) \|f\| \sqrt{\frac{\log(1/\eta)}{n}}
+
(M^2 \vee 1)\frac{\log(1/\eta)}{n} \right).
\]
In particular, if $\delta_n \ \gtrsim\ (M \vee 1)\sqrt{\frac{\log(1/\eta)}{n}}$, then, up to universal constants, $(P_n-P)f
\lesssim
\|f\|\delta_n
+
\delta_n^2.$
\end{theorem}

The proof of the theorem above relies on the following technical lemmas. The first lemma, due to \citet{bousquet2002bennett}, provides a finite-sample concentration inequality for the supremum of an empirical process around its expectation. This result is a Bennett/Bernstein-type refinement of Talagrand's concentration inequality, and its explicit dependence on the envelope and the maximal variance is particularly convenient for localization. For a related bound, see Theorem~3.27 of \cite{wainwright2019high}.

\begin{lemma}[Bousquet’s version of Talagrand’s inequality.]\label{lem:bousquet}
 Let $\mathcal{G}$ be a class of measurable
functions $g:\mathcal{O}\to\mathbb{R}$ such that
$\sup_{g\in\mathcal{G}}\|g\|_\infty \le M$ and
$\sigma^2 := \sup_{g\in\mathcal{G}} \Var\{g(O)\} < \infty$.
Then there exists a universal constant $c>0$ such that, for all $u\ge 0$, with
probability at least $1-e^{-u}$,
\[
\sup_{g\in\mathcal{G}} (P_n-P)g
\le
\mathbb{E}\Bigl[\sup_{g\in\mathcal{G}} (P_n-P)g\Bigr]
+
c\Bigl(
  \sqrt{\frac{u\sigma^2}{n}}
  + \frac{Mu}{n}
\Bigr).
\]
\end{lemma}

The next lemma shows that the expected supremum term $E\bigl[\sup_{g\in\mathcal{G}} (P_n-P)g\bigr]$ in Bousquet's inequality can be upper bounded by the (unlocalized) Rademacher complexity of the function class. Rademacher symmetrization is a standard technique, used for example to derive Dudley-type entropy bounds in \citet{van1996weak} and in the study of Rademacher complexities in \citet{bartlett2002rademacher}  (see Lemma~2.3.1 of \citet{van1996weak}; Proposition 4.11 of \citet{wainwright2019high}).

\begin{lemma}[Rademacher symmetrization bound] \label{lem:rademacher}
For any measurable class of functions $\mathcal{G}$,
\[
\mathbb{E}\Bigl[\sup_{g\in\mathcal{G}} (P_n-P)g\Bigr]
\ \le\
2\,\mathfrak{R}_n(\mathcal{G}).
\]
\end{lemma}

The following lemma restates Lemma 3.4 of \cite{bartlett2005local}.

\begin{lemma}[Scaling property of localized Rademacher complexity]\label{lem:rad_scaling}
Let \(\mathcal{G}\) be star-shaped. Then, for all \(t\in[0,1]\),
\[
\mathfrak{R}_n(\mathcal{G},t\delta)\ \ge\ t\,\mathfrak{R}_n(\mathcal{G},\delta),
\]
and consequently \(\delta \mapsto \mathfrak{R}_n(\mathcal{G},\delta)/\delta\)
is nonincreasing on \((0,\infty)\).
\end{lemma}
\begin{proof}
Fix \(\delta>0\) and \(t\in[0,1]\). If \(\|f\|\le \delta\) and \(\mathcal{G}\) is
star-shaped, then \(tf\in\mathcal{G}\) and \(\|tf\|=t\|f\|\le t\delta\). Hence,
the collection \(\{tf: f\in\mathcal{G},\ \|f\|\le \delta\}\) is contained in
\(\{h\in\mathcal{G}: \|h\|\le t\delta\}\), and therefore
\begin{align*}
\mathfrak{R}_n(\mathcal{G},t\delta)
&=
\mathbb{E}\Bigl[\sup_{\substack{h\in\mathcal{G}\\ \|h\|\le t\delta}}
P_n^\epsilon h\Bigr]
\ \ge\
\mathbb{E}\Bigl[\sup_{\substack{f\in\mathcal{G}\\ \|f\|\le \delta}}
P_n^\epsilon (t f)\Bigr]
=
t\,\mathbb{E}\Bigl[\sup_{\substack{f\in\mathcal{G}\\ \|f\|\le \delta}}
P_n^\epsilon f\Bigr]
=
t\,\mathfrak{R}_n(\mathcal{G},\delta).
\end{align*}

For monotonicity, let \(0<\delta_1\le \delta_2\) and set \(t=\delta_1/\delta_2\in(0,1]\).
Applying the first claim with \(\delta=\delta_2\) gives
\(\mathfrak{R}_n(\mathcal{G},\delta_1)\ge (\delta_1/\delta_2)\mathfrak{R}_n(\mathcal{G},\delta_2)\).
Dividing by \(\delta_1\) yields
\(\mathfrak{R}_n(\mathcal{G},\delta_1)/\delta_1 \ge \mathfrak{R}_n(\mathcal{G},\delta_2)/\delta_2\),
so \(\delta\mapsto \mathfrak{R}_n(\mathcal{G},\delta)/\delta\) is nonincreasing.
\end{proof}

The next lemma bounds the Rademacher complexity of the self-normalized class
$\{f/(\|f\|\vee \delta_n): f\in\mathcal{F}\}$, which yields uniform
high-probability bounds for ratio-type functionals such as $(P-P_n)f/\|f\|$.
This self-normalization (via star-shapedness) is a standard device in the
local Rademacher complexity literature; see, e.g., \citet{bartlett2005local, mendelson2002improving}.

\begin{lemma}[Rademacher complexity for self-normalized classes]\label{lem:self_normalize_rad}
Let $\mathcal{F}$ be star-shaped. Fix $\delta_n>0$ and define $\mathcal{G}:=\bigl\{g=f/(\|f\|\vee \delta_n): f\in\mathcal{F}\bigr\}$. Then
\[
\mathfrak{R}_n(\mathcal{G})
\ \le\
\frac{2}{\delta_n}\,\mathfrak{R}_n(\mathcal{F},\delta_n).
\]
\end{lemma}
\begin{proof}
Define
\[
S(\epsilon)
:=
\sup_{g\in\mathcal{G}} (P_n^\epsilon)g
=
\sup_{f\in\mathcal{F}} \frac{(P_n^\epsilon)f}{\|f\|\vee \delta_n}.
\]
Let $\mathcal{F}_{\le}:=\{f\in\mathcal{F}:\|f\|\le \delta_n\}$ and
$\mathcal{F}_{>}:=\{f\in\mathcal{F}:\|f\|>\delta_n\}$. Then
\[
S(\epsilon)
=
\max\Bigl\{
\sup_{f\in\mathcal{F}_{\le}}\frac{(P_n^\epsilon)f}{\delta_n},\,
\sup_{f\in\mathcal{F}_{>}}\frac{(P_n^\epsilon)f}{\|f\|}
\Bigr\}
\le A(\epsilon)+B(\epsilon),
\]
where
\[
A(\epsilon):=\sup_{f\in\mathcal{F}:\ \|f\|\le \delta_n}\frac{(P_n^\epsilon)f}{\delta_n},
\qquad
B(\epsilon):=\sup_{f\in\mathcal{F}:\ \|f\|>\delta_n}\frac{(P_n^\epsilon)f}{\|f\|}.
\]
We show that $B(\epsilon)\le A(\epsilon)$. Fix any $f\in\mathcal{F}$ with
$\|f\|>\delta_n$, set $\lambda:=\delta_n/\|f\|\in(0,1)$, and define $h:=\lambda f$.
By star-shapedness, $h\in\mathcal{F}$ and $\|h\|=\delta_n$. Moreover,
\[
\frac{(P_n^\epsilon)f}{\|f\|}
=
\frac{(P_n^\epsilon)h}{\delta_n}
\le
\sup_{u\in\mathcal{F}:\ \|u\|\le \delta_n}\frac{(P_n^\epsilon)u}{\delta_n}
=
A(\epsilon).
\]
Taking the supremum over $f\in\mathcal{F}$ with $\|f\|>\delta_n$ yields
$B(\epsilon)\le A(\epsilon)$, and hence $S(\epsilon)\le 2A(\epsilon)$, that is,
\[
\sup_{g\in\mathcal{G}} (P_n^\epsilon)g
\le
\frac{2}{\delta_n}\sup_{f\in\mathcal{F}:\ \|f\|\le \delta_n}(P_n^\epsilon)f.
\]
Taking expectations over $\epsilon$ and over the data completes the proof.
\end{proof}

We now provide the proof of our main result.

\begin{proof}[\textbf{Proof of Theorem \ref{theorem:loc_max_ineq}}]
Define the self-normalized class $\mathcal{G}:=\{g=f/(\|f\|\vee \delta_n): f\in\mathcal{F}\}$, where, by definition, $\delta_n$ satisfies the critical inequality $\mathfrak{R}_n(\mathcal{F},\delta_n)\le \delta_n^2$. Since $\sup_{f\in\mathcal{F}}\|f\|_\infty \le M$ and, for any $f\in\mathcal{F}$,
\[
\Bigl\|\frac{f}{\|f\|\vee \delta_n}\Bigr\|
=
\frac{\|f\|}{\|f\|\vee \delta_n}
\le 1,
\]
we have
\[
\sup_{g\in\mathcal{G}}\|g\|_\infty \le \frac{M}{\delta_n},
\qquad
\sup_{g\in\mathcal{G}}\|g\|\le 1.
\]
Hence, by Bousquet's inequality in Lemma~\ref{lem:bousquet}, there exists a universal constant $c>0$ such that, for all $u\ge 0$, with probability at least $1-e^{-u}$,
\[
\sup_{g\in\mathcal{G}} (P_n-P)g
\le
\mathbb{E}\Bigl[\sup_{g\in\mathcal{G}} (P_n-P)g\Bigr]
+
c\Bigl(
  \sqrt{\frac{u}{n}}
  + \frac{Mu}{\delta_n n}
\Bigr).
\]
By the Rademacher symmetrization bound in Lemma~\ref{lem:rademacher},
\[
\mathbb{E}\Bigl[\sup_{g\in\mathcal{G}} (P_n-P)g\Bigr]
\le
2\mathbb{E}\Bigl[\sup_{g\in\mathcal{G}} (P_n^\epsilon)g\Bigr]
=
2\mathfrak{R}_n(\mathcal{G}).
\]
By Lemma~\ref{lem:self_normalize_rad},
\[
\mathfrak{R}_n(\mathcal{G})
\le
\frac{2}{\delta_n}\,\mathfrak{R}_n(\mathcal{F},\delta_n).
\]
Thus, combining the above, there exists a universal constant $C>0$ such that, for all $u\ge 0$, with probability at least $1-e^{-u}$,
\[
\sup_{g\in\mathcal{G}} (P_n-P)g
\le
\frac{4}{\delta_n}\,\mathfrak{R}_n(\mathcal{F},\delta_n)
+
C\Bigl(
  \sqrt{\frac{u}{n}}
  + \frac{Mu}{\delta_n n}
\Bigr).
\]

Writing out $\mathcal{G}$, this means that, for all $u\ge 0$, with probability at least $1-e^{-u}$,
\[
\sup_{f\in\mathcal{F}}
\frac{(P_n-P)f}{\|f\|\vee \delta_n}
\le
\frac{4}{\delta_n}\,\mathfrak{R}_n(\mathcal{F},\delta_n)
+
C\Bigl(
  \sqrt{\frac{u}{n}}
  + \frac{Mu}{\delta_n n}
\Bigr).
\]
Thus, on the same event, for every $f\in\mathcal{F}$,
\[
(P_n-P)f
\le
(\|f\|\vee \delta_n)\,
\Biggl[
\frac{4}{\delta_n}\,\mathfrak{R}_n(\mathcal{F},\delta_n)
+
C\Bigl(
  \sqrt{\frac{u}{n}}
  + \frac{Mu}{\delta_n n}
\Bigr)
\Biggr].
\]

By definition of the critical radius $\delta_n$, we have $\mathfrak{R}_n(\mathcal{F},\delta_n)\le \delta_n^2$, and hence
\[
\frac{4}{\delta_n}\,\mathfrak{R}_n(\mathcal{F},\delta_n) \le 4\delta_n.
\]
Therefore, there exists a universal constant $C>0$ such that, for all $u\ge 0$, with probability at least $1-e^{-u}$, for every $f\in\mathcal{F}$,
\[
(P_n-P)f
\le
(\|f\|\vee \delta_n)\,
\Biggl[
4\delta_n
+
C\Bigl(
  \sqrt{\frac{u}{n}}
  + \frac{Mu}{\delta_n n}
\Bigr)
\Biggr].
\]
 Equivalently, absorbing constants into a universal $c\in(0,\infty)$ and using
$(\|f\|\vee\delta_n)\delta_n=\|f\|\delta_n \vee \delta_n^2$, we may write
\[
(P_n-P)f
\le
c\Biggl[
(\|f\|\delta_n \vee \delta_n^2)
+
(\|f\|\vee\delta_n)\Bigl(
  \sqrt{\frac{u}{n}}
  + \frac{Mu}{\delta_n n}
\Bigr)
\Biggr].
\]

Setting $u=\log(1/\eta)$ and using $(a\vee b)\le a+b$, we obtain that there exists a
universal constant $C>0$ such that, with probability at least $1-\eta$, for every
$f\in\mathcal{F}$,
\[
(P_n-P)f
\le
C\Biggl[
\delta_n\|f\|+\delta_n^2
+
(\|f\|+\delta_n)\sqrt{\frac{\log(1/\eta)}{n}}
+
(\|f\|+\delta_n)\frac{M\log(1/\eta)}{\delta_n n}
\Biggr].
\]
Recall that $\eta\in(0,1)$ and assume that, up to universal constants,
\[
\delta_n \ \gtrsim\ M\sqrt{\frac{\log(1/\eta)}{n}}.
\]
Then
\[
\frac{M\log(1/\eta)}{\delta_n n}
=
\frac{M}{\delta_n}\cdot \frac{\log(1/\eta)}{n}
\ \lesssim\
\sqrt{\frac{\log(1/\eta)}{n}},
\]
so the term $\{M\log(1/\eta)/(\delta_n n)\}\|f\|$ can be absorbed into
$\|f\|\sqrt{\log(1/\eta)/n}$. Moreover,
\[
\delta_n\sqrt{\frac{\log(1/\eta)}{n}}
\ \gtrsim\
M\frac{\log(1/\eta)}{n},
\]
so the term $(M\log(1/\eta)/n)$ is absorbed into
$\delta_n\sqrt{\log(1/\eta)/n}$. Regrouping yields that there exists a universal
constant $C>0$ such that, with probability at least $1-\eta$, for every
$f\in\mathcal{F}$,
\[
(P_n-P)f
\le
C\Biggl[
\|f\|\delta_n
+
\delta_n^2
+
(\|f\| + \delta_n)\sqrt{\frac{\log(1/\eta)}{n}}
\Biggr].
\]
Therefore, 
\[
(P_n-P)f
\le
C(\|f\|+\delta_n)\left(\sqrt{\frac{\log(1/\eta)}{n}}+\delta_n\right).
\]
Moreover, if $\delta_n \ \gtrsim\ (M \vee 1)\sqrt{\frac{\log(1/\eta)}{n}}$, then
\[
(P_n-P)f
\lesssim
2C\bigl(\|f\|\delta_n +\delta_n^2\bigr),
\]
which gives the second bound.

Finally, if $\widetilde \delta_n$ satisfies $\mathfrak{R}_n(\mathcal{F},\widetilde \delta_n)\le \widetilde \delta_n^2$, then $\delta_n := \widetilde \delta_n + (M \vee 1)\sqrt{\frac{\log(1/\eta)}{n}}$ satisfies $\delta_n \ \gtrsim\ (M \vee 1)\sqrt{\frac{\log(1/\eta)}{n}}$. Hence, applying the bound with this choice of $\delta_n$ and expanding squares, we conclude that, on the same event,
\[
(P_n-P)f
\lesssim
\|f\|\widetilde \delta_n
+
\widetilde \delta_n^2
+
(M \vee 1)\|f\| \sqrt{\frac{\log(1/\eta)}{n}}
+
(M^2 \vee 1)\frac{\log(1/\eta)}{n},
\]
which gives the first bound.

\end{proof}

\subsection{Uniform local concentration for Lipschitz transformations}
\label{appendix::locallip}

The following result is a direct corollary of Theorem~14.20(b) in
\cite{wainwright2019high} (see Lemma~14 of \cite{foster2023orthogonal} for a vector-valued extension). The main
differences are that we state the bound in PAC-style form and, rather than imposing a
lower-bound condition on the critical radius \(\delta_n\), we include the
additional term \(L^2\,\frac{\log\log(M e n)}{n}\).

\begin{theorem}[Uniform local concentration (Lipschitz transform; PAC-style form)]
\label{thm:uniform_local_conc_peeling}
Fix \(f_0\in\mathcal F\) and assume the difference class
\(\mathcal F-f_0:=\{f-f_0:\ f\in\mathcal F\}\) is star-shaped. Let
\(\varphi :\mathcal Z\times\mathcal F \rightarrow \mathbb{R}\) be a map with
\(\sup_{f\in\mathcal F}\|\varphi(\cdot,f)-\varphi(\cdot,f_0)\|_\infty \leq M\) for some $M \in [1, \infty)$.
Assume \(\varphi\) is pointwise Lipschitz: there exists \(L\in(0,\infty)\) such
that for all \(f,f'\in\mathcal F\),
\[
\bigl|\varphi(z,f)-\varphi(z,f')\bigr|
\le L\,|f(z)-f'(z)|,\qquad z\in\mathcal Z.
\]
Let \(\delta_n>0\) satisfy the critical-radius condition
\begin{equation}
\mathfrak R_n\bigl(\mathcal F-f_0,\delta_n\bigr)
\ \lesssim\
\delta_n^2.
\label{eq:crit_F}
\end{equation}
Then there exists a universal constant \(C\in(0,\infty)\) such that for every
\(\eta\in(0,1)\), with probability at least \(1-\eta\), the following holds
simultaneously for all \(f\in\mathcal F\):
\begin{align}
\bigl|(P_n-P)\{\varphi(\cdot,f)-\varphi(\cdot,f_0)\}\bigr|
\;\le\;
C\Biggl[
&L\,\delta_n\,\|f-f_0\|
+
L^2\,\delta_n^2
+
L^2\,\frac{\log\log(M e n)}{n}\\
& \quad +
L\,\|f-f_0\|\sqrt{\frac{\log(1/\eta)}{n}}
+
\frac{M\log(1/\eta)}{n}
\Biggr].
\label{eq:uniform_local_conc_peeling}
\end{align}
\end{theorem}

\textbf{Remark.} A similar bound to Theorem \ref{thm:uniform_local_conc_peeling} can also be obtained by applying Theorem~\ref{theorem:loc_max_ineq}
to the star-shaped hull of the transformed class
\(\mathcal F_{\varphi}:=\{\varphi(\cdot,f)-\varphi(\cdot,f_0): f\in\mathcal F\}\). This approach avoids the additional
\(L^2\,\frac{\log\log(M e n)}{n}\) term, which is an artifact of the peeling
argument. The tradeoff is that Theorem~\ref{theorem:loc_max_ineq} would then be
stated in terms of the critical radius of \(\mathcal F_{\varphi}\), rather than
that of \(\mathcal F\). As shown in Lemma~\ref{lemma::starentropybounds} in
Section~\ref{sec:entropy}, this distinction is typically mild once we pass to
entropy integrals, since star-shaped hulls inflate covering numbers only
slightly. By contrast, the result in Theorem \ref{thm:uniform_local_conc_peeling} leverages the Lipschitz property of
\(\varphi\) and star-shapedness of $\mathcal{F} - f_0$ to express the bound directly in terms of the critical radius of
\(\mathcal{F} - f_0\). In most settings, the two approaches lead to comparable ERM
guarantees.

\subsection{Local maximal inequalities via metric entropy}

\label{appendix::localmax2}

We restate and prove the local maximal inequalities from Section~\ref{sec:entropy}.
Our proofs rely on the following standard lemma, taken from Theorem~5.22 (see
also equation~(5.48)) of \citet{wainwright2019high}. The lemma follows from a
chaining argument combined with sub-Gaussian concentration for Rademacher
processes.

\begin{lemma}[Dudley entropy bound \citep{dudley1967sizes}]
\label{lemma:dudley}
Let $\delta>0$ and define the empirical radius $\hat{\delta}_n \;:=\; \sup\bigl\{\|f\|_n : f\in\mathcal G,\ \|f\|\le \delta\bigr\}.$
Then,
\[
\mathfrak{R}_n(\mathcal{G},\delta)
\;\le\;
\frac{12}{\sqrt{n}}
\mathbb{E} \left[ \int_{0}^{\hat{\delta}_n}
\sqrt{\log N\left(\varepsilon,\ \mathcal G,\ L^2(P_n)\right)}\, d\varepsilon \right].
\]
\end{lemma}

\begin{theorem*}[Local maximal inequality under
uniform $L^2$-entropy]
Let $\delta >0$. Define $\delta_{\infty}
:= \sup_{f \in \mathcal{G}(\delta)}\|f\|_{\infty}.$
Then, there exists a universal constant $C \in (0,\infty)$ such that
\[
\mathfrak{R}_n(\mathcal{G},\delta)
\ \le\
\frac{C}{\sqrt{n}}\,\mathcal{J}_2\bigl(\delta,\mathcal{G}(\delta)\bigr)
\left\{
1
+
\frac{\delta_{\infty}}{\sqrt{n}\,\delta^2}\,
\mathcal{J}_2\bigl(\delta,\mathcal{G}(\delta)\bigr)
\right\}.
\]
\end{theorem*}

\begin{proof}[\textbf{Proof of Theorem \ref{theorem::maximaluniform}}]
Dudley's bound in Lemma~\ref{lemma:dudley} implies that
\[
\mathfrak{R}_n(\mathcal{G},\delta)
\;\le\;
\frac{12}{\sqrt{n}}\, \mathbb{E}\left[\mathcal{J}_2(\hat{\delta}_n,\mathcal{G})\right], \quad \hat{\delta}_n
\;:=\;
\sup\bigl\{\|f\|_n : f\in\mathcal G,\ \|f\|\le \delta\bigr\}.
\]
In the proof of Theorem~2.1 of \cite{van2011local}, it is shown that the map $t \mapsto \mathcal{J}_2(\sqrt{t},\mathcal{G})$ is concave. Hence, by Jensen's inequality,
\[
\mathbb{E}\left[\mathcal{J}_2(\hat{\delta}_n,\mathcal{G})\right]
\;\le\;
\mathcal{J}_2(\delta_n,\mathcal{G}),
\qquad
\delta_n^2 \;:=\; \mathbb{E}\left[\hat{\delta}_n^2\right].
\]
Thus,
\begin{equation}
\label{eqn::firstentropybound}
\mathfrak{R}_n(\mathcal{G},\delta) \le \frac{12}{\sqrt{n}}  \mathcal{J}_2(\delta_n,\mathcal{G}).
\end{equation}
We now proceed by bounding $\delta_n^2$. Observe that
\begin{align*}
\delta_n^2
&= \mathbb{E}\left[ \sup_{f \in \mathcal{G}: \|f\|\le \delta} P_n f^2 \right] \\
&\le \mathbb{E}\left[ \sup_{f \in \mathcal{G}: \|f\|\le \delta} (P_n - P) f^2 \right]
\;+\;
\mathbb{E}\left[ \sup_{f \in \mathcal{G}: \|f\|\le \delta} P f^2 \right] \\
&\le \mathbb{E}\left[ \sup_{f \in \mathcal{G}: \|f\|\le \delta} (P_n - P) f^2 \right] + \delta^2.
\end{align*}
Furthermore, by Lemma~\ref{lem:rademacher},
\[
\mathbb{E}\left[ \sup_{f \in \mathcal{G}: \|f\|\le \delta} (P_n - P) f^2 \right]
\le 2\,\mathfrak{R}_n(\mathcal{G}^2,\delta),
\]
where $\mathcal{G}^2 := \{f^2 : f \in \mathcal{G}\}$.

Next, define
\[
\delta_{\infty}
:= \sup\bigl\{\|f\|_{\infty} : f \in \mathcal{G},\ \|f\|\le \delta\bigr\}.
\]
The map $t \mapsto t^2$ is $2\delta_{\infty}$-Lipschitz on $[-\delta_{\infty},\delta_{\infty}]$. Hence, by Talagrand's contraction lemma and \eqref{eqn::firstentropybound},
\[
\mathfrak{R}_n(\mathcal{G}^2,\delta)
\le 2\delta_{\infty}\,\mathfrak{R}_n(\mathcal{G},\delta)
\le \delta_{\infty}\,\frac{24}{\sqrt{n}}\,\mathcal{J}_2(\delta_n,\mathcal{G}).
\]
We conclude that
\[
\delta_n^2 \le \delta^2 + \delta_{\infty}\,\frac{24}{\sqrt{n}}\,\mathcal{J}_2(\delta_n,\mathcal{G}).
\]
As in the proof of Theorem~2.1 of \cite{van2011local} (see Equation~2.4), we apply Lemma~2.1 of that same work with $r=1$, $A=\delta$, and $B^2 = 24\delta_{\infty}/\sqrt{n}$ to conclude that
\begin{equation}
\mathcal{J}_2(\delta_n,\mathcal{G})
\le
\mathcal{J}_2(\delta,\mathcal{G})
+
\frac{24\delta_{\infty}}{\sqrt{n}\,\delta^2}\,\mathcal{J}_2(\delta,\mathcal{G})^2.
\end{equation}
Combining the above with \eqref{eqn::firstentropybound} yields the claim.

\end{proof}

\begin{theorem*}[Local maximal inequality under $L^\infty$-entropy]
Let \(\delta > 0\). Define $\mathcal{G}(\delta) := \{ f \in \mathcal{G}: \|f\| \le \delta\}$ and
\(B := 1 + \mathcal{J}_\infty\left(\infty, \mathcal{G}\right)\).
Then there exists a universal constant \(C\in(0,\infty)\) such that
\begin{equation*}
\mathfrak{R}_n(\mathcal{G},\delta)
\leq
\frac{C B}{\sqrt{n}}\,\mathcal{J}_\infty(\delta \vee n^{-1/2},\mathcal{G}(\delta)).
\end{equation*}
\end{theorem*}
\begin{proof}[\textbf{Proof of Theorem \ref{theorem::maximalsup}}]
Without loss of generality, assume that $\delta > n^{-1/2}$, so that $\delta \vee n^{-1/2} = \delta$. The proof of Theorem \ref{theorem::maximaluniform} showed that
\begin{equation*}
\mathfrak{R}_n(\mathcal{G},\delta) \le \frac{12}{\sqrt{n}}  \mathcal{J}_2(\delta_n,\mathcal{G}), 
\qquad
\delta_n^2 \;:=\; \mathbb{E}\left[\hat{\delta}_n^2\right],
\end{equation*}
where $\hat{\delta}_n
\;:=\;
\sup\bigl\{\|f\|_n : f\in\mathcal G,\ \|f\|\le \delta\bigr\}$. Since $\|\cdot\|_{L^2(Q)} \le \|\cdot\|_{\infty}$ for any distribution $Q$, it follows that $\mathcal{J}_2(\delta,\mathcal{G}) \le  \mathcal{J}_\infty(\delta,\mathcal{G})$. Thus,
\begin{equation}
\label{eqn::firstentropybound:sup}
\mathfrak{R}_n(\mathcal{G},\delta) \le \frac{12}{\sqrt{n}}  \mathcal{J}_\infty(\delta_n,\mathcal{G}).
\end{equation}

We now proceed by bounding $\delta_n^2$. The proof of Theorem \ref{theorem::maximaluniform} further showed that
\begin{align*}
\delta_n^2
&\le \mathbb{E}\left[ \sup_{f \in \mathcal{G}: \|f\|\le \delta} (P_n - P) f^2 \right] + \delta^2.
\end{align*}
Let $\mathcal{G}_{\delta}:=\{f\in\mathcal{G}:\|f\|\le \delta\}$ and set $\delta_{\infty} := \sup_{f\in\mathcal{G}_{\delta}}\|f\|_\infty.$
Since $(P_n-P)f^2 = \|f\|_n^2-\|f\|^2$, we have
\[
\sup_{f\in\mathcal{G}_{\delta}}(P_n-P)f^2
\le
\sup_{f\in\mathcal{G}_{\delta}}\bigl|\|f\|_n^2-\|f\|^2\bigr|.
\]
Therefore, Theorem~2.1 in \cite{van2014uniform} (applied with $R=\sup_{f\in\mathcal{G}_{\delta}}\|f\|\le \delta$ and $K = \delta_{\infty}$) yields
\begin{align*}
\mathbb{E}\left[ \sup_{f \in \mathcal{G}: \|f\|\le \delta} (P_n - P) f^2 \right]
&\le
\frac{2\,\delta}{\sqrt{n}}\,\mathcal{J}_\infty\left(\delta_{\infty} ,\ \mathcal{G}\right)
+
\frac{4}{n}\,\mathcal{J}_\infty^2\left(\delta_{\infty} , \mathcal{G}\right).
\end{align*}
Thus,
\begin{align*}
\delta_n^2
&\le \delta^2
+ \frac{2B\,\delta}{\sqrt{n}}
+ \frac{4B^2}{n},
\end{align*}
where \(B := \mathcal{J}_\infty\left(\delta_{\infty}, \mathcal{G}\right)\).
By assumption \(\delta \ge n^{-1/2}\), we have \(\delta/\sqrt{n} \le \delta^2\) and \(B^2/n \le B^2\delta^2\).
Therefore,
\begin{align*}
\delta_n^2
&\le (1+2B)\delta^2 + 4B^2\delta^2
\;\le\; (2+8B^2)\delta^2.
\end{align*}

Combining the previous display with \eqref{eqn::firstentropybound:sup}, we obtain that there exists a universal constant \(C\in(0,\infty)\) such that
\begin{equation*}
\mathfrak{R}_n(\mathcal{G},\delta)
\le \frac{C}{\sqrt{n}}\,\mathcal{J}_\infty\bigl((1+B)\delta,\mathcal{G}\bigr)
\le \frac{C(1+B)}{\sqrt{n}}\,\mathcal{J}_\infty(\delta,\mathcal{G}).
\end{equation*}
where $B$ is upper bounded by $ \mathcal{J}_\infty\left(\infty, \mathcal{G}\right)$. The claim follows.
\end{proof}

\subsection{Template for deriving uniform local concentration inequalities}
\label{appendix::localuniform}
Let \(\mathcal{G}\) be a uniformly bounded function class. This section gives a general template for deriving local concentration inequalities for the empirical process \(\{(P_n-P)g : g\in\mathcal{G}\}\). The goal is to control the empirical process not only on a single localized set \(\mathcal{G}(\delta)\), but uniformly over all localization levels \(\delta\). Given such a bound, a uniform local concentration inequality then follows.

Let \(\mathcal{G}(\delta)\subseteq \mathcal{G}\) be a localized subset that is monotone increasing in \(\delta\), in the sense that \(\mathcal{G}(\delta)\subseteq \mathcal{G}(\delta')\) whenever \(\delta \le \delta'\). For example, one may take
\[
\mathcal{G}(\delta):=\{g\in\mathcal{G}:\|g\|\le \delta\}.
\]
Alternatively, if \(\mathcal{G}:=\{\varphi(f): f\in\mathcal{F}\}\) for some mapping \(\varphi:\mathcal{F}\to L^2(P)\), it can be useful to localize through the underlying class \(\mathcal{F}\) and define
\[
\mathcal{G}(\delta):=\{\varphi(f): f\in\mathcal{F},\ \|f\|\le \delta\}.
\]
The latter is useful when \(\varphi(f)\) is pointwise Lipschitz in \(f\), since Talagrand's contraction theorem for Rademacher processes then shows that $\mathfrak{R}(\mathcal{G}(\delta)) \lesssim \mathfrak{R}(\mathcal{F}, \delta),$
where \(\mathfrak{R}(\mathcal{G}(\delta)) := \mathfrak{R}(\mathcal{G}(\delta), \infty)\) denotes the global complexity.

The basic strategy has three steps. First, one derives a local maximal inequality for each fixed radius \(\delta\), typically in terms of an envelope \(\phi_n(\delta)\), such as a localized Rademacher complexity or an entropy-integral bound. Second, one upgrades this pointwise-in-\(\delta\) control to a bound that holds simultaneously over a range of radii using a peeling argument. Third, one rewrites the resulting bound in terms of a critical radius \(\delta_n\), defined through the envelope \(\phi_n\), yielding a cleaner expression that is especially useful for regret analysis.

A first core ingredient is thus access to a local maximal inequality for \(\mathcal{G}(\delta)\): assume that, for all \(\delta>0\), we have an envelope \(\phi_n(\delta)>0\) such that
\[
E\left[\sup_{f \in \mathcal{G}(\delta)} (P_n - P)f \right] \le \phi_n(\delta).
\]
For example, by the Rademacher symmetrization bound in Lemma~\ref{lem:rademacher}, one may take \(\phi_n(\delta)\) to be proportional to the localized Rademacher complexity \(\mathfrak{R}_n(\mathcal{G}, \delta)\), or one of its entropy-based upper bounds from Section~\ref{sec:entropy}.

Given such an envelope, Talagrand's concentration inequality (Lemma~\ref{lem:bousquet}) yields a high-probability bound for the empirical process restricted to a fixed local set \(\mathcal{G}(\delta)\). The next lemma records this pointwise local maximal inequality. 

\begin{lemma}[High-probability local maximal inequality]
\label{lem:bousquet}
Assume there exist bounds \(M(\delta), \sigma^2(\delta)\in(0,\infty)\) such that
\[
\sup_{g\in\mathcal{G}(\delta)}\|g\|_\infty \le M(\delta)
\qquad\text{and}\qquad
\sup_{g\in\mathcal{G}(\delta)} \|g-Pg\|^2 \le \sigma^2(\delta).
\]
Then there exists a universal constant \(c>0\) such that, for all \(u\ge 0\), with probability at least \(1-e^{-u}\),
\[
\sup_{g\in\mathcal{G}(\delta)} (P_n-P)g
\le
\phi_n(\delta)
+
c\Bigl(
  \sqrt{\frac{u\,\sigma^2(\delta)}{n}}
  + \frac{M(\delta)\,u}{n}
\Bigr).
\]
\end{lemma}

Next, by applying a peeling argument, we can strengthen Lemma~\ref{lem:bousquet} to obtain a bound that holds uniformly over \(\delta\).

\begin{lemma}[Peeling-based Uniform local concentration bound]
\label{lem:local-peeling}
Fix $\delta_0 > 0$. Suppose on $(0, \delta_0]$:
\begin{itemize}
\item \(\mathcal G(\delta)\) is nondecreasing in \(\delta\), i.e.\ \(\delta_1\le \delta_2\) implies \(\mathcal G(\delta_1)\subseteq \mathcal G(\delta_2)\);
\item \(\phi_n(0)=0\);
\item \(\delta\mapsto \phi_n(\delta)\) is nondecreasing;
\item \(\delta\mapsto \phi_n(\delta)/\delta\) is nonincreasing.
\end{itemize}
 Assume moreover that, for all \(\delta\in(0,\delta_0]\),
\[
\sup_{g\in\mathcal G(\delta)}\|g\|_\infty \le B
\qquad\text{and}\qquad
\sup_{g\in\mathcal G(\delta)} \|g-Pg\| \leq \sigma(\delta) ,
\]
where $\delta \mapsto \sigma(\delta) $ is nondecreasing and \(\sigma(2\delta)\leq c \sigma(\delta)\) for a universal constant $c \in (0,\infty)$.
Then there exists a universal constant \(C>0\) such that, for all \(u\ge 0\), with probability at least \(1-Ce^{-u}\), the following holds simultaneously for all \(\delta\in(0,\delta_0]\):
\[
\sup_{g\in\mathcal G(\delta)} (P_n-P)g
\le
C\left[
\phi_n(\delta)
+
\sigma(\delta)\sqrt{\frac{u+\log\!\Bigl(1+\log(\delta_0/\delta)\Bigr)}{n}}
+
B\frac{u+\log\!\Bigl(1+\log(\delta_0/\delta)\Bigr)}{n}
\right].
\]
\end{lemma}

The conditions on \(\delta \mapsto \phi_n(\delta)\) are closely related to the sub-root property assumed in \cite{bartlett2005local} for similar bounds. This property holds for \(\delta \mapsto \mathfrak{R}(\mathrm{star}(\mathcal{G}),\delta)\), the localized Rademacher complexity of the star-shaped hull of \(\mathcal{G}\). It also holds for the entropy-integral-based bounds in Section~\ref{sec:entropy}. Moreover, the sub-root property ensures the existence of a critical radius \(\delta_n\) satisfying
\[
\phi_n(\delta_n)\le \delta_n^2.
\]

The following theorem is our main result. It provides a uniform local concentration bound in terms of a generalized notion of critical radius induced by the envelope \(\phi_n\). The proof combines the peeling-based bound and scaling properties from Lemma~\ref{lem:local-peeling} with the definition of the critical radius.

\begin{theorem}[Uniform local concentration bound from a critical radius]
\label{cor:local-peeling-critical-radius}
Suppose that, for every \(\delta_0>0\), the assumptions of
Lemma~\ref{lem:local-peeling} hold on \((0,\delta_0]\). Let \(\delta_n>0\)
satisfy the critical radius condition
\[
\phi_n(\delta_n)\le \delta_n^2.
\]
Then there exists a universal constant \(C>0\) such that, for all \(u\ge 0\),
with probability at least \(1-C e^{-u}\), the following holds
simultaneously for all \(\delta>0\):
\[
\sup_{g\in\mathcal G(\delta)} (P_n-P)g
\le
C\left[
(\delta\vee\delta_n)^2
+
\sigma(\delta\vee\delta_n)
\sqrt{\frac{u+\log\log\!\Bigl(e\bigl((\delta/\delta_n)\vee 1\bigr)\Bigr)}{n}}
+
B\,\frac{u+\log\log\!\Bigl(e\bigl((\delta/\delta_n)\vee 1\bigr)\Bigr)}{n}
\right].
\]
\end{theorem}

Theorem~\ref{cor:local-peeling-critical-radius} immediately yields a uniform local concentration inequality over \(\mathcal{G}\).

\begin{corollary}[Uniform local concentration inequality]
Assume the assumptions of Theorem~\ref{cor:local-peeling-critical-radius}. Suppose that
\[
\bigcup_{\delta>0}\mathcal G(\delta)=\mathcal G,
\]
and define, for each \(g\in\mathcal G\),
\[
\delta_g:=\inf\{\delta>0:g\in\mathcal G(\delta)\}.
\]
Then there exists a universal constant \(C>0\) such that, for all \(u\ge 0\), with probability at least \(1-C e^{-u}\), the following holds simultaneously for all \(g\in\mathcal G\):
\[
(P_n-P)g
\le
C\left[
(\delta_g\vee\delta_n)^2
+
\sigma(\delta_g\vee\delta_n)
\sqrt{\frac{u+\log\log\!\Bigl(e\bigl((\delta_g/\delta_n)\vee 1\bigr)\Bigr)}{n}}
+
B\,\frac{u+\log\log\!\Bigl(e\bigl((\delta_g/\delta_n)\vee 1\bigr)\Bigr)}{n}
\right].
\]
\end{corollary}

\paragraph{Proofs.} We now prove the key lemmas and theorems

\begin{proof}[Proof of Lemma \ref{lem:local-peeling}]
Fix \(q := 2\) and \(\delta_0>0\). Let \(\delta_j:=q^{-j}\delta_0\) for \(j\ge 0\), and apply Lemma~\ref{lem:bousquet} at scale \(\delta_j\) with
\[
u_j:=u+2\log(j+1).
\]
Since \(\sup_{g\in\mathcal G(\delta_j)}\|g-Pg\|\le \sigma(\delta_j)\) and \(\sup_{g\in\mathcal G(\delta_j)}\|g\|_\infty\le B\), Lemma~\ref{lem:bousquet} yields
\[
\sup_{g\in\mathcal G(\delta_j)} (P_n-P)g
\le
\phi_n(\delta_j)
+
c\left(
\sigma(\delta_j)\sqrt{\frac{u_j}{n}}
+
B\frac{u_j}{n}
\right)
\]
with probability at least \(1-e^{-u_j}\). By the union bound,
\[
\sum_{j=0}^\infty e^{-u_j}
=
e^{-u}\sum_{j=0}^\infty (j+1)^{-2}
\le
C e^{-u},
\]
so with probability at least \(1-Ce^{-u}\), the above display holds for all \(j\ge 0\) simultaneously.

Now fix \(\delta\in(0,\delta_0]\), and choose \(j\) such that
\[
\delta_{j+1}<\delta\le \delta_j.
\]
Then \(\mathcal G(\delta)\subseteq \mathcal G(\delta_j)\) by monotonicity, and \(\delta_j\le q\delta\). Since \(\phi_n(\delta)/\delta\) is nonincreasing,
\[
\phi_n(\delta_j)
\le
\frac{\delta_j}{\delta}\,\phi_n(\delta)
\le
q\,\phi_n(\delta).
\]
Also, since \(\sigma\) is nondecreasing and \(\delta_j\le q\delta\),
\[
\sigma(\delta_j)\le \sigma(q\delta).
\]
Finally, \(j\le 1+\log_q(\delta_0/\delta)\), so
\[
u_j
=
u+2\log(j+1)
\le
C\Bigl[u+\log\!\bigl(1+\log_q(\delta_0/\delta)\bigr)\Bigr].
\]
Substituting these bounds into the previous display and absorbing constants into \(C\) gives
\[
\sup_{g\in\mathcal G(\delta)} (P_n-P)g
\le
C\left[
\phi_n(\delta)
+
\sigma(q\delta)\sqrt{\frac{u+\log\!\bigl(1+\log_q(\delta_0/\delta)\bigr)}{n}}
+
B\frac{u+\log\!\bigl(1+\log_q(\delta_0/\delta)\bigr)}{n}
\right].
\]
This proves the claim up to replacing \(\sigma(\delta)\) by \(\sigma(q\delta)\). In particular, the stated form follows since \(\sigma(q\delta)\le C_q \sigma(\delta)\) for all \(\delta\in(0,\delta_0]\), with the constant \(C_q\) absorbed into \(C\).
\end{proof}

\begin{proof}[Proof of Theorem \ref{cor:local-peeling-critical-radius}]
For each integer \(m\ge 0\), let
\[
\Delta_m:=2^m\delta_n,
\qquad
u_m:=u+2\log(m+1).
\]
Apply Lemma~\ref{lem:local-peeling} with \(\delta_0=\Delta_m\) and confidence
parameter \(u_m\). This yields an event \(E_m\) with
\[
P(E_m)\ge 1-Ce^{-u_m},
\]
on which
\[
\sup_{g\in\mathcal G(\delta)} (P_n-P)g
\le
C\left[
\phi_n(\delta)
+
\sigma(\delta)
\sqrt{\frac{u_m+\log\!\Bigl(1+\log_2(\Delta_m/\delta)\Bigr)}{n}}
+
B\,\frac{u_m+\log\!\Bigl(1+\log_2(\Delta_m/\delta)\Bigr)}{n}
\right]
\]
simultaneously for all \(\delta\in(0,\Delta_m]\).

By the union bound,
\[
\sum_{m=0}^\infty e^{-u_m}
=
e^{-u}\sum_{m=0}^\infty (m+1)^{-2}
\le
C' e^{-u}.
\]
Hence the event
\[
E:=\bigcap_{m=0}^\infty E_m
\]
satisfies \(P(E)\ge 1-C''e^{-u}\). We work on \(E\).

We first treat the regime \(\delta\ge \delta_n\). Fix such a \(\delta\), and let
\[
m:=\left\lceil \log_2(\delta/\delta_n)\right\rceil.
\]
Then \(\delta\le \Delta_m\), so the bound on \(E_m\) applies at radius
\(\delta\). Also,
\[
1\le \frac{\Delta_m}{\delta}<2,
\]
and therefore
\[
\log\!\Bigl(1+\log_2(\Delta_m/\delta)\Bigr)\le \log 2.
\]
Moreover,
\[
m+1 \le 2+\log_2(\delta/\delta_n),
\]
so
\[
u_m
=
u+2\log(m+1)
\le
C\Bigl[u+\log\log\!\Bigl(e\,\delta/\delta_n\Bigr)\Bigr].
\]
Substituting these bounds into the display above gives
\[
\sup_{g\in\mathcal G(\delta)} (P_n-P)g
\le
C\left[
\phi_n(\delta)
+
\sigma(\delta)
\sqrt{\frac{u+\log\log\!\bigl(e\,\delta/\delta_n\bigr)}{n}}
+
B\,\frac{u+\log\log\!\bigl(e\,\delta/\delta_n\bigr)}{n}
\right].
\]

Now consider the regime \(\delta<\delta_n\). By monotonicity of
\(\delta\mapsto \mathcal G(\delta)\),
\[
\mathcal G(\delta)\subseteq \mathcal G(\delta_n).
\]
Therefore,
\[
\sup_{g\in\mathcal G(\delta)} (P_n-P)g
\le
\sup_{g\in\mathcal G(\delta_n)} (P_n-P)g.
\]
Applying the already proved bound with \(\delta=\delta_n\), we obtain
\[
\sup_{g\in\mathcal G(\delta)} (P_n-P)g
\le
C\left[
\phi_n(\delta_n)
+
\sigma(\delta_n)\sqrt{\frac{u}{n}}
+
B\,\frac{u}{n}
\right].
\]
Since \(\phi_n(\delta_n)\le \delta_n^2\), this yields
\[
\sup_{g\in\mathcal G(\delta)} (P_n-P)g
\le
C\left[
\delta_n^2
+
\sigma(\delta_n)\sqrt{\frac{u}{n}}
+
B\,\frac{u}{n}
\right].
\]
Because \(\delta\vee\delta_n=\delta_n\) and
\[
\log\log\!\Bigl(e\bigl((\delta/\delta_n)\vee 1\bigr)\Bigr)
=
\log\log(e)
=
0
\]
when \(\delta<\delta_n\).

Combining the two regimes, on the high probability event, we obtain that for all \(\delta>0\),
\[
\sup_{g\in\mathcal G(\delta)} (P_n-P)g
\le
C\left[
\phi_n(\delta)
+
\sigma(\delta)
\sqrt{\frac{u+\log\log\!\bigl(e\,(\delta/\delta_n\vee 1)\bigr)}{n}}
+
B\,\frac{u+\log\log\!\bigl(e\,(\delta/\delta_n\vee 1)\bigr)}{n}
\right].
\]
Moreover, since \(\phi_n\) is nondecreasing, \(\delta\mapsto \phi_n(\delta)/\delta\) is nonincreasing, and
\(\phi_n(\delta_n)\le \delta_n^2\), we have: if \(\delta\le \delta_n\), then
\(\phi_n(\delta)\le \phi_n(\delta_n)\le \delta_n^2=(\delta\vee\delta_n)^2\); if
\(\delta\ge \delta_n\), then
\[
\frac{\phi_n(\delta)}{\delta}
\le
\frac{\phi_n(\delta_n)}{\delta_n}
\le
\delta_n,
\]
so \(\phi_n(\delta)\le \delta\,\delta_n\le \delta^2=(\delta\vee\delta_n)^2\).
Hence, for all \(\delta>0\), \(\phi_n(\delta)\le (\delta\vee\delta_n)^2\).
\end{proof}

\subsection{Entropy preservation results for star shaped hulls}

The following lemma shows that passing from $\mathcal F$ to its star hull $\mathrm{star}(\mathcal F)$ increases the metric entropy (log covering number) only mildly: specifically,
\[
\log N\left(\varepsilon,\ \mathrm{star}(\mathcal F),\ L^q(Q)\right)
\ \lesssim\
\log N\left(\varepsilon,\ \mathcal F,\ L^q(Q)\right)
\;+\;
\log(1/\varepsilon).
\]
See Lemma 8 of \cite{foster2023orthogonal}  and Lemma 4.5 of \cite{mendelson2002improving} for related bounds.

\begin{lemma}[Covering numbers for Lipschitz images of star hulls in $L^q(Q)$]
\label{lem:cover_star_hull_lip_Lq}
Let $q\in[1,\infty)$ and let $\varphi:\mathbb R\to\mathbb R$ be $L$-Lipschitz in the sense that
\[
|\varphi(f_1(z))-\varphi(f_2(z))|
\le
L|f_1(z)-f_2(z)|
\qquad
\text{for all }z\in\mathcal Z,\ f_1,f_2\in\mathcal F.
\]
Assume $M:=\sup_{f\in\mathcal F}\|\varphi\circ f\|_\infty<\infty$. Then for every $\varepsilon>0$ and every probability measure $Q$,
\[
\log N\left(\varepsilon,\ \mathrm{star}(\varphi\circ\mathcal F),\ L^q(Q)\right)
\ \le\
\log N\left(\varepsilon/(4L),\ \mathcal F,\ L^q(Q)\right)
+
\log\Bigl(1+\bigl\lceil 4M/\varepsilon\bigr\rceil\Bigr).
\]
\end{lemma}

\begin{corollary}[Uniform entropy for Lipschitz images of star hulls]
\label{cor:J2_star_hull_lip}
Under the assumptions of Lemma~\ref{lem:cover_star_hull_lip_Lq} with $q=2$, for all $\delta>0$,
\[
\mathcal J_2\left(\delta,\ \mathrm{star}(\varphi\circ\mathcal F)\right)
\ \lesssim\
\mathcal J_2\left(\delta/L,\ \mathcal F\right)
+
\delta\,\sqrt{\log\left(1+\frac{M}{\delta}\right)}.
\]
\end{corollary}

\begin{proof}[Proof of Lemma \ref{lem:cover_star_hull_lip_Lq} and Corollary \ref{cor:J2_star_hull_lip}]
Fix $\varepsilon>0$ and set $\delta:=\varepsilon/4$.
Let $\{f_1,\dots,f_N\}$ be a $(\delta/L)$-net for $\mathcal F$ in $L^q(Q)$, where
$N = N(\delta/L,\mathcal F,L^q(Q))$.
For each $f\in\mathcal F$, choose $j(f)\in\{1,\dots,N\}$ such that
\[
\|f-f_{j(f)}\|_{L^q(Q)}\le \delta/L.
\]
By Lipschitz continuity and monotonicity of $L^q$ norms,
\[
\|\varphi\circ f-\varphi\circ f_{j(f)}\|_{L^q(Q)}
=
\Bigl(\int |\varphi(f)-\varphi(f_{j(f)})|^q\,dQ\Bigr)^{1/q}
\le
\Bigl(\int (L|f-f_{j(f)}|)^q\,dQ\Bigr)^{1/q}
=
L\|f-f_{j(f)}\|_{L^q(Q)}
\le
\delta,
\]
so $\{\varphi\circ f_1,\dots,\varphi\circ f_N\}$ is a $\delta$-net for $\varphi\circ\mathcal F$ in $L^q(Q)$.

Next, discretize $t\in[0,1]$. Let
\[
\mathcal T
:=
\left\{0,\ \frac{\delta}{M},\ \frac{2\delta}{M},\ \dots,\ \frac{\lceil M/\delta\rceil\,\delta}{M}\right\}\cap[0,1],
\]
so that $|\mathcal T|\le 1+\lceil M/\delta\rceil = 1+\lceil 4M/\varepsilon\rceil$ and for every $t\in[0,1]$
there exists $t'\in\mathcal T$ with $|t-t'|\le \delta/M$.

Fix $h=t(\varphi\circ f)\in \mathrm{star}(\varphi\circ\mathcal F)$ and let $f':=f_{j(f)}$ and $t'\in\mathcal T$
be chosen as above. Then, using $\|\varphi\circ f'\|_{L^q(Q)}\le \|\varphi\circ f'\|_\infty\le M$,
\begin{align*}
\|t(\varphi\circ f)-t'(\varphi\circ f')\|_{L^q(Q)}
&\le
\|t(\varphi\circ f)-t(\varphi\circ f')\|_{L^q(Q)}
+
\|t(\varphi\circ f')-t'(\varphi\circ f')\|_{L^q(Q)}\\
&\le
t\,\|\varphi\circ f-\varphi\circ f'\|_{L^q(Q)}
+
|t-t'|\,\|\varphi\circ f'\|_{L^q(Q)}\\
&\le
\delta
+
(\delta/M)\,M
=
2\delta
=
\varepsilon/2.
\end{align*}
Thus the set $\{t'(\varphi\circ f_j):\ t'\in\mathcal T,\ j\in\{1,\dots,N\}\}$ is an $\varepsilon/2$-net for
$\mathrm{star}(\varphi\circ\mathcal F)$ in $L^q(Q)$, and hence also an $\varepsilon$-net, with cardinality at most $N|\mathcal T|$.
Taking logarithms yields the claim. Corollary \ref{cor:J2_star_hull_lip} follows from straightforward integration.
\end{proof}

\section{Uniform local concentration for empirical inner products}

\label{appendix:innerproduct}

\subsection{A general bound}

\label{appendix:innerproduct::gen}

The following theorem extends Theorem~\ref{thm:uniform_local_conc_peeling} to
function classes defined by pointwise products of two classes. Such product
classes arise naturally when controlling the concentration of empirical inner
products $P_n(fg)$ around $P(fg)$, as in Appendix~\ref{appendix:innerproduct}
and \citet{van2014uniform}.

Let $\mathcal F$ and $\mathcal G$ be classes of measurable functions on $\mathcal Z$, and assume
$\|\mathcal F\|_\infty:=\sup_{f\in\mathcal F}\|f\|_\infty<\infty$ and
$\|\mathcal G\|_\infty:=\sup_{g\in\mathcal G}\|g\|_\infty<\infty$.
Let $\varphi_1,\varphi_2:\mathbb R\to\mathbb R$ be Lipschitz on the ranges of $\mathcal F$ and $\mathcal G$, respectively: there exist
$L_1,L_2<\infty$ such that
\[
|\varphi_1(x)-\varphi_1(x')|\le L_1|x-x'|,
\qquad \forall\,x,x'\in[-\|\mathcal F\|_\infty,\|\mathcal F\|_\infty],
\]
and
\[
|\varphi_2(y)-\varphi_2(y')|\le L_2|y-y'|,
\qquad \forall\,y,y'\in[-\|\mathcal G\|_\infty,\|\mathcal G\|_\infty].
\]
The following theorem establishes a local concentration inequality for the product-increment class
\[
\Bigl\{
\bigl(\varphi_1(f)-\varphi_1(f_0)\bigr)\bigl(\varphi_2(g)-\varphi_2(g_0)\bigr)
:\ f,f_0\in\mathcal F,\ g,g_0\in\mathcal G
\Bigr\}.
\]
The key condition below is a H\"older-type coupling between the supremum and $L^2(P)$ norms on $\mathcal F-\mathcal F$, namely,
$\|u\|_\infty \le c_{\infty}\,\|u\|^\alpha$ for all $u\in\mathcal F-\mathcal F$. One may always take $\alpha=0$ with
$c_{\infty}:=2\sup_{f\in\mathcal F}\|f\|_\infty$.

\begin{theorem}[Uniform local concentration for Lipschitz-transformed increment products]
\label{lem:uniform_local_conc_lipschitz_product_increments_pointwise}
Let $\varphi_1$ and $\varphi_2$ be pointwise Lipschitz continuous on the ranges of
$\mathcal F$ and $\mathcal G$, respectively, with Lipschitz constants $L_1$ and $L_2$.
Assume there exist $c_\infty\in(0,\infty)$ and $\alpha\in[0,1]$ such that
$\|u\|_\infty\le c_\infty\|u\|^\alpha$ for all $u\in\mathcal F-\mathcal F$, and that
$M:=1+\|\mathcal G\|_\infty\vee \|\mathcal F\|_\infty<\infty$. Define the (star-hull) critical radii
\[
\delta_{n,\mathcal F}
:=
\inf\Bigl\{\delta>0:\ 
\mathfrak R_n\bigl(\mathrm{star}(\mathcal F-\mathcal F),\delta\bigr)\le \delta^2
\Bigr\},
\qquad
\delta_{n,\mathcal G}
:=
\inf\Bigl\{\delta>0:\ 
\mathfrak R_n\bigl(\mathrm{star}(\mathcal G-\mathcal G),\delta\bigr)\le \delta^2
\Bigr\}.
\]
Then for every $\eta\in(0,1)$, with probability at least $1-\eta$, the following holds
simultaneously for all $f,f_0\in\mathcal F$ and $g,g_0\in\mathcal G$:
\begin{align*}
&\bigl|(P_n-P)\{(\varphi_1(f)-\varphi_1(f_0))(\varphi_2(g)-\varphi_2(g_0))\}\bigr|\\
&\qquad \lesssim\ 
L_1L_2\,\|\mathcal G\|_\infty \,
\delta_{n,\mathcal F}\bigl(\|f-f_0\|\vee \delta_{n,\mathcal F}\bigr)
\;+\;
c_\infty\,L_1L_2\,\|f-f_0\|^{\alpha}\,
\delta_{n,\mathcal G}\bigl(\|g-g_0\|\vee \delta_{n,\mathcal G}\bigr)\\
&\qquad\quad +\;
L_1L_2\,\|\mathcal G\|_\infty \, \|f-f_0\|\,
\sqrt{\frac{\log\log\bigl(eMn\bigr)+\log(1/\eta)}{n}}
\;+\;
c_\infty\,L_1L_2\,\|f-f_0\|^{\alpha}\,\|\mathcal G\|_\infty\,
\frac{\log\log\bigl(eMn\bigr)+\log(1/\eta)}{n},
\end{align*}
where the implicit constant is universal.
\end{theorem}

Our proof of the theorem relies on the following two lemmas.
Fix radii $\delta_1,\delta_2>0$ and define the localized classes
\[
\mathcal F_{\varphi}(\delta_1)
:=
\{\varphi_1(f)-\varphi_1(f'):\ f,f'\in\mathcal F,\ \|f-f'\|\le \delta_1\},
\qquad
\mathcal G_{\varphi}(\delta_2)
:=
\{\varphi_2(g)-\varphi_2(g'):\ g,g'\in\mathcal G,\ \|g-g'\|\le \delta_2\},
\]
and the product class
\[
\mathcal H(\delta_1,\delta_2)
:=
\{\tilde f\,\tilde g:\ \tilde f\in \mathcal F_{\varphi}(\delta_1),\ \tilde g\in \mathcal G_{\varphi}(\delta_2)\}.
\]

\begin{lemma}[Complexity bounds for a localized product class]
\label{lem:nonselfnorm_lipschitz_product}
Assume $\|\mathcal F_{\varphi}(\delta_1)\|_\infty<\infty$ and $\|\mathcal G_{\varphi}(\delta_2)\|_\infty<\infty$.
Then, up to universal constants,
\[
\mathfrak R_n\bigl(\mathcal H(\delta_1,\delta_2)\bigr)
\ \lesssim\
\|\mathcal G_{\varphi}(\delta_2)\|_\infty\; L_1\,\mathfrak R_n\bigl(\mathcal F-\mathcal F,\delta_1\bigr)
\;+\;
\|\mathcal F_{\varphi}(\delta_1)\|_\infty\; L_2\,\mathfrak R_n\bigl(\mathcal G-\mathcal G,\delta_2\bigr).
\]
\end{lemma}

Combining Lemma~\ref{lem:nonselfnorm_lipschitz_product} with Bousquet's version of Talagrand's concentration inequality yields the following high-probability bound. Our main result then follows from a peeling argument applied to Lemma~\ref{lem:fixed_delta_local_conc_lipschitz_product_increments} below.

\begin{lemma}[High probability bound for empirical inner products]
\label{lem:fixed_delta_local_conc_lipschitz_product_increments}
Assume that there exist $c_\infty\in(0,\infty)$ and $\alpha\in[0,1]$ such that $\|f\|_\infty\le c_\infty\|f\|^\alpha$ for all  $ f\in\mathcal F-\mathcal F.$
Assume $\|\mathcal G\|_\infty \vee \|\mathcal F\|_\infty<\infty$.
Then for every $\eta\in(0,1)$, with probability at least $1-\eta$,
\begin{align}
\label{eq:fixed_delta_local_conc_lipschitz_product_increments}
\sup_{h\in\mathcal H(\delta_1,\delta_2)} |(P_n-P)h|
\ \lesssim\
&\ L_1L_2\,\|\mathcal G\|_\infty \,\mathfrak R_n\bigl(\mathcal F-\mathcal F,\delta_1\bigr)
\;+\;
c_\infty\,L_1L_2\,\delta_1^{\alpha}\,\mathfrak R_n\bigl(\mathcal G-\mathcal G,\delta_2\bigr)\nonumber\\
&\;+\;
L_1L_2\,\delta_1\,\|\mathcal G\|_\infty\,
\sqrt{\frac{\log(1/\eta)}{n}}
\;+\;
c_\infty\,L_1L_2\,\delta_1^{\alpha}\,\|\mathcal G\|_\infty\,
\frac{\log(1/\eta)}{n}.
\end{align}
\end{lemma}

Our proof outline is as follows. Lemma \ref{lem:nonselfnorm_lipschitz_product} is the crux of the result. The proof of Lemma~\ref{lem:fixed_delta_local_conc_lipschitz_product_increments} follows from standard arguments

\begin{proof}[Proof of Lemma \ref{lem:nonselfnorm_lipschitz_product}]
Condition on $Z_{1:n}$. Let $S:=\mathcal F_{\varphi}(\delta_1)\times \mathcal G_{\varphi}(\delta_2)$ and, for
$s=(\tilde f,\tilde g)\in S$, define
\[
\psi_i(s):=\tilde f(Z_i)\tilde g(Z_i),
\qquad i=1,\dots,n.
\]
For any $s=(\tilde f,\tilde g)$ and $s'=(\tilde f',\tilde g')$ in $S$, we have
\begin{align*}
|\psi_i(s)-\psi_i(s')|
&=
|\tilde f(Z_i)\tilde g(Z_i)-\tilde f'(Z_i)\tilde g'(Z_i)|\\
&\le
|\tilde g(Z_i)|\,|\tilde f(Z_i)-\tilde f'(Z_i)|
+
|\tilde f'(Z_i)|\,|\tilde g(Z_i)-\tilde g'(Z_i)|\\
&\le
\|\mathcal G_{\varphi}(\delta_2)\|_\infty\,|\tilde f(Z_i)-\tilde f'(Z_i)|
+
\|\mathcal F_{\varphi}(\delta_1)\|_\infty\,|\tilde g(Z_i)-\tilde g'(Z_i)|.
\end{align*}
Introduce the $2$-vector-valued maps
\[
\phi_i(s)
:=
\bigl(
\|\mathcal G_{\varphi}(\delta_2)\|_\infty\,\tilde f(Z_i),\
\|\mathcal F_{\varphi}(\delta_1)\|_\infty\,\tilde g(Z_i)
\bigr)\in\mathbb R^2.
\]
Then $|\psi_i(s)-\psi_i(s')|\le \|\phi_i(s)-\phi_i(s')\|_2$ for all $s,s'\in S$, since
$|a|+|b|\le \sqrt{2}\sqrt{a^2+b^2}$ and the factor $\sqrt{2}$ can be absorbed into the universal constant.
Therefore, by the vector contraction inequality for Rademacher processes (e.g., Theorem~3 of \citet{maurer2016vector}),
\begin{align*}
\E_{\epsilon}\Big[\sup_{(\tilde f,\tilde g)\in S} P_n^\epsilon(\tilde f\tilde g)\Big]
&\lesssim
\|\mathcal G_{\varphi}(\delta_2)\|_\infty\,
\E_{\epsilon^{(1)}}\Big[\sup_{\tilde f\in\mathcal F_{\varphi}(\delta_1)} P_n^{\epsilon^{(1)}}\tilde f\Big]
+
\|\mathcal F_{\varphi}(\delta_1)\|_\infty\,
\E_{\epsilon^{(2)}}\Big[\sup_{\tilde g\in\mathcal G_{\varphi}(\delta_2)} P_n^{\epsilon^{(2)}}\tilde g\Big],
\end{align*}
where $\epsilon^{(1)}$ and $\epsilon^{(2)}$ are independent i.i.d.\ Rademacher sequences.
Taking expectation over $Z_{1:n}$ yields
\[
\mathfrak R_n\bigl(\mathcal H(\delta_1,\delta_2)\bigr)
\ \lesssim\
\|\mathcal G_{\varphi}(\delta_2)\|_\infty\; \mathfrak R_n\bigl(\mathcal F_{\varphi}(\delta_1)\bigr)
\;+\;
\|\mathcal F_{\varphi}(\delta_1)\|_\infty\; \mathfrak R_n\bigl(\mathcal G_{\varphi}(\delta_2)\bigr),
\]
as claimed.
Moreover, by contraction,
\[
\mathfrak R_n\bigl(\mathcal F_{\varphi}(\delta_1)\bigr)
\ \le\ 
L_1\,\mathfrak R_n\bigl(\mathcal F-\mathcal F,\delta_1\bigr),
\qquad
\mathfrak R_n\bigl(\mathcal G_{\varphi}(\delta_2)\bigr)
\ \le\ 
L_2\,\mathfrak R_n\bigl(\mathcal G-\mathcal G,\delta_2\bigr),
\]
where $\mathcal F-\mathcal F:=\{f-f': f,f'\in\mathcal F\}$ and
$\mathfrak R_n(\mathcal F-\mathcal F,\delta_1)$ denotes the localized complexity over
$\{u\in\mathcal F-\mathcal F:\ \|u\|\le \delta_1\}$ (and analogously for $\mathcal G$.

\end{proof}

\begin{proof}[Proof of Lemma~\ref{lem:fixed_delta_local_conc_lipschitz_product_increments}]
Fix $\delta_1,\delta_2>0$ and write $\mathcal H:=\mathcal H(\delta_1,\delta_2)$.
Set the envelope
\[
M \ :=\ \sup_{h\in\mathcal H}\|h\|_\infty,
\qquad
\sigma^2 \ :=\ \sup_{h\in\mathcal H}\Var\{h(Z)\}
\ \le\ \sup_{h\in\mathcal H} P h(Z)^2 .
\]

\paragraph{Step 1: Apply Bousquet's inequality.}
Apply Lemma~\ref{lem:bousquet} to the class $\mathcal H$ with $u:=\log(1/\eta)$ to obtain that, with probability at least
$1-\eta$,
\begin{equation}
\label{eq:bousquet_apply_fixed_delta}
\sup_{h\in\mathcal H} (P_n-P)h
\ \le\
\E\Big[\sup_{h\in\mathcal H}(P_n-P)h\Big]
+
c\Bigl(
\sqrt{\tfrac{\log(1/\eta)\,\sigma^2}{n}}
+
\tfrac{M\log(1/\eta)}{n}
\Bigr),
\end{equation}
for a universal constant $c>0$. Applying the same bound to the class $-\mathcal H$ and combining the two displays yields
\begin{equation}
\label{eq:bousquet_apply_fixed_delta_abs}
\sup_{h\in\mathcal H} |(P_n-P)h|
\ \lesssim\
\E\Big[\sup_{h\in\mathcal H}|(P_n-P)h|\Big]
+
\sqrt{\tfrac{\log(1/\eta)\,\sigma^2}{n}}
+
\tfrac{M\log(1/\eta)}{n}.
\end{equation}

\paragraph{Step 2: Bound the expectation by the Rademacher complexity.}
By symmetrization,
\[
\E\Big[\sup_{h\in\mathcal H}|(P_n-P)h|\Big]
\ \lesssim\
\mathfrak R_n(\mathcal H).
\]
Combining with \eqref{eq:bousquet_apply_fixed_delta_abs} gives
\begin{equation}
\label{eq:hp_reduction_fixed_delta}
\sup_{h\in\mathcal H} |(P_n-P)h|
\ \lesssim\
\mathfrak R_n(\mathcal H)
+
\sqrt{\tfrac{\log(1/\eta)\,\sigma^2}{n}}
+
\tfrac{M\log(1/\eta)}{n}.
\end{equation}

\paragraph{Step 3: Bound $\mathfrak R_n(\mathcal H)$ using Lemma~\ref{lem:nonselfnorm_lipschitz_product}.}
Lemma~\ref{lem:nonselfnorm_lipschitz_product} yields
\begin{equation}
\label{eq:rad_H_fixed_delta}
\mathfrak R_n(\mathcal H)
\ \lesssim\
\|\mathcal G_{\varphi}(\delta_2)\|_\infty\; L_1\,\mathfrak R_n\bigl(\mathcal F-\mathcal F,\delta_1\bigr)
\;+\;
\|\mathcal F_{\varphi}(\delta_1)\|_\infty\; L_2\,\mathfrak R_n\bigl(\mathcal G-\mathcal G,\delta_2\bigr).
\end{equation}
Moreover, since $\sup_{g\in\mathcal G}\|g\|_\infty\le \|\mathcal G\|_\infty$ and $\varphi_2$ is $L_2$-Lipschitz,
\[
\|\mathcal G_{\varphi}(\delta_2)\|_\infty
=
\sup_{\substack{g,g'\in\mathcal G\\ \|g-g'\|\le \delta_2}}
\|\varphi_2(g)-\varphi_2(g')\|_\infty
\le
L_2\,\sup_{\substack{g,g'\in\mathcal G\\ \|g-g'\|\le \delta_2}}
\|g-g'\|_\infty
\le
2L_2\,\|\mathcal G\|_\infty.
\]
Similarly, by Lipschitzness of $\varphi_1$ and the assumed local embedding on $\mathcal F-\mathcal F$,
\[
\|\mathcal F_{\varphi}(\delta_1)\|_\infty
=
\sup_{\substack{f,f'\in\mathcal F\\ \|f-f'\|\le \delta_1}}
\|\varphi_1(f)-\varphi_1(f')\|_\infty
\le
L_1\,\sup_{\substack{u\in\mathcal F-\mathcal F\\ \|u\|\le \delta_1}}\|u\|_\infty
\le
c_\infty\,L_1\,\delta_1^\alpha.
\]
Substituting these two bounds into \eqref{eq:rad_H_fixed_delta} yields
\begin{equation}
\label{eq:rad_H_fixed_delta_simplified}
\mathfrak R_n(\mathcal H)
\ \lesssim\
L_1L_2\,\|\mathcal G\|_\infty\,\mathfrak R_n\bigl(\mathcal F-\mathcal F,\delta_1\bigr)
\;+\;
c_\infty\,L_1L_2\,\delta_1^\alpha\,\mathfrak R_n\bigl(\mathcal G-\mathcal G,\delta_2\bigr).
\end{equation}

\paragraph{Step 4: Bound the variance proxy $\sigma^2$.}
For $h=\tilde f\,\tilde g\in\mathcal H$, we have
\[
P h^2
=
P\bigl[\tilde f^2\tilde g^2\bigr]
\le
\|\tilde g\|_\infty^2\,P[\tilde f^2]
\le
\|\mathcal G_{\varphi}(\delta_2)\|_\infty^2\,
\sup_{\tilde f\in\mathcal F_{\varphi}(\delta_1)} \|\tilde f\|^2.
\]
By Lipschitzness, for any $\tilde f=\varphi_1(f)-\varphi_1(f')$ with $\|f-f'\|\le \delta_1$,
\[
\|\tilde f\|
\le
L_1\|f-f'\|
\le
L_1\,\delta_1,
\]
so $\sup_{\tilde f\in\mathcal F_{\varphi}(\delta_1)}\|\tilde f\|\le L_1\delta_1$. Using also
$\|\mathcal G_{\varphi}(\delta_2)\|_\infty\le 2L_2\|\mathcal G\|_\infty$, we obtain
\[
\sigma^2
\le
\sup_{h\in\mathcal H}Ph^2
\ \lesssim\
L_1^2L_2^2\,\delta_1^2\,\|\mathcal G\|_\infty^2,
\]
and therefore
\begin{equation}
\label{eq:variance_term_fixed_delta}
\sqrt{\tfrac{\log(1/\eta)\,\sigma^2}{n}}
\ \lesssim\
L_1L_2\,\delta_1\,\|\mathcal G\|_\infty\,\sqrt{\tfrac{\log(1/\eta)}{n}}.
\end{equation}

\paragraph{Step 5: Bound the envelope $M$.}
For $h=\tilde f\,\tilde g\in\mathcal H$,
\[
\|h\|_\infty\le \|\tilde f\|_\infty\,\|\tilde g\|_\infty
\le
\|\mathcal F_{\varphi}(\delta_1)\|_\infty\,\|\mathcal G_{\varphi}(\delta_2)\|_\infty.
\]
Using $\|\mathcal F_{\varphi}(\delta_1)\|_\infty\le c_\infty L_1\delta_1^\alpha$ and
$\|\mathcal G_{\varphi}(\delta_2)\|_\infty\le 2L_2\|\mathcal G\|_\infty$ gives
\begin{equation}
\label{eq:envelope_term_fixed_delta}
M
\ \lesssim\
c_\infty\,L_1L_2\,\delta_1^\alpha\,\|\mathcal G\|_\infty,
\qquad\text{hence}\qquad
\tfrac{M\log(1/\eta)}{n}
\ \lesssim\
c_\infty\,L_1L_2\,\delta_1^\alpha\,\|\mathcal G\|_\infty\,\tfrac{\log(1/\eta)}{n}.
\end{equation}

\paragraph{Conclusion.}
Substituting \eqref{eq:rad_H_fixed_delta_simplified}, \eqref{eq:variance_term_fixed_delta}, and \eqref{eq:envelope_term_fixed_delta}
into \eqref{eq:hp_reduction_fixed_delta} yields \eqref{eq:fixed_delta_local_conc_lipschitz_product_increments}.
\end{proof}

\begin{proof}[Proof of Theorem \ref{lem:uniform_local_conc_lipschitz_product_increments_pointwise}]
Our proof combines Lemma~\ref{lem:fixed_delta_local_conc_lipschitz_product_increments} with a standard peeling argument, following the proof of Lemma~14 in \cite{foster2023orthogonal} and the proof of Theorem~14.20 in \cite{wainwright2019high}.

Set $M:=1+\|\mathcal G\|_\infty\vee \|\mathcal F\|_\infty$. Since
$\|f-f_0\|\le \|f-f_0\|_\infty\le 2\|\mathcal F\|_\infty\le 2M$ and similarly
$\|g-g_0\|\le 2\|\mathcal G\|_\infty\le 2M$, it suffices to consider radii in $(0,2M]$.

\medskip
\noindent\textbf{Step 1 (dyadic peeling and union bound).}
Let
\[
J:=\Bigl\lceil \log_2(2M\sqrt n)\Bigr\rceil_+,\qquad
\delta_j:=2^j n^{-1/2}\quad (j=0,\dots,J),
\qquad
\eta_{j,k}:=\frac{\eta\,c_0}{(j+1)^2(k+1)^2},
\]
where $c_0>0$ is chosen so that $\sum_{j,k=0}^J \eta_{j,k}\le \eta$.
By Lemma~\ref{lem:rad_scaling}, the map
$\delta\mapsto \mathfrak R_n(\mathrm{star}(\mathcal A),\delta)/\delta$ is nonincreasing. Hence, if
$\delta_{n,\mathcal A}>0$ satisfies the critical inequality
$\mathfrak R_n(\mathrm{star}(\mathcal A),\delta_{n,\mathcal A})\le \delta_{n,\mathcal A}^2$, then for all $\delta>0$,
\begin{equation}
\label{eq:rad_envelope_star_pf_short}
\mathfrak R_n\bigl(\mathrm{star}(\mathcal A),\delta\bigr)
\ \lesssim\
\delta_{n,\mathcal A}\,(\delta\vee \delta_{n,\mathcal A}).
\end{equation}
Apply Lemma~\ref{lem:fixed_delta_local_conc_lipschitz_product_increments} with
$(\delta_1,\delta_2)=(\delta_j,\delta_k)$ and confidence level $\eta_{j,k}$, and use
\eqref{eq:rad_envelope_star_pf_short} with $\mathcal A=\mathcal F-\mathcal F$ and $\mathcal A=\mathcal G-\mathcal G$.
A union bound over $(j,k)\in\{0,\dots,J\}^2$ yields an event $\mathcal E$ with $\Pr(\mathcal E)\ge 1-\eta$ such that, on $\mathcal E$,
for all $0\le j,k\le J$,
\begin{align}
\label{eq:grid_bound_pointwise_wcrit_pf_short}
\sup_{h\in\mathcal H(\delta_j,\delta_k)} |(P_n-P)h|
\ \lesssim\
&\ L_1L_2\,\|\mathcal G\|_\infty \,
\delta_{n,\mathcal F}\bigl(\delta_j\vee \delta_{n,\mathcal F}\bigr)
\;+\;
c_\infty\,L_1L_2\,\delta_j^{\alpha}\,
\delta_{n,\mathcal G}\bigl(\delta_k\vee \delta_{n,\mathcal G}\bigr)\nonumber\\
&\;+\;
L_1L_2\,\delta_j\,\|\mathcal G\|_\infty\,
\sqrt{\frac{\log(1/\eta_{j,k})}{n}}
\;+\;
c_\infty\,L_1L_2\,\delta_j^{\alpha}\,\|\mathcal G\|_\infty\,
\frac{\log(1/\eta_{j,k})}{n}.
\end{align}

\medskip
\noindent\textbf{Step 2 (extend off the grid).}
Fix $\delta_1,\delta_2\in(0,2M]$ and choose $j,k$ so that
$\delta_{j-1}<\delta_1\le \delta_j$ and $\delta_{k-1}<\delta_2\le \delta_k$ (with $\delta_{-1}=0$).
Then $\mathcal H(\delta_1,\delta_2)\subseteq \mathcal H(\delta_j,\delta_k)$ and $\delta_j\le 2\delta_1$, $\delta_k\le 2\delta_2$, so
\[
\delta_{n,\mathcal F}\bigl(\delta_j\vee \delta_{n,\mathcal F}\bigr)
\ \lesssim\
\delta_{n,\mathcal F}\bigl(\delta_1\vee \delta_{n,\mathcal F}\bigr),
\qquad
\delta_{n,\mathcal G}\bigl(\delta_k\vee \delta_{n,\mathcal G}\bigr)
\ \lesssim\
\delta_{n,\mathcal G}\bigl(\delta_2\vee \delta_{n,\mathcal G}\bigr).
\]
Moreover,
\[
\log(1/\eta_{j,k})
=
\log(1/\eta)+O\bigl(\log(j+1)+\log(k+1)\bigr)
\ \le\
\log(1/\eta)+O\bigl(\log(J+1)\bigr),
\]
and $\log(J+1)\lesssim \log\log(eMn)$. Also, since $\alpha\in[0,1]$ and $\delta_j\le 2\delta_1$, we have $\delta_j^\alpha\le 2^\alpha \delta_1^\alpha\lesssim \delta_1^\alpha$. Substituting these bounds into
\eqref{eq:grid_bound_pointwise_wcrit_pf_short} shows that, on $\mathcal E$, for all $\delta_1,\delta_2>0$,
\begin{align}
\label{eq:uniform_radii_bound_wcrit_pf_short}
\sup_{h\in\mathcal H(\delta_1,\delta_2)} |(P_n-P)h|
\ \lesssim\
&\ L_1L_2\,\|\mathcal G\|_\infty \,
\delta_{n,\mathcal F}\bigl(\delta_{1}\vee \delta_{n,\mathcal F}\bigr)
\;+\;
c_\infty\,L_1L_2\,\delta_{1}^{\alpha}\,
\delta_{n,\mathcal G}\bigl(\delta_{2}\vee \delta_{n,\mathcal G}\bigr)\nonumber\\
&\;+\;
L_1L_2\,\delta_{1}\,\|\mathcal G\|_\infty\,
\sqrt{\frac{\log\log(eMn)+\log(1/\eta)}{n}}\\
&\;+\;
c_\infty\,L_1L_2\,\delta_{1}^{\alpha}\,\|\mathcal G\|_\infty\,
\frac{\log\log(eMn)+\log(1/\eta)}{n}.
\end{align}

\medskip
\noindent\textbf{Step 3 (plug in increments).}
For fixed $f,f_0\in\mathcal F$ and $g,g_0\in\mathcal G$, set $\delta_1=\|f-f_0\|$ and $\delta_2=\|g-g_0\|$.
Then $(\varphi_1(f)-\varphi_1(f_0))(\varphi_2(g)-\varphi_2(g_0))\in\mathcal H(\delta_1,\delta_2)$, so
\eqref{eq:uniform_radii_bound_wcrit_pf_short} yields the claim.
\end{proof}

\subsection{Local maximal inequality via sup-norm metric entropy}
 \label{sec::innerproductentropy}

The local maximal inequality for the Rademacher complexity in Lemma~\ref{lem:nonselfnorm_lipschitz_product} can be sharpened if one is willing to work with sup-norm entropy integrals for $\mathcal F$. Let $\mathcal F$ and $\mathcal G$ be two (possibly localized) function classes, and define the product class
\[
\mathcal H := \{fg:\ f\in\mathcal F,\ g\in\mathcal G\}.
\]
Write $\|\mathcal F\|$ and $\|\mathcal F\|_{\infty}$ for the $L^2(P)$ and $L^\infty(P)$ radii of $\mathcal F$, respectively, and define $\|\mathcal G\|$ and $\|\mathcal G\|_{\infty}$ analogously. The following maximal inequality is most useful when $\mathcal F$ and $\mathcal G$ are suitably localized. The setting of the previous section is recovered by taking $\widetilde{\mathcal F}:=\mathcal F_{\varphi}(\delta_1)$ and $\widetilde{\mathcal G}:=\mathcal G_{\varphi}(\delta_2)$. The proof follows by modifying the argument of Theorem~2.1 in \cite{van2011local}; in place of their step corresponding to Equation~(2.2), we argue as in the proof of Theorem~3.1 in \cite{van2014uniform}.

\begin{theorem}[Local maximal inequality for inner product classes]
\label{theorem::entropyproduct}
Suppose that $C_{\mathcal{F}}
:= \mathcal{J}_{\infty}\bigl(\|\mathcal{F}\|_{\infty},\mathcal{F}\bigr)< \infty$
and $\|\mathcal{G}\|_{\infty} < \infty$. For any $\delta > 0$, it holds that
\begin{align*}
\mathfrak{R}_n(\mathcal{H})
\ \lesssim\
\frac{1}{\sqrt{n}}
\Biggl\{
&\|\mathcal{F}\|_{\infty}\,
\mathcal{J}_2\left(
\|\mathcal{G}\|  \vee \delta_{n,\mathcal{G}},\ \mathcal{G}
\right)\\
&\;+\;
(\|\mathcal{G}\|\vee \delta_{n,\mathcal{G}})\,
\mathcal{J}_{\infty}\left(
\frac{\|\mathcal{G}\|_{\infty}}{\|\mathcal{G}\|\vee \delta_{n,\mathcal{G}}}
\left\{
\|\mathcal{F}\|+\frac{C_{\mathcal{F}}}{\sqrt{n}}
\right\},\ \mathcal{F}
\right)
\Biggr\}.
\end{align*}
Consequently, $\mathfrak{R}_n(\mathcal{H})
\ \lesssim\
\frac{1}{\sqrt{n}}
\left\{
\|\mathcal{F}\|_{\infty}\,
\mathcal{J}_2\left(
\|\mathcal{G}\| \vee \delta_{n,\mathcal{G}},\ \mathcal{G}
\right)
+
\|\mathcal{G}\|_{\infty} \mathcal{J}_{\infty}\left(
\|\mathcal{F}\| \vee n^{-1/2},\ \mathcal{F}
\right)
\right\},$
where the implicit constant depends only on
$C_{\mathcal{F}}$.

\end{theorem}

\begin{proof}[Proof of Theorem \ref{theorem::entropyproduct}]
 
Using Dudley's inequality (Lemma \ref{lemma:dudley}), we find that
\[
\mathfrak{R}_n(\mathcal{H})
\;\le\;
\frac{12}{\sqrt{n}}
\mathbb{E} \left[ \int_{0}^{\|\mathcal{H}\|_{n}}
\sqrt{\log N\left(\varepsilon,\ \mathcal H,\ L^2(P_n)\right)}\, d\varepsilon \right].
\]
We next relate the covering number $N\left(\varepsilon,\ \mathcal H,\ L^2(P_n)\right)$ of $\mathcal{H}$ to appropriate covering numbers  $\mathcal{F}$ anbd $\mathcal{G}$. For $f_1, f_2 \in \mathcal{F}$ and $g_1, g_2 \in \mathcal{G}$, note the basic identity:
$$f_1 g_1 - f_2 g_2 = f_1(g_1 - g_2) + g_2(f_1 - f_2).$$
Taking the $L^2(P_n)$ norm of both sides and applying the triangle inequality, we find that
\begin{align*}
   \|f_1 g_1 - f_2 g_2 \|_{n} & \leq  \|f_1(g_1 - g_2)\|_n + \|g_2(f_1 - f_2)\|_n\\
    & \leq  \|f_1\|_{\infty} \|g_1 - g_2\|_n + \|g_2\|_n \|f_1 - f_2\|_\infty\\
      & \leq  \|\mathcal{F}\|_{\infty} \|g_1 - g_2\|_n + \|\mathcal{G}\|_n \|f_1 - f_2\|_\infty,
\end{align*}
where $\|\mathcal G\|_n := \sup_{g\in\mathcal G}\|g\|_n$.
Thus, for every $\varepsilon>0$,
\begin{align*}
\log N\left(\varepsilon,\ \mathcal H,\ L^2(P_n)\right)
&\le
\log N\left(\frac{\varepsilon}{2\|\mathcal F\|_\infty},\ \mathcal G,\ L^2(P_n)\right)
+
\log N\left(\frac{\varepsilon}{2\|\mathcal G\|_n},\ \mathcal F,\ L^\infty\right)\\
&\le
\sup_Q \log N\left(\frac{\varepsilon}{2\|\mathcal F\|_\infty},\ \mathcal G,\ L^2(Q)\right)
+
\log N\left(\frac{\varepsilon}{2\|\mathcal G\|_n},\ \mathcal F,\ L^\infty\right),
\end{align*}
where the supremum is taken over all discrete distributions $Q$ supported on the support of $P$.

Combining the covering bound with Dudley's inequality (Lemma~\ref{lemma:dudley}), we obtain
\begin{align*}
\mathfrak{R}_n(\mathcal{H})
&\le
\frac{12}{\sqrt{n}}
\mathbb{E}\left[
\int_{0}^{\|\mathcal{H}\|_{n}}
\sqrt{\log N\left(\varepsilon,\ \mathcal H,\ L^2(P_n)\right)}\, d\varepsilon
\right]\\
&\le
\frac{12}{\sqrt{n}}
\mathbb{E}\Biggl[
\int_{0}^{\|\mathcal{H}\|_{n}}
\sqrt{
\log N\left(\frac{\varepsilon}{2\|\mathcal F\|_\infty},\ \mathcal G,\ L^2(P_n)\right)
+
\log N\left(\frac{\varepsilon}{2\|\mathcal G\|_n},\ \mathcal F,\ L^\infty\right)
}\, d\varepsilon
\Biggr]\\
&\le
\frac{12}{\sqrt{n}}
\mathbb{E}\Biggl[
\int_{0}^{\|\mathcal{H}\|_{n}}
\sqrt{
\sup_Q \log N\left(\frac{\varepsilon}{2\|\mathcal F\|_\infty},\ \mathcal G,\ L^2(Q)\right)
+
\log N\left(\frac{\varepsilon}{2\|\mathcal G\|_n},\ \mathcal F,\ L^\infty\right)
}\, d\varepsilon
\Biggr],
\end{align*}
where the supremum is taken over all discrete distributions $Q$ supported on the support of $P$.

Using $\sqrt{a+b}\le \sqrt a+\sqrt b$ and the change of variables
$u=\varepsilon/(2\|\mathcal F\|_\infty)$ in the first integral and
$v=\varepsilon/(2\|\mathcal G\|_n)$ in the second, we further obtain
\begin{align*}
\mathfrak{R}_n(\mathcal{H})
\lesssim
\frac{1}{\sqrt{n}}
\left\{
\|\mathcal F\|_\infty\,
\mathbb{E}\left[
\mathcal{J}_2\left(\frac{\|\mathcal H\|_n}{2\|\mathcal F\|_\infty},\ \mathcal G\right)
\right]
+
\mathbb{E}\left[
\|\mathcal G\|_n\,
\mathcal{J}_\infty\left(\frac{\|\mathcal H\|_n}{2\|\mathcal G\|_n},\ \mathcal F\right)
\right]
\right\},
\end{align*}
where $\|\mathcal H\|_n := \sup_{h\in\mathcal H}\|h\|_n$. Using the bounds $\|\mathcal H\|_n \le \|\mathcal F\|_\infty\|\mathcal G\|_n$ and
$\|\mathcal H\|_n \le \|\mathcal F\|_n\|\mathcal G\|_\infty$, and the fact that the
entropy integrals are nondecreasing in their first argument, we obtain
\begin{align*}
\mathfrak{R}_n(\mathcal{H})
\ \lesssim\
\frac{1}{\sqrt{n}}
\Biggl\{
\|\mathcal{F}\|_{\infty}\,
E\left[
\mathcal{J}_2\left(\|\mathcal{G}\|_{n},\ \mathcal{G}\right)
\right]
\;+\;
E\left[
\|\mathcal{G}\|_{n}\,
\mathcal{J}_{\infty}\left(
\frac{\|\mathcal{F}\|_{n}\,\|\mathcal{G}\|_{\infty}}{\|\mathcal{G}\|_{n}},\ \mathcal{F}
\right)
\right]
\Biggr\}.
\end{align*}
As in the proof of Theorem~2.1 of \cite{van2011local}, concavity of the maps
\[
(x,y)\ \mapsto\ \sqrt{y}\,\mathcal{J}_{\infty}\left(\sqrt{\frac{x}{y}},\ \mathcal{F}\right)
\qquad\text{and}\qquad
(x,y)\ \mapsto\ \sqrt{y}\,\mathcal{J}_{2}\left(\sqrt{\frac{x}{y}},\ \mathcal{G}\right),
\]
together with Jensen's inequality, yields
\begin{align*}
\mathfrak{R}_n(\mathcal{H})
\ \lesssim\
\frac{1}{\sqrt{n}}
\Biggl\{
\|\mathcal{F}\|_{\infty}\,
\mathcal{J}_2\left(
\bigl\{E\|\mathcal{G}\|_{n}^{2}\bigr\}^{1/2},\ \mathcal{G}
\right)
\;+\;
\bigl\{E\|\mathcal{G}\|_{n}^{2}\bigr\}^{1/2}\,
\mathcal{J}_{\infty}\left(
\frac{\bigl\{E\|\mathcal{F}\|_{n}^{2}\bigr\}^{1/2}\,\|\mathcal{G}\|_{\infty}}
{\bigl\{E\|\mathcal{G}\|_{n}^{2}\bigr\}^{1/2}},\ \mathcal{F}
\right)
\Biggr\}.
\end{align*}

To complete the proof, we apply Theorems~2.1 and~2.2 of \cite{van2014uniform}, which
respectively bound $E\|\mathcal{F}\|_{n}^{2}$ and $E\|\mathcal{G}\|_{n}^{2}$. In
particular, these results imply
\[
E\|\mathcal{F}\|_{n}^{2}
\ \lesssim\
\|\mathcal{F}\|^{2}
\;+\;
\frac{1}{n}\,
\mathcal{J}_{\infty}^2\bigl(\|\mathcal{F}\|_{\infty},\ \mathcal{F}\bigr); \qquad E\|\mathcal{G}\|_{n}^{2}
\ \lesssim\
\|\mathcal{G}\|^{2}
\;+\;
\delta_{n,\mathcal{G}}^{2},
\]
where $\delta_{n,\mathcal{G}}$ is any solution of the critical inequality
\[
\mathcal{J}_2\left(\delta_{n,\mathcal{G}},\ \mathcal{G}\right)
\ \lesssim\
\sqrt{n}\,\delta_{n,\mathcal{G}}^{2}.
\]
Hence,
\begin{align*}
\mathcal{J}_2\left(
\bigl\{E\|\mathcal{G}\|_{n}^{2}\bigr\}^{1/2},\ \mathcal{G}
\right)
&\lesssim
\mathcal{J}_2\left(
\|\mathcal{G}\|,\ \mathcal{G}
\right)
\;+\;
\sqrt{n}\,\delta_{n,\mathcal{G}}^{2},
\end{align*}
and
\begin{align*}
\bigl\{E\|\mathcal{G}\|_{n}^{2}\bigr\}^{1/2}\,
\mathcal{J}_{\infty}\left(
\frac{\bigl\{E\|\mathcal{F}\|_{n}^{2}\bigr\}^{1/2}\,\|\mathcal{G}\|_{\infty}}
{\bigl\{E\|\mathcal{G}\|_{n}^{2}\bigr\}^{1/2}},\ \mathcal{F}
\right)
&\lesssim
(\|\mathcal{G}\|+\delta_{n,\mathcal{G}})\,
\mathcal{J}_{\infty}\left(
\frac{\|\mathcal{G}\|_{\infty}}{\|\mathcal{G}\|+\delta_{n,\mathcal{G}}}
\left\{
\|\mathcal{F}\|+\frac{1}{\sqrt{n}}\,
\mathcal{J}_{\infty}\bigl(\|\mathcal{F}\|_{\infty},\mathcal{F}\bigr)
\right\},\ \mathcal{F}
\right).
\end{align*}
Combining these bounds yields
\begin{align*}
\mathfrak{R}_n(\mathcal{H})
\ \lesssim\
\frac{1}{\sqrt{n}}
\Biggl\{
&\|\mathcal{F}\|_{\infty}\,
\mathcal{J}_2\left(
\|\mathcal{G}\| \vee \delta_{n,\mathcal{G}},\ \mathcal{G}
\right)\\
&\;+\;
(\|\mathcal{G}\|\vee \delta_{n,\mathcal{G}})\,
\mathcal{J}_{\infty}\left(
\frac{\|\mathcal{G}\|_{\infty}}{\|\mathcal{G}\|\vee \delta_{n,\mathcal{G}}}
\left\{
\|\mathcal{F}\|+\frac{1}{\sqrt{n}}\,
\mathcal{J}_{\infty}\bigl(\|\mathcal{F}\|_{\infty},\mathcal{F}\bigr)
\right\},\ \mathcal{F}
\right)
\Biggr\}.
\end{align*}

\end{proof}

\subsection{A specialized bound for star-shaped classes}

Let $\mathcal F$ and $\mathcal G$ be classes of measurable functions. In this section, we study the empirical inner-product process
\[
\{(P_n-P)(fg):\ f\in\mathcal F,\ g\in\mathcal G\}.
\]
This is a special case of the setup in Appendix~\ref{appendix:innerproduct::gen}, in which $\varphi_1$ and $\varphi_2$ are the identity maps.

The following theorem extends Theorem~\ref{theorem:loc_max_ineq} to pointwise product classes of the form $\mathcal H$.
It is typically most useful when $\mathcal G$ is suitably localized. Refined local maximal inequalities and local concentration bounds based on the sup-norm entropy integral $\mathcal J_{\infty}(\delta,\mathcal F)$ are provided in Appendix~\ref{sec::innerproductentropy}.

\begin{theorem}[Uniform local concentration bound for empirical inner products]
\label{theorem:selfnorm_product_max}
Let $\mathcal F$ be star-shaped, and let $\mathcal G$ satisfy
$\sup_{g\in\mathcal G}\|g\|_\infty \le \|\mathcal G\|_\infty$.
Assume that there exist constants $c_{\infty}>0$ and $\alpha\in[0,1]$ such that
$\|f\|_\infty \le c_{\infty}\,\|f\|^\alpha$ for all $f\in\mathcal F$.
Let $\delta_{n,\mathcal F}>0$ satisfy the critical radius condition
$$\mathfrak{R}_n(\mathcal F,\delta_{n,\mathcal F})\le \delta_{n,\mathcal F}^2.$$
For all $\eta\in(0,1)$, there exists a universal constant $C>0$ such that, with probability at least $1-\eta$, for every
$f\in\mathcal F$ and $g\in\mathcal G$,
\begin{equation*}
(P_n-P)(fg)
\le
(\|f\|\vee \delta_{n,\mathcal F})\,
C\Biggl[
\|\mathcal G\|_\infty\,\delta_{n,\mathcal F}
+
c_{\infty}\,\delta_{n,\mathcal F}^{\alpha-1}\,\mathfrak{R}_n(\mathcal G)
+
\|\mathcal G\|_\infty\sqrt{\frac{\log(1/\eta)}{n}}
+
c_{\infty}\,\|\mathcal G\|_\infty\,\delta_{n,\mathcal F}^{\alpha-1}\frac{\log(1/\eta)}{n}
\Biggr].
\end{equation*}
Consequently, if $\delta_{n,\mathcal F} > \sqrt{\log(1/\eta)/n}$ and
$\left\{\mathfrak{R}_n(\mathcal G) \right\}^{1/2} > \sqrt{\log(1/\eta)/n}$, then
\begin{equation*}
(P_n-P)(fg)
\lesssim
(\|f\|\vee \delta_{n,\mathcal F}) \left(
\|\mathcal G\|_\infty\,\delta_{n,\mathcal F}
+
c_{\infty}\,(1 \vee \|\mathcal G\|_\infty)\,\delta_{n,\mathcal F}^{\alpha-1}\,\mathfrak{R}_n(\mathcal G)\right).
\end{equation*}
\end{theorem}

\textbf{Remark.} A result similar to Theorem~\ref{lem:uniform_local_conc_lipschitz_product_increments_pointwise} could also be obtained from Theorem~\ref{theorem:selfnorm_product_max} by applying it to the star-shaped hull of the transformed classes
$\widetilde{\mathcal F}:=\{\varphi_1(f) - \varphi_1(f_0): f\in\mathcal F\}$ and
$\widetilde{\mathcal G}:=\{\varphi_2(g) - \varphi_2(g_0): g\in\mathcal G\}$.
A limitation of this approach is that the sup-norm bound condition in Theorem~\ref{theorem:selfnorm_product_max} would need to hold for $\widetilde{\mathcal F}$, which does not, in general, follow from the corresponding condition on $\mathcal F$, even when the transformations are Lipschitz. We include the less general result because it leverages the star-shapedness of $\mathcal F$ to yield a clean bound and proof, without invoking the peeling argument used in Theorem~\ref{lem:uniform_local_conc_lipschitz_product_increments_pointwise}.

The proof of the above theorem relies on the following generalization of
Lemma~\ref{lem:self_normalize_rad}, which bounds the Rademacher complexity of
self-normalized inner-product classes and is proved afterward.

\begin{lemma}[Complexity bounds for self-normalized product class]
\label{lemma::selfnormproduct}
Let $\mathcal{F}$ and $\mathcal{G}$ be classes of functions. Assume $\mathcal{F}$ is star shaped and that there exists a $c \in (0,\infty)$ and $\alpha \in [0,1]$ such that, for all $f \in \mathcal{F}$ that $\|f\|_{\infty} \leq c \|f\|^{\alpha}$. For $\delta_n > 0$, define the normalized class
\[
\mathcal{H}
:=
\left\{\frac{fg}{\|f\|\vee \delta_n} : f\in\mathcal{F},\ g\in\mathcal{G}\right\}.
\]
Then, up to universal constants,
\begin{align*}
\mathfrak R_n(\mathcal H)
&\lesssim
\delta_n^{-1}\,\|\mathcal G\|_\infty\,\mathfrak R_n(\mathcal F,\delta_n)
\;+\;
c\,\delta_n^{\alpha-1}\,\mathfrak R_n(\mathcal G).
\end{align*}
\end{lemma}

\begin{proof}[Proof of Theorem \ref{theorem:selfnorm_product_max}]
Let $\delta_{n,\mathcal G}$ be any solution to $\mathfrak{R}_n(\mathcal G)\le \delta_{n,\mathcal G}^2$. Define the self-normalized product class
\[
\mathcal{H}
:=
\left\{\frac{fg}{\|f\|\vee \delta_{n,\mathcal F}} : f\in\mathcal{F},\ g\in\mathcal{G}\right\},
\]
and let \(\delta_{n,\mathcal F}\) satisfy the critical inequality
\(\mathfrak R_n(\mathcal F,\delta_{n,\mathcal F})\le \delta_{n,\mathcal F}^2\).
For any \(h=\frac{fg}{\|f\|\vee \delta_{n,\mathcal F}}\in\mathcal H\),
\[
\|h\|
\le
\|g\|_\infty \frac{\|f\|}{\|f\|\vee \delta_{n,\mathcal F}}
\le
\|\mathcal G\|_\infty.
\]
Moreover, by the assumption \(\|f\|_\infty \le c\|f\|^\alpha\) (with \(\alpha\in[0,1]\)),
\[
\|h\|_\infty
\le
\|g\|_\infty \frac{\|f\|_\infty}{\|f\|\vee \delta_{n,\mathcal F}}
\le
c\,\|\mathcal G\|_\infty \frac{\|f\|^\alpha}{\|f\|\vee \delta_{n,\mathcal F}}
\le
c\,\|\mathcal G\|_\infty\,\delta_{n,\mathcal F}^{\alpha-1}.
\]
Hence, by Bousquet's inequality in Lemma~\ref{lem:bousquet}, there exists a universal
constant \(c_0>0\) such that, for all \(u\ge 0\), with probability at least \(1-e^{-u}\),
\[
\sup_{h\in\mathcal H}(P_n-P)h
\le
\mathbb E\Bigl[\sup_{h\in\mathcal H}(P_n-P)h\Bigr]
+
c_0\Bigl(
\|\mathcal G\|_\infty\sqrt{\frac{u}{n}}
+
c\,\|\mathcal G\|_\infty\,\delta_{n,\mathcal F}^{\alpha-1}\frac{u}{n}
\Bigr).
\]
By the Rademacher symmetrization bound in Lemma~\ref{lem:rademacher},
\[
\mathbb E\Bigl[\sup_{h\in\mathcal H}(P_n-P)h\Bigr]
\le
2\,\mathfrak R_n(\mathcal H).
\]
By Lemma~\ref{lemma::selfnormproduct} with \(\delta_n=\delta_{n,\mathcal F}\),
up to universal constants,
\[
\mathfrak R_n(\mathcal H)
\lesssim
\delta_{n,\mathcal F}^{-1}\,\|\mathcal G\|_\infty\,\mathfrak R_n(\mathcal F,\delta_{n,\mathcal F})
\;+\;
c\,\delta_{n,\mathcal F}^{\alpha-1}\,\mathfrak R_n(\mathcal G).
\]
Let \(\delta_{n,\mathcal G}\) be such that \(\mathfrak R_n(\mathcal G)\le \delta_{n,\mathcal G}^2\).
Then, by the critical inequality for \(\mathcal F\),
\[
\mathfrak R_n(\mathcal H)
\lesssim
\|\mathcal G\|_\infty\,\delta_{n,\mathcal F}
\;+\;
c\,\delta_{n,\mathcal F}^{\alpha-1}\,\delta_{n,\mathcal G}^2.
\]
Combining the above displays, there exists a universal constant \(C>0\) such that,
for all \(u\ge 0\), with probability at least \(1-e^{-u}\),
\[
\sup_{f\in\mathcal F,\ g\in\mathcal G}
\frac{(P_n-P)(fg)}{\|f\|\vee \delta_{n,\mathcal F}}
\le
C\Biggl[
\|\mathcal G\|_\infty\,\delta_{n,\mathcal F}
+
c\,\delta_{n,\mathcal F}^{\alpha-1}\,\delta_{n,\mathcal G}^2
+
\|\mathcal G\|_\infty\sqrt{\frac{u}{n}}
+
c\,\|\mathcal G\|_\infty\,\delta_{n,\mathcal F}^{\alpha-1}\frac{u}{n}
\Biggr].
\]
Equivalently, on the same event, for every \(f\in\mathcal F\) and \(g\in\mathcal G\),
\[
(P_n-P)(fg)
\le
(\|f\|\vee \delta_{n,\mathcal F})\,
C\Biggl[
\|\mathcal G\|_\infty\,\delta_{n,\mathcal F}
+
c\,\delta_{n,\mathcal F}^{\alpha-1}\,\delta_{n,\mathcal G}^2
+
\|\mathcal G\|_\infty\sqrt{\frac{u}{n}}
+
c\,\|\mathcal G\|_\infty\,\delta_{n,\mathcal F}^{\alpha-1}\frac{u}{n}
\Biggr].
\]

Setting $u=\log(1/\eta)$ in the preceding display, there exists a universal
constant $C>0$ such that, with probability at least $1-\eta$,
\begin{equation}\label{eq:HN_hprob_raw}
\sup_{f\in\mathcal F,\ g\in\mathcal G}
\frac{(P_n-P)(fg)}{\|f\|\vee \delta_{n,\mathcal F}}
\le
C\Biggl[
\|\mathcal G\|_\infty\,\delta_{n,\mathcal F}
+
c\,\delta_{n,\mathcal F}^{\alpha-1}\,\delta_{n,\mathcal G}^2
+
\|\mathcal G\|_\infty\sqrt{\frac{\log(1/\eta)}{n}}
+
c\,\|\mathcal G\|_\infty\,\delta_{n,\mathcal F}^{\alpha-1}\frac{\log(1/\eta)}{n}
\Biggr].
\end{equation}
Equivalently, on the same event, for every $f\in\mathcal F$ and $g\in\mathcal G$,
\begin{equation}\label{eq:HN_hprob_pointwise_raw}
(P_n-P)(fg)
\le
(\|f\|\vee \delta_{n,\mathcal F})\,
C\Biggl[
\|\mathcal G\|_\infty\,\delta_{n,\mathcal F}
+
c\,\delta_{n,\mathcal F}^{\alpha-1}\,\delta_{n,\mathcal G}^2
+
\|\mathcal G\|_\infty\sqrt{\frac{\log(1/\eta)}{n}}
+
c\,\|\mathcal G\|_\infty\,\delta_{n,\mathcal F}^{\alpha-1}\frac{\log(1/\eta)}{n}
\Biggr].
\end{equation}
Now assume that $\delta_{n, \mathcal{F}} > \sqrt{\log (1 / \eta) / n}$, and $\delta_{n, \mathcal{G}} > \sqrt{\log (1 / \eta) / n}$. Then, combining like terms, for a possibly different $C$,
\begin{equation}
(P_n-P)(fg)
\le
(\|f\|\vee \delta_{n,\mathcal F})\,
C\Biggl[
\|\mathcal G\|_\infty\,\delta_{n,\mathcal F}
+
c\,(1 \vee \|\mathcal G\|_\infty)\,\delta_{n,\mathcal F}^{\alpha-1}\,\delta_{n,\mathcal G}^2
\Biggr].
\end{equation}

\end{proof}

We now prove the lemma.

\begin{proof}[Proof of Lemma \ref{lemma::selfnormproduct}]

Define $d_n(f):=\|f\|\vee \delta_n$. Consider the classes
\[
\mathcal{H}:=\{fg/d_n(f): f\in\mathcal{F},\ g\in\mathcal{G}\},
\]
\[
\mathcal{H}_{\le}:=\{fg/d_n(f)=fg/\delta_n: f\in\mathcal{F},\ g\in\mathcal{G},\ \|f\|\le \delta_n\},
\]
\[
\mathcal{H}_{>}:=\{fg/d_n(f)=fg/\|f\|: f\in\mathcal{F},\ g\in\mathcal{G},\ \|f\|> \delta_n\}.
\]

We first show that $\mathfrak{R}_n(\mathcal{H}) \lesssim \mathfrak{R}_n(\mathcal{H}_{\leq})$, so that it suffices to bound $\mathfrak{R}_n(\mathcal{H}_{\le})$.
Fix any $g\in\mathcal{G}$ and $f\in\mathcal{F}$ with $\|f\|>\delta_n$. Set
$\lambda:=\delta_n/\|f\|\in(0,1)$ and define $h:=\lambda f$.
By star-shapedness, $h\in\mathcal{F}$ and $\|h\|=\delta_n$. Then $d_n(f)=\|f\|$ and
\[
\frac{(P_n^\epsilon)(fg)}{d_n(f)}
=
\frac{(P_n^\epsilon)(fg)}{\|f\|}
=
\frac{(P_n^\epsilon)(hg)}{\delta_n}
\le
\sup_{f\in\mathcal{F}:\ \|f\|\le \delta_n}
\frac{(P_n^\epsilon)(fg)}{\delta_n}.
\]
Taking expectations, it follows that $\mathfrak{R}_n(\mathcal{H}_{>})
\le
\mathfrak{R}_n(\mathcal{H}_{\le}).$
Hence, by sub-additivity of Rademacher complexities,
\[
\mathfrak{R}_n(\mathcal{H})
\le
\mathfrak{R}_n(\mathcal{H}_{\le})+\mathfrak{R}_n(\mathcal{H}_{>})
\le
2\,\mathfrak{R}_n(\mathcal{H}_{\le}).
\]

We now bound ${\mathfrak R}_n(\mathcal H_{\le})$ conditional on $Z_{1:n}$.
Note that $(fg)(z)=\varphi\{f(z),g(z)\}$ with $\varphi(x,y)=xy$ and $\varphi(0,0)=0$.
If $\|\mathcal G\|_\infty<\infty$ and $\|\mathcal F_n\|_\infty<\infty$, then for all
$(x,y),(x',y')\in [-\|\mathcal F_n\|_\infty,\|\mathcal F_n\|_\infty]\times[-\|\mathcal G\|_\infty,\|\mathcal G\|_\infty]$,
\begin{align*}
|\varphi(x,y)-\varphi(x',y')|
=|xy-x'y'|
&\le
\|\mathcal G\|_\infty |x-x'| + \|\mathcal F_n\|_\infty |y-y'|\\
&\le \sqrt{2}
\sqrt{\|\mathcal G\|_\infty^2 |x-x'|^2 + \|\mathcal F_n\|_\infty^2 |y-y'|^2}.
\end{align*}
Moreover, by the assumed embedding $\|f\|_\infty\le c\|f\|^\alpha$ and the definition
$\mathcal F_n=\{f\in\mathcal F:\|f\|\le \delta_n\}$, we have
$\|\mathcal F_n\|_\infty \le c\,\delta_n^\alpha$.
Therefore, $\varphi$ is Lipschitz in each coordinate with constants
$\|\mathcal G\|_\infty$ and $c\,\delta_n^\alpha$. We now apply a vector contraction inequality for Rademacher processes (Theorem~3 of \citet{maurer2016vector}).
Let $S:=\mathcal F_n\times\mathcal G$ and, for $s=(f,g)\in S$, define
\[
\psi_i(s):=f(Z_i)g(Z_i),
\qquad
\phi_i(s):=\sqrt{2}\,\bigl(\|\mathcal G\|_\infty f(Z_i),\; c\,\delta_n^\alpha g(Z_i)\bigr)\in\mathbb R^2.
\]
Then for all $s=(f,g),s'=(f',g')\in S$,
\begin{align*}
|\psi_i(s)-\psi_i(s')|
&=|f(Z_i)g(Z_i)-f'(Z_i)g'(Z_i)|\\
&\le \|\mathcal G\|_\infty |f(Z_i)-f'(Z_i)| + \|\mathcal F_n\|_\infty |g(Z_i)-g'(Z_i)|\\
&\le \sqrt{2}\sqrt{\|\mathcal G\|_\infty^2|f(Z_i)-f'(Z_i)|^2 + \|\mathcal F_n\|_\infty^2|g(Z_i)-g'(Z_i)|^2}\\
&\le \|\phi_i(s)-\phi_i(s')\|_2,
\end{align*}
where we used $\|\mathcal F_n\|_\infty\le c\,\delta_n^\alpha$.
Thus, the hypotheses of Theorem~3 in \citet{maurer2016vector} hold. We conclude that
\[
\mathbb E_{\epsilon}\Big[\sup_{f\in\mathcal F_n,\ g\in\mathcal G} P_n^\epsilon(fg)\Big]
\ \lesssim\
\|\mathcal G\|_\infty\,\mathbb E_{\epsilon^{(1)}}\Big[\sup_{f\in\mathcal F_n} P_n^{\epsilon^{(1)}} f\Big]
+
c\,\delta_n^\alpha\,\mathbb E_{\epsilon^{(2)}}\Big[\sup_{g\in\mathcal G} P_n^{\epsilon^{(2)}} g\Big],
\]
where $\epsilon^{(1)}$ and $\epsilon^{(2)}$ are independent i.i.d.\ Rademacher sequences. Taking expectations yields
\begin{align*}
\mathfrak R_n(\mathcal H_{\le})
&=
\delta_n^{-1}\,
\mathbb E\Bigl[\sup_{f\in\mathcal F_n,\ g\in\mathcal G} P_n^\epsilon(fg)\Bigr]\\
&\lesssim
\|\mathcal G\|_\infty\,\mathfrak R_n(\mathcal F_n)
\;+\;
c\,\delta_n^\alpha\,\mathfrak R_n(\mathcal G)\\
&\lesssim
\delta_n^{-1}\,\|\mathcal G\|_\infty\,\mathfrak R_n(\mathcal F,\delta_n)
\;+\;
c\,\delta_n^{\alpha-1}\,\mathfrak R_n(\mathcal G).
\end{align*}
We conclude that, up to universal constants,
\begin{align*}
\mathfrak R_n(\mathcal H)
&\lesssim
\delta_n^{-1}\,\|\mathcal G\|_\infty\,\mathfrak R_n(\mathcal F,\delta_n)
\;+\;
c\,\delta_n^{\alpha-1}\,\mathfrak R_n(\mathcal G).
\end{align*}

\end{proof}

\section{Proofs of main results}

\subsection{Proofs for Section \ref{sec:genregret}}

\begin{proof}[Proof of Theorem~\ref{theorem::localmaxloss}]
The first part of the theorem is a direct application of Theorem~\ref{theorem:loc_max_ineq} with
$\mathcal{F}:=\overline{\mathcal{F}}_{\ell}$. The second part follows similarly by taking
$\mathcal{F}:=\mathcal{F}_{\ell}$.
\end{proof}

\begin{proof}[Proof of Theorem \ref{theorem:genregret}]
We first prove the result assuming that, up to universal constants,
$\delta_n \ \gtrsim\ (1 \vee M)\sqrt{\frac{\log(1/\eta)}{n}}$, where
$M :=  \sup_{f \in \mathcal{F}} \|\ell(\cdot, f_0) - \ell(\cdot, f)\|_{\infty}$. By Lemma \ref{theorem::basic},
\[
R(\hat f_n)-R(f_0) \le (P_n-P)\bigl\{\ell(\cdot,f_0)-\ell(\cdot,\hat f_n)\bigr\}.
\]
By Theorem \ref{theorem::localmaxloss}, there exists a universal constant $C>0$ such that, with probability at least $1-\eta$, for every $f\in\mathcal{F}$,
letting $\sigma_f^2 := \Var\{\ell(Z, f_0) - \ell(Z, f)\}$,
\[
(P_n-P) \{\ell(\cdot, f_0) - \ell(\cdot, f)\}
\le
C\left[
\sigma_f\delta_n
+
\delta_n^2
\right].
\]
Since the above bound holds uniformly over $f \in \mathcal{F}$, it also holds for the random element
$\hat{f}_n \in \mathcal{F}$. Hence, with the same confidence,
\[
(P_n-P) \{\ell(\cdot, f_0) - \ell(\cdot, \hat f_n)\}
\le
C\Biggl[
\hat{\sigma}_n \delta_n
+
\delta_n^2
\Biggr],
\]
where $\hat{\sigma}_n := \sigma_{\hat f_n}$. By the Bernstein condition in \ref{cond::bernstein},
\[
\hat{\sigma}_n^2
=
\Var\bigl(\ell(Z,\hat f_n)-\ell(Z,f_0)\bigr)
\le
c_{\mathrm{Bern}}\{R(\hat f_n)-R(f_0)\}.
\]
Letting $\mathrm{d}_n(\hat f_n, f_0) := \{R(\hat f_n) - R(f_0)\}^{1/2}$ and applying the basic
inequality $\mathrm{d}_n^2(\hat f_n, f_0) \le (P_n-P)\{\ell(\cdot, f_0) - \ell(\cdot, \hat f_n)\}$,
we obtain that there exists a universal constant $C>0$ such that, with probability at least $1-\eta$,
\begin{equation}
\label{eqn::proofcriticalineq}
\mathrm{d}_n^2(\hat f_n, f_0)
\le
C\Biggl[
\sqrt{c_{\mathrm{Bern}}}\,\mathrm{d}_n(\hat f_n, f_0)\,\delta_n
+
\delta_n^2
\Biggr].
\end{equation}

We now extract a regret rate using Young's inequality. Let
\(d := \mathrm{d}_n(\hat f_n,f_0)\). From
\[
d^2 \le C\Bigl[\sqrt{c_{\mathrm{Bern}}}\,d\,\delta_n + \delta_n^2\Bigr],
\]
we have, for any \(\alpha\in(0,1)\),
\[
\sqrt{c_{\mathrm{Bern}}}\,d\,\delta_n
\le
\frac{\alpha}{2}\,d^2
+
\frac{c_{\mathrm{Bern}}}{2\alpha}\,\delta_n^2.
\]
Taking \(\alpha := 1/C\) and substituting yields
\[
d^2
\le
\frac{1}{2}d^2
+
C\Bigl(1+c_{\mathrm{Bern}}\Bigr)\delta_n^2.
\]
Rearranging gives, for a universal constant \(C'>0\),
\[
d^2
\le
C'\Bigl(1+c_{\mathrm{Bern}}\Bigr)\delta_n^2.
\]
Equivalently, with probability at least \(1-\eta\),
\[
R(\hat f_n)-R(f_0)
\le
C'\Bigl(1+c_{\mathrm{Bern}}\Bigr)\delta_n^2.
\]

Now, the bound above was derived under the assumption that
\[
\delta_n \ \gtrsim\ (1 \vee M) \sqrt{\frac{\log(1/\eta)}{n}}.
\]
If \(\delta_n\) does not satisfy this condition, define the enlarged radius
\[
\tilde\delta_n := \delta_n \ \vee\ c\,M\sqrt{\frac{\log(1/\eta)}{n}},
\]
for a universal constant \(c>0\). Since \(\delta_n\) satisfies the critical inequality
\(\mathfrak{R}_n(\mathrm{star}(\overline{\mathcal{F}}_{\ell}),\delta_n)\le \delta_n^2\), any
\(\tilde\delta_n\ge \delta_n\) also satisfies
\(\mathfrak{R}_n(\mathrm{star}(\overline{\mathcal{F}}_{\ell}),\tilde\delta_n)\le \tilde\delta_n^2\),
because \(\delta \mapsto \mathfrak{R}_n(\mathcal{G},\delta)/\delta\) is nonincreasing by
Lemma~\ref{lem:rad_scaling}. Repeating the argument yields, with probability at least \(1-\eta\),
\[
R(\hat f_n)-R(f_0)
\le
C'\Bigl(1+c_{\mathrm{Bern}}\Bigr)\tilde\delta_n^2.
\]
In particular, using \((a\vee b)^2\le a^2+b^2\), we have
\[
\tilde\delta_n^2
\le
\delta_n^2 + c^2 M^2\frac{\log(1/\eta)}{n},
\]
and hence, with probability at least \(1-\eta\),
\[
R(\hat f_n)-R(f_0)
\le
C'\Bigl(1+c_{\mathrm{Bern}}\Bigr)\delta_n^2
+
C'\Bigl(1+c_{\mathrm{Bern}}\Bigr)c^2 M^2\frac{\log(1/\eta)}{n}.
\]
The first claim follows by changing the constant $C'$ since $c_{\mathrm{Bern}} \geq 1$.

We now show that, if \(n\) is large enough that \(R(\hat f_n)-R(f_0)\le 1\) with probability
at least \(1-\eta/2\), then the same bound holds (up to a universal constant) when \(\delta_n>0\)
is instead chosen to satisfy the critical radius condition
\(\mathfrak{R}_n(\mathrm{star}(\mathcal{F}_{\ell}),\delta_n)\le \delta_n^2\) for the uncentered
class \(\mathcal{F}_{\ell}\).

We note that Theorem~\ref{theorem::localmaxloss} can be modified so that \(\delta_n\) satisfies the
critical radius condition \(\mathfrak{R}_n(\mathrm{star}(\mathcal{F}_{\ell}),\delta_n)\le \delta_n^2\)
for the uncentered loss-difference class. As before, first assume
\(\delta_n \gtrsim M\sqrt{\frac{\log(1/\eta)}{n}}\). Then, applying
Theorem~\ref{theorem:loc_max_ineq} to \(\mathcal{F}_{\ell}\), there exists a universal constant
\(C>0\) such that, with probability at least \(1-\eta/2\), for every \(f\in\mathcal{F}\), letting
\(\delta_f^2 := \|\ell(\cdot,f)-\ell(\cdot,f_0)\|^2\),
\[
(P_n-P)\{\ell(\cdot,f_0)-\ell(\cdot,f)\}
\le
C\Biggl[
\delta_f \delta_n
+
\delta_n^2
+
\delta_f \sqrt{\frac{\log(1/\eta)}{n}}
+
\frac{M\log(1/\eta)}{n}
\Biggr].
\]
The difference is that \(\sigma_f^2\) is replaced by the larger norm
\(\|\ell(\cdot,f)-\ell(\cdot,f_0)\|^2\). By the elementary decomposition
\[
\|\ell(\cdot,f)-\ell(\cdot,f_0)\|^2
=
\Var\bigl(\ell(Z,f)-\ell(Z,f_0)\bigr)
+\{R(f)-R(f_0)\}^2,
\]
Condition~\ref{cond::bernstein} implies
\[
\|\ell(\cdot,f)-\ell(\cdot,f_0)\|
\le
\sqrt{c_{\mathrm{Bern}}\,\{R(f)-R(f_0)\}}
+
\{R(f)-R(f_0)\}.
\]
Provided \(n\) is large enough that \(R(\hat f_n)-R(f_0)\le 1\), it follows that
\[
\|\ell(\cdot,\hat f_n)-\ell(\cdot,f_0)\|
\le
\bigl(1+\sqrt{c_{\mathrm{Bern}}}\bigr)\sqrt{R(\hat f_n)-R(f_0)}.
\]
Combining the high-probability bound with the basic inequality, we find that, with probability at
least \(1-\eta\), letting \(\hat d_n := \sqrt{R(\hat f_n)-R(f_0)}\),
\[
\hat d_n^2
\le
C\Biggl[
\bigl(1+\sqrt{c_{\mathrm{Bern}}}\bigr)\hat d_n \delta_n
+
\delta_n^2
+
\bigl(1+\sqrt{c_{\mathrm{Bern}}}\bigr)\hat d_n\sqrt{\frac{\log(1/\eta)}{n}}
+
\frac{M\log(1/\eta)}{n}
\Biggr].
\]
This is the same inequality as \eqref{eqn::proofcriticalineq}, up to the multiplicative factor
\((1+\sqrt{c_{\mathrm{Bern}}})\) on the linear terms. Applying Young's inequality as before absorbs
these terms into the left-hand side and yields the same regret bound, with constants depending on
\(c_{\mathrm{Bern}}\) only through universal factors. Hence, the desired bound holds with probability
at least \(1-\eta/2\). If, in addition, \(\Pr\{R(\hat f_n)-R(f_0)\le 1\}\ge 1-\eta/2\), then by a union
bound the event \(\{R(\hat f_n)-R(f_0)\le 1\}\) and the desired bound both hold with probability at
least \(1-\eta\).

\end{proof}

Our proof of Theorem~\ref{theorem:genregret::ell2} uses the second part of Theorem~\ref{theorem::localmaxloss}, which for clarity we restate as the following corollary.

\begin{corollary}[Uniform local concentration inequality] 
\label{cor::localmaxloss::ell2}
Let $\delta_n>0$ satisfy the critical radius condition
$\mathfrak{R}_n(\mathrm{star}(\mathcal{F}_{\ell}),\delta_n)\le \delta_n^2$. Define $M := \sup_{f \in \mathcal{F}} \|\ell(\cdot, f_0) - \ell(\cdot, f)\|_{\infty}$ for each $f\in\mathcal{F}$.
For all $\eta \in (0,1)$, there exists a universal constant $C>0$ such that, with probability at least $1-\eta$, for every $f\in\mathcal{F}$,
\[
(P_n-P) \{\ell(\cdot, f_0) - \ell(\cdot, f)\}
\le
C\left[
\|\ell(\cdot, f_0) - \ell(\cdot, f)\|\delta_n
+
\delta_n^2
+
\sigma_f (M \vee 1) \sqrt{\frac{\log(1/\eta)}{n}}
+
(M \vee 1)^2\frac{\log(1/\eta)}{n}
\right].
\]
\end{corollary}

\begin{proof}[Proof of Theorem~\ref{theorem:genregret::ell2}]
The proof largely follows that of Theorem~\ref{theorem:genregret}, with minor modifications as described in the remark in Section~\ref{sec:template}.

Let $\hat f_n\in\arg\min_{f\in\mathcal F} P_n\ell(\cdot,f)$. By Theorem~\ref{theorem::basic},
\[
R(\hat f_n)-R(f_0)
\le
(P_n-P)\bigl\{\ell(\cdot,f_0)-\ell(\cdot,\hat f_n)\bigr\}.
\]
By the curvature bound in Condition~\ref{cond::bernstein},
\begin{equation}
\label{eqn:l2_basic_curvature}
\kappa\|\hat f_n-f_0\|^2
\le
(P_n-P)\bigl\{\ell(\cdot,f_0)-\ell(\cdot,\hat f_n)\bigr\}.
\end{equation}

Next, apply Corollary~\ref{cor::localmaxloss::ell2}. 
Define the enlarged radius
\[
\tilde\delta_n
:=
\delta_n\ \vee\ c\,M\sqrt{\frac{\log(1/\eta)}{n}},
\]
for a universal constant $c>0$. Since $\tilde\delta_n\ge \delta_n$ and
$\delta\mapsto\mathfrak{R}_n(\mathrm{star}(\mathcal F_\ell),\delta)/\delta$ is nonincreasing
(Lemma~\ref{lem:rad_scaling}), $\tilde\delta_n$ also satisfies the critical inequality
$\mathfrak{R}_n(\mathrm{star}(\mathcal F_\ell),\tilde\delta_n)\lesssim \tilde\delta_n^2$.

By Corollary~\ref{cor::localmaxloss::ell2}, there exists a universal constant $C>0$ and an event
$\mathcal E$ with $\Pr(\mathcal E)\ge 1-\eta$ such that, on $\mathcal E$,
\emph{uniformly over all $f\in\mathcal F$},
\begin{equation}
\label{eqn:localmax_uncentered}
(P_n-P)\{\ell(\cdot,f_0)-\ell(\cdot,f)\}
\le
C\Bigl[
\|\ell(\cdot,f_0)-\ell(\cdot,f)\|\,\tilde\delta_n
+
\tilde\delta_n^2
\Bigr].
\end{equation}
In particular, on $\mathcal E$ this bound holds for $f=\hat f_n$. Using the $L^2(P)$-Lipschitz bound
in Condition~\ref{cond::bernstein},
\[
\|\ell(\cdot,f_0)-\ell(\cdot,\hat f_n)\|
\le
L\|\hat f_n-f_0\|,
\]
and substituting into \eqref{eqn:localmax_uncentered} yields, on $\mathcal E$,
\[
(P_n-P)\{\ell(\cdot,f_0)-\ell(\cdot,\hat f_n)\}
\le
C\Bigl[
L\|\hat f_n-f_0\|\,\tilde\delta_n
+
\tilde\delta_n^2
\Bigr].
\]
Combining with \eqref{eqn:l2_basic_curvature} gives, on $\mathcal E$,
\begin{equation}
\label{eqn:l2_fixed_point}
\kappa\|\hat f_n-f_0\|^2
\le
C\Bigl[
L\|\hat f_n-f_0\|\,\tilde\delta_n
+
\tilde\delta_n^2
\Bigr].
\end{equation}

Finally, apply Young's inequality to the linear term. Writing $u:=\|\hat f_n-f_0\|$,
\[
CLu\tilde\delta_n
\le
\frac{\kappa}{2}u^2
+
\frac{C^2L^2}{2\kappa}\tilde\delta_n^2.
\]
Substituting into \eqref{eqn:l2_fixed_point} and absorbing $(\kappa/2)u^2$ into the left-hand side
gives, for a (possibly larger) universal constant $C>0$,
\[
u^2
\le
C\,\kappa^{-1}L^2\,\tilde\delta_n^2.
\]
Using $(a\vee b)^2\le a^2+b^2$ and $M\le 2LM$, we have
\[
\tilde\delta_n^2
\le
\delta_n^2
+
c^2M^2\frac{\log(1/\eta)}{n}
\le
\delta_n^2
+
C\,L^2M^2\frac{\log(1/\eta)}{n},
\]
after adjusting the universal constant. Hence, with probability at least $1-\eta$,
\[
\|\hat f_n-f_0\|^2
\le
C\,\kappa^{-1}L^2\Biggl[
\delta_n^2
+
M^2\frac{\log(1/\eta)}{n}
\Biggr],
\]
which is the desired bound (up to a universal constant).
\end{proof}

\begin{proof}[Proof of Corollary \ref{cor::localmaxloss_as}]
Apply Theorem~\ref{theorem::localmaxloss} with \(\eta_n=n^{-2}\). Since
\(\sum_{n=1}^\infty \eta_n<\infty\), the Borel--Cantelli lemma implies that, with probability one, the theorem's bound fails only finitely often. The result follows because \(\log(1/\eta_n)=2\log n\), and the factor \(2\) is absorbed into the universal constant \(C\).
\end{proof}
\begin{proof}[Proof of Corollary \ref{cor::genregret_as}]
Apply Theorem~\ref{theorem:genregret} with \(\eta_n=n^{-2}\). Since \(\sum_{n=1}^\infty \eta_n<\infty\), the Borel--Cantelli lemma implies that, with probability one, the theorem's bound fails only finitely often. The result follows because \(\log(1/\eta_n)=2\log n\), and the factor \(2\) is absorbed into the universal constant \(C\).
\end{proof}

\subsection{Proofs for Section \ref{sec:entropy}}

The metric entropy maximal inequalities are proven in Appendix \ref{appendix::localmax2}. Here, we prove the remaining results.

\begin{proof}[Proof of Lemma \ref{lemma::starentropybounds}]
We first show that
\[
\mathcal J_2\left(\delta,\ \mathrm{star}(\overline{\mathcal{F}_{\ell}})\right)
\ \lesssim\
\mathcal J_2\left(\delta,\ \mathrm{star}(\mathcal{F}_{\ell})\right)
+
\delta\sqrt{\log\left(1+\frac{M}{\delta}\right)}.
\]
By linearity of $P$, centering commutes with scaling, so
\[
\mathrm{star}(\overline{\mathcal{F}}_{\ell})
=
\overline{\mathrm{star}(\mathcal{F}_{\ell})}.
\]
Moreover, for any $f\in \mathrm{star}(\mathcal{F}_{\ell})$ we have
$f-Pf = f + c$ with $c:=-Pf$. Since $\sup_{f\in\mathcal{F}}\|\ell(\cdot,f)\|_\infty\le M$, we have
$\sup_{h\in\mathcal F_\ell}\|h\|_\infty\le 2M$, and hence also
$\sup_{f\in\mathrm{star}(\mathcal F_\ell)}\|f\|_\infty\le 2M$, which implies $|c|\le 2M$. Therefore,
\[
\overline{\mathrm{star}(\mathcal{F}_{\ell})}
\subseteq
\mathrm{star}(\mathcal{F}_{\ell}) + [-2M,2M].
\]
Consequently, for every $\varepsilon>0$ and every $Q$, using the standard product-cover bound
$N(\varepsilon,A+B)\le N(\varepsilon/2,A)\,N(\varepsilon/2,B)$,
\[
\log N\left(\varepsilon,\ \overline{\mathrm{star}(\mathcal{F}_{\ell})},\ L^2(Q)\right)
\le
\log N\left(\varepsilon/2,\ \mathrm{star}(\mathcal{F}_{\ell}),\ L^2(Q)\right)
+
\log\Bigl(1+\bigl\lceil 4M/\varepsilon\bigr\rceil\Bigr).
\]
Integrating this bound yields
\[
\mathcal J_2\left(\delta,\ \overline{\mathrm{star}(\mathcal{F}_{\ell})}\right)
\ \lesssim\
\mathcal J_2\left(\delta,\ \mathrm{star}(\mathcal{F}_{\ell})\right)
+
\delta\sqrt{\log\left(1+\frac{M}{\delta}\right)}.
\]

Finally, a direct application of Corollary~\ref{cor:J2_star_hull_lip} yields
\[
\mathcal J_2\left(\delta,\ \mathrm{star}(\mathcal{F}_{\ell})\right)
\ \lesssim\
\mathcal J_2\left(\delta,\ \mathcal{F}_{\ell}\right)
+
\delta\,\sqrt{\log\left(1+\frac{M}{\delta}\right)}.
\]
Combining the previous two displays and absorbing constants gives
\[
\mathcal J_2\left(\delta,\ \mathrm{star}(\overline{\mathcal{F}_{\ell}})\right)
=
\mathcal J_2\left(\delta,\ \overline{\mathrm{star}(\mathcal{F}_{\ell})}\right)
\ \lesssim\
\mathcal J_2\left(\delta,\ \mathcal{F}_{\ell}\right)
+
\delta\sqrt{\log\left(1+\frac{M}{\delta}\right)}.
\]

If Condition~\ref{cond::lipschitz} holds, then Theorem~\ref{theorem:vw_lipschitz_entropy} implies
$\mathcal J_2(\delta,\mathcal F_\ell)\lesssim \mathcal J_2(\delta/L,\mathcal F)$, and the same argument yields
\[
\mathcal J_2\left(\delta,\ \mathrm{star}(\overline{\mathcal{F}_{\ell}})\right)
\ \lesssim\
\mathcal J_2\left(\delta/L,\ \mathcal{F}\right)
+
\delta\,\sqrt{\log\left(1+\frac{M}{\delta}\right)}.
\]
\end{proof}

\subsection{Proofs for Section  \ref{sec:nuis}}

\begin{proof}[Proof of Theorem \ref{theorem:weightedERM}]
Denote $R(f;w) := P\{w(\cdot)\,\ell(\cdot,f)\}$. Let $\hat f_0 \in \arg\min_{f\in\mathcal F} R(f;\hat w)$.  Add and subtract
$R(\hat f_0;\hat w)$ and $R(f_0;\hat w)$ to obtain
\begin{align*}
R(\hat f_0;w_0)-R(f_0;w_0)
&=
\bigl\{R(\hat f_0;w_0)-R(\hat f_0;\hat w)\bigr\}
+
\bigl\{R(\hat f_0;\hat w)-R(f_0;\hat w)\bigr\}
+
\bigl\{R(f_0;\hat w)-R(f_0;w_0)\bigr\}.
\end{align*}
By definition, $\hat f_0$ minimizes $f\mapsto R(f;\hat w)$ over $\mathcal F$, so
$R(\hat f_0;\hat w)\le R(f_0;\hat w)$ and hence
$\bigl\{R(\hat f_0;\hat w)-R(f_0;\hat w)\bigr\}\le 0$. Therefore,
\begin{align*}
R(\hat f_0;w_0)-R(f_0;w_0)
&\le
\bigl\{R(\hat f_0;w_0)-R(\hat f_0;\hat w)\bigr\}
+
\bigl\{R(f_0;\hat w)-R(f_0;w_0)\bigr\}\\
&\le
P\left\{(w_0-\hat w)\,\ell(\cdot,\hat f_0)\right\}
-
P\left\{(w_0-\hat w)\,\ell(\cdot,f_0)\right\}\\
&=
P\left\{(w_0-\hat w)\bigl(\ell(\cdot,\hat f_0)-\ell(\cdot,f_0)\bigr)\right\}\\
&=
P\left\{(1-\hat w/w_0)w_0\bigl(\ell(\cdot,\hat f_0)-\ell(\cdot,f_0)\bigr)\right\}.
\end{align*}
We further decompose
\begin{align*}
P\left\{\Bigl(1-\frac{\hat w}{w_0}\Bigr) w_0\bigl(\ell(\cdot,\hat f_0)-\ell(\cdot,f_0)\bigr)\right\}
&=
P\left\{\Bigl(1-\frac{\hat w}{w_0}\Bigr)\Bigl[w_0\bigl(\ell(\cdot,\hat f_0)-\ell(\cdot,f_0)\bigr)
- P\left\{w_0\bigl(\ell(\cdot,\hat f_0)-\ell(\cdot,f_0)\bigr)\right\}\Bigr]\right\}\\
&\quad+
P\left(1-\frac{\hat w}{w_0}\right)\,
P\left\{w_0\bigl(\ell(\cdot,\hat f_0)-\ell(\cdot,f_0)\bigr)\right\}.
\end{align*}
By the Bernstein condition and the Cauchy--Schwarz inequality, the first term satisfies
\begin{align*}
&P\left\{\Bigl(1-\frac{\hat w}{w_0}\Bigr)\Bigl[w_0\bigl(\ell(\cdot,\hat f_0)-\ell(\cdot,f_0)\bigr)
- P\left\{w_0\bigl(\ell(\cdot,\hat f_0)-\ell(\cdot,f_0)\bigr)\right\}\Bigr]\right\} \\
&\le \left\|1-\frac{\hat w}{w_0}\right\|\,
\left\|w_0\bigl(\ell(\cdot,\hat f_0)-\ell(\cdot,f_0)\bigr)
- P\left\{w_0\bigl(\ell(\cdot,\hat f_0)-\ell(\cdot,f_0)\bigr)\right\}\right\|\\
&= \left\|1-\frac{\hat w}{w_0}\right\|\,
\Var\left\{w_0\bigl(\ell(\cdot,\hat f_0)-\ell(\cdot,f_0)\bigr)\right\}^{1/2}\\
&\le c\left\|1-\frac{\hat w}{w_0}\right\|\,\bigl\{R(\hat f_0;w_0)-R(f_0;w_0)\bigr\}^{1/2}.
\end{align*}
Next, since $P\left\{w_0\bigl(\ell(\cdot,\hat f_0)-\ell(\cdot,f_0)\bigr)\right\}
=
R(\hat f_0;w_0)-R(f_0;w_0)$, the second term satisfies
\begin{align*}
P\left(1-\frac{\hat w}{w_0}\right)\,
P\left\{w_0\bigl(\ell(\cdot,\hat f_0)-\ell(\cdot,f_0)\bigr)\right\}
&=
P\left(1-\frac{\hat w}{w_0}\right)\,
\bigl\{R(\hat f_0;w_0)-R(f_0;w_0)\bigr\}\\
&\le
\left|P\left(1-\frac{\hat w}{w_0}\right)\right|\,
\bigl\{R(\hat f_0;w_0)-R(f_0;w_0)\bigr\}\\
&\le
\left\|1-\frac{\hat w}{w_0}\right\|\,
\bigl\{R(\hat f_0;w_0)-R(f_0;w_0)\bigr\}.
\end{align*}
Thus, combining both bounds,
\[
P\left\{\Bigl(1-\frac{\hat w}{w_0}\Bigr) w_0\bigl(\ell(\cdot,\hat f_0)-\ell(\cdot,f_0)\bigr)\right\}
\le
c\left\|1-\frac{\hat w}{w_0}\right\|\,\bigl\{R(\hat f_0;w_0)-R(f_0;w_0)\bigr\}^{1/2}
+
\left\|1-\frac{\hat w}{w_0}\right\|\,
\bigl\{R(\hat f_0;w_0)-R(f_0;w_0)\bigr\}.
\]
We conclude that
\begin{align*}
R(\hat f_0;w_0)-R(f_0;w_0) \le
c\left\|1-\frac{\hat w}{w_0}\right\|\,\bigl\{R(\hat f_0;w_0)-R(f_0;w_0)\bigr\}^{1/2}
+
\left\|1-\frac{\hat w}{w_0}\right\|\,
\bigl\{R(\hat f_0;w_0)-R(f_0;w_0)\bigr\}.
\end{align*}
Assuming that $\left\|1-\frac{\hat w}{w_0}\right\| < 1/2$, it follows that
\[
R(\hat f_0;w_0)-R(f_0;w_0)\le 4c^2\left\|1-\frac{\hat w}{w_0}\right\|^2.
\]

\end{proof}

\begin{proof}[Proof of Theorem \ref{lemma::starentropybounds}]
By Lemma~\ref{lemma::starentropybounds}, for all \(\delta>0\),
\[
\mathcal J_2\Bigl(\delta,\ \mathrm{star}(\overline{\mathcal F}_\ell)\Bigr)
\ \lesssim\
\mathcal J_2(\delta/L,\mathcal F)\;+\;\delta\sqrt{\log\Bigl(1+\frac{M}{\delta}\Bigr)}.
\]
Using \(\mathcal J_2(\delta/L,\mathcal F)\lesssim \phi(\delta/L)\), define
\[
\phi_n(\delta)\ :=\ \phi(\delta/L)\;+\;\delta\sqrt{\log\Bigl(1+\frac{M}{\delta}\Bigr)}.
\]

Let \(c>0\) denote the implied constant, and define
\[
\delta_{n,0}\ :=\ \inf\bigl\{\delta>0:\ \phi(\delta)\le c\,\sqrt n\,\delta^2\bigr\},
\qquad
\delta_{n,1}\ :=\ \sqrt{\frac{\log(eMn)}{n}}.
\]
Set \(\delta_n:= L\delta_{n,0}\vee \delta_{n,1}\). We show that
\(\phi_n(\delta_n)\lesssim \sqrt n\,\delta_n^2\).
By Lemma~\ref{lemma::critentropy}, \(\delta_n\) then satisfies the critical-radius condition, and the result follows by direct application of Theorem~\ref{theorem:genregret}.

For the first term, consider two cases.
If \(\delta_n=L\delta_{n,0}\), then \(\delta_n/L=\delta_{n,0}\) and
\[
\phi(\delta_n/L)=\phi(\delta_{n,0})\le c\,\sqrt n\,\delta_{n,0}^2\le c\,\sqrt n\,\delta_n^2.
\]
If \(\delta_n=\delta_{n,1}\), then \(\delta_n/L=\delta_{n,1}/L\ge \delta_{n,0}\), so by definition of \(\delta_{n,0}\),
\[
\phi(\delta_n/L)=\phi(\delta_{n,1}/L)\ \le\ c\,\sqrt n\,(\delta_{n,1}/L)^2\ \le\ c\,\sqrt n\,\delta_n^2.
\]
Thus in all cases \(\phi(\delta_n/L)\lesssim \sqrt n\,\delta_n^2\).

For the second term, since \(\delta_n\ge \delta_{n,1}\),
\[
\delta_n\sqrt{\log\Bigl(1+\frac{M}{\delta_n}\Bigr)}
\ \le\
\delta_n\sqrt{\log\Bigl(1+\frac{M}{\delta_{n,1}}\Bigr)}.
\]
Moreover, \(\delta_{n,1}\ge n^{-1}\) implies \(M/\delta_{n,1}\le Mn\), hence
\(1+M/\delta_{n,1}\le 1+Mn\le eMn\) (since \(Mn\ge 1\). Therefore
\[
\delta_n\sqrt{\log\Bigl(1+\frac{M}{\delta_n}\Bigr)}
\ \le\
\delta_n\sqrt{\log(eMn)}
\ \le\
\sqrt n\,\delta_n^2,
\]
where the last step uses \(\sqrt{\log(eMn)}\le \sqrt n\,\delta_n\) since \(\delta_n\ge \delta_{n,1}=\sqrt{\log(eMn)/n}\).

Combining yields \(\phi_n(\delta_n)\lesssim \sqrt n\,\delta_n^2\), hence
\(\mathcal J_2(\delta_n,\mathrm{star}(\overline{\mathcal F}_\ell))\lesssim \sqrt n\,\delta_n^2\).
Applying Theorem~\ref{theorem:genregret} to \(\mathrm{star}(\overline{\mathcal F}_\ell)\) gives the claimed regret bound (after collecting constants).

\end{proof}

\subsection{Proof for Section \ref{sec:insample}}

\begin{proof}[Proof of Theorem~\ref{theorem::insamplenuis}]
In what follows, we use the shorthand $\alpha := 1 - 1/(2\beta)$, so that Condition~\ref{cond::holdersup} can be written as the existence of an $\alpha\ge 0$ such that
\[
\|f-f'\|_\infty \le c_{\infty}\,\|f-f'\|^{\alpha}
\qquad\text{for all } f,f'\in\mathcal F.
\]

Theorem~\ref{theorem::basic} implies that
\begin{equation}
\label{eqn::basicnuis}
\begin{aligned}
\Reg(\hat f_n, \hat g)
&\le (P_n-P)\bigl\{\ell_{\hat g}(\cdot,\hat f_0)-\ell_{\hat g}(\cdot,\hat f_n)\bigr\} \\
&=
(P_n-P)\bigl\{\ell_{g_0}(\cdot,\hat f_0)-\ell_{g_0}(\cdot,\hat f_n)\bigr\}\\
&\quad + (P_n-P)\Bigl(\bigl\{\ell_{\hat g}(\cdot,\hat f_0)-\ell_{\hat g}(\cdot,\hat f_n)\bigr\}
-\bigl\{\ell_{g_0}(\cdot,\hat f_0)-\ell_{g_0}(\cdot,\hat f_n)\bigr\}\Bigr).
\end{aligned}
\end{equation}
We proceed by deriving high-probability bounds for each term on the right-hand side, following the strategy used in the proof of Theorem~\ref{theorem::localmaxloss}.

\paragraph{Step 1: The first term.}
Define the class
\[
\mathcal{F}_{\ell,g_0}
:=
\Bigl\{\ell_{g_0}(\cdot,f)-\ell_{g_0}(\cdot,f') : f,f'\in\mathcal F\Bigr\}.
\]
By the Lipschitz conditions in Condition~\ref{cond::lipcross}, we have
\[
\|\mathcal{F}_{\ell,g_0}\|_{\infty}
:= \sup_{h \in \mathcal{F}_{\ell,g_0}}\|h\|_{\infty}
\ \lesssim\ L \sup_{f \in \mathcal{F}}\|f\|_{\infty}.
\]
Moreover, for all $z$,
\begin{align*}
\Bigl|
\ell_{g_0}(z,f_1)-\ell_{g_0}(z,f_1')
-
\bigl\{\ell_{g_0}(z,f_2)-\ell_{g_0}(z,f_2')\bigr\}
\Bigr|
&\lesssim
L \|\mathcal{G}\|_{\infty} \Bigl\{|f_1(z)-f_1'(z)|+|f_2(z)-f_2'(z)|\Bigr\}\\
&\lesssim
L\|\mathcal{G}\|_{\infty}\,
\sqrt{|f_1(z)-f_1'(z)|^2+|f_2(z)-f_2'(z)|^2}.
\end{align*}

By convexity of $\mathcal{F}$, $\mathrm{star}(\mathcal F-\mathcal F) = \mathcal{F}$. Hence, a slight modification of the proof of Theorem~\ref{thm:uniform_local_conc_peeling}, replacing $\mathcal{F}-\{f_0\}$ with the difference class $\mathcal{F}-\mathcal{F}$, or a direct application of Theorem~\ref{lem:uniform_local_conc_lipschitz_product_increments_pointwise}, yields the following: there exists a universal constant \(C\in(0,\infty)\) such that, for every \(\eta\in(0,1)\), with probability at least \(1-\eta/4\), the following holds simultaneously for all \(f,f'\in\mathcal F\):
\begin{align*}
\bigl|(P_n-P)\{\ell_{g_0}(\cdot,f)-\ell_{g_0}(\cdot,f')\}\bigr|
\;\le\;
C\Biggl[
&L\,\delta_{n,\mathcal F}\,\|f-f'\|
+
L^2\,\delta_{n,\mathcal F}^2
+
L^2\,\frac{\log\log(M e n)}{n}\\
&\quad +
L\,\|f-f'\|\sqrt{\frac{\log(1/\eta)}{n}}
+
\frac{M\log(1/\eta)}{n}
\Biggr],
\end{align*}
where we used that $\delta_{n,\mathcal F}$ satisfies the critical inequality
\[
\mathfrak R_n\bigl(\mathrm{star}(\mathcal F-\mathcal F),\delta_{n,\mathcal F}\bigr)
\le \delta_{n,\mathcal F}^2.
\]
By assumption, we have
\[
\delta_{n,\mathcal F}
\ \gtrsim\
(1 \vee M)\sqrt{\frac{\log(1/\eta)+\log\log(M e n)}{n}}.
\]
Therefore,
\[
\sqrt{\frac{\log(1/\eta)}{n}}
\ \lesssim\
\delta_{n,\mathcal F},
\qquad
\frac{\log\log(M e n)}{n}
\ \lesssim\
\delta_{n,\mathcal F}^2,
\qquad
\frac{M\log(1/\eta)}{n}
\ \lesssim\
\delta_{n,\mathcal F}^2,
\]
and the explicit remainder terms above are absorbed. Hence, on the same event,
\begin{align*}
\bigl|(P_n-P)\{\ell_{g_0}(\cdot,f)-\ell_{g_0}(\cdot,f')\}\bigr|
\;\lesssim\;
\delta_{n,\mathcal F}\,\|f-f'\|
+
\delta_{n,\mathcal F}^2,
\end{align*}
where the implicit constant depends only on $L$, $M$, and $\|\mathcal G\|_\infty$.

Taking $f=\hat f_n$ and $f'=\hat f_0$, we obtain that, with probability at least $1-\eta/4$,
\begin{equation}
\label{eqn::firstnuis}
(P_n-P)\{\ell_{g_0}(\cdot,\hat f_0)-\ell_{g_0}(\cdot,\hat f_n)\}
\;\lesssim\;
\|\hat f_n-\hat f_0\|\,\delta_{n,\mathcal{F}-\mathcal{F}}
+
\delta_{n,\mathcal{F}-\mathcal{F}}^2,
\end{equation}
where the implicit constant depends only on $L$, $M$, and $\|\mathcal G\|_\infty$.

\paragraph{Step 2: the second term.}
By Condition~\ref{cond::lipcross}, write
$\ell_g(z,f)=m_1(z,f)m_2(z,g)+r_1(z,f)+r_2(z,g)$. Then
\begin{align*}
&(P_n-P)\Bigl(\bigl\{\ell_{\hat g}(\cdot,\hat f_0)-\ell_{\hat g}(\cdot,\hat f_n)\bigr\}
-\bigl\{\ell_{g_0}(\cdot,\hat f_0)-\ell_{g_0}(\cdot,\hat f_n)\bigr\}\Bigr) \\
&\qquad=
(P_n-P)\Bigl[
\{m_1(\cdot,\hat f_0)-m_1(\cdot,\hat f_n)\}\,
\{m_2(\cdot,\hat g)-m_2(\cdot,g_0)\}
\Bigr],
\end{align*}
since the $r_1$ terms cancel between the two brackets and the $r_2$ terms cancel within each bracket.

Under Condition~\ref{cond::lipcross}, we may write $m_1(\cdot,f)=\varphi_1(f)$ and $m_2(\cdot,g)=\varphi_2(g)$ for the pointwise maps
$\varphi_1,\varphi_2$ introduced above. Hence the display above equals
\[
(P_n-P)\Bigl[\{\varphi_1(\hat f_0)-\varphi_1(\hat f_n)\}\{\varphi_2(\hat g)-\varphi_2(g_0)\}\Bigr].
\]

By Condition~\ref{cond::nuispac}, with probability at least $1-\eta/2$,
\[
\|\hat g-g_0\|_{\mathcal G}\ \le\ \varepsilon_{\mathrm{nuis}}(n,\eta).
\]
Apply Theorem~\ref{lem:uniform_local_conc_lipschitz_product_increments_pointwise} with
$f=\hat f_0$, $f_0=\hat f_n$, $g=\hat g$, and $g_0=g_0$ (with classes $\mathcal F$ and $\mathcal G$).
Since the theorem holds uniformly over $g,g_0\in\mathcal G$, it applies to the random pair $(\hat g,g_0)$.
Thus, with probability at least $1-\eta/4$,
\begin{align}
\label{eqn::secnuis_corrected}
&\bigl|(P_n-P)\bigl[\{\varphi_1(\hat f_0)-\varphi_1(\hat f_n)\}\{\varphi_2(\hat g)-\varphi_2(g_0)\}\bigr]\bigr|
\nonumber\\
&\qquad\lesssim\ 
\|\mathcal G\|_\infty\,
\delta_{n,\mathcal F}\,
\Bigl(\|\hat f_0-\hat f_n\|\vee \delta_{n,\mathcal F}\Bigr)
\;+\;
c_\infty\,
\|\hat f_0-\hat f_n\|^{\alpha}\,
\delta_{n,\mathcal G}\,
\Bigl(\|\hat g-g_0\|\vee \delta_{n,\mathcal G}\Bigr)
\nonumber\\
&\qquad\quad+\ 
\|\mathcal G\|_\infty\,
\|\hat f_0-\hat f_n\|\,
\sqrt{\frac{\log\log(eMn)+\log(1/\eta)}{n}}
\;+\;
c_\infty\,
\|\hat f_0-\hat f_n\|^{\alpha}\,\|\mathcal G\|_\infty\,
\frac{\log\log(eMn)+\log(1/\eta)}{n}.
\end{align}
where $M:=1+\|\mathcal G\|_\infty\vee \|\mathcal F\|_\infty$, and the implicit constant equals $CL_1L_2$ for a universal $C\in(0,\infty)$.

Intersecting the event in \eqref{eqn::secnuis_corrected} with the PAC-style event
$\{\|\hat g-g_0\|_{\mathcal G}\le \varepsilon_{\mathrm{nuis}}(n,\eta)\}$ and taking a union bound yields that,
with probability at least $1-3\eta/4$,
\begin{align}
\label{eqn::secnuis_corrected_pac}
&\bigl|(P_n-P)\bigl[\{\varphi_1(\hat f_0)-\varphi_1(\hat f_n)\}\{\varphi_2(\hat g)-\varphi_2(g_0)\}\bigr]\bigr|
\nonumber\\
&\qquad\lesssim\ 
\|\mathcal G\|_\infty\,
\delta_{n,\mathcal F}\,
\Bigl(\|\hat f_0-\hat f_n\|\vee \delta_{n,\mathcal F}\Bigr)
\;+\;
c_\infty\,
\|\hat f_0-\hat f_n\|^{\alpha}\,
\delta_{n,\mathcal G}\,
\Bigl(\varepsilon_{\mathrm{nuis}}(n,\eta)\vee \delta_{n,\mathcal G}\Bigr)
\nonumber\\
&\qquad\quad+\ 
\|\mathcal G\|_\infty\,
\|\hat f_0-\hat f_n\|\,
\sqrt{\frac{\log\log(eMn)+\log(1/\eta)}{n}}
\;+\;
c_\infty\,
\|\hat f_0-\hat f_n\|^{\alpha}\,\|\mathcal G\|_\infty\,
\frac{\log\log(eMn)+\log(1/\eta)}{n}.
\end{align}
\paragraph{Step 3: Conclude.}
Let $\mathcal E_1$ denote the event on which \eqref{eqn::firstnuis} holds with confidence level $\eta/4$, and let $\mathcal E_2$ denote the event on which \eqref{eqn::secnuis_corrected_pac} holds with confidence level $3\eta/4$.
By the preceding steps and a union bound (adjusting constants if needed), we may assume that
\[
\Pr(\mathcal E_1\cap \mathcal E_2)\ \ge\ 1-\eta.
\]
Work on $\mathcal E_1\cap \mathcal E_2$. Combining \eqref{eqn::basicnuis}, \eqref{eqn::firstnuis}, and \eqref{eqn::secnuis_corrected_pac}, and using the strong convexity condition \eqref{cond::strongconvexnuis} (so that $\kappa\|\hat f_n-\hat f_0\|^2\le \Reg(\hat f_n;\hat g)$), we obtain
\begin{align}
\label{eqn::step3_start}
\kappa\|\hat f_n-\hat f_0\|^2
&\lesssim
\|\hat f_n-\hat f_0\|\,\delta_{n,\mathcal F}
+\delta_{n,\mathcal F}^2
\nonumber\\
&\quad+
\|\mathcal G\|_\infty\,
\delta_{n,\mathcal F}\,
\Bigl(\|\hat f_n-\hat f_0\|\vee \delta_{n,\mathcal F}\Bigr)
\nonumber\\
&\quad+
c_\infty\,
\|\hat f_n-\hat f_0\|^{\alpha}\,
\delta_{n,\mathcal G}\,
\Bigl(\varepsilon_{\mathrm{nuis}}(n,\eta)\vee \delta_{n,\mathcal G}\Bigr)
\nonumber\\
&\quad+
\|\mathcal G\|_\infty\,
\|\hat f_n-\hat f_0\|\,
\sqrt{\frac{\log\log(eMn)+\log(1/\eta)}{n}}
+
c_\infty\,
\|\hat f_n-\hat f_0\|^{\alpha}\,\|\mathcal G\|_\infty\,
\frac{\log\log(eMn)+\log(1/\eta)}{n},
\end{align}
where $M:=1+\|\mathcal G\|_\infty\vee \|\mathcal F\|_\infty$.

Next, invoke the assumed lower bounds on the critical radii (and the definition of the high-probability radii in the theorem statement) to absorb the explicit logarithmic remainder terms. In particular, under the same type of lower bound used in Step~1,
\[
\delta_{n,\mathcal F}
\ \gtrsim\
(1 \vee M)\sqrt{\frac{\log(1/\eta)+\log\log(eMn)}{n}},
\qquad
\delta_{n,\mathcal G}
\ \gtrsim\
\sqrt{\frac{\log(1/\eta)+\log\log(eMn)}{n}},
\]
we have, after enlarging the implicit constants,
\[
\sqrt{\frac{\log\log(eMn)+\log(1/\eta)}{n}}
\ \lesssim\
\delta_{n,\mathcal F},
\qquad
\frac{\log\log(eMn)+\log(1/\eta)}{n}
\ \lesssim\
\delta_{n,\mathcal F}^2.
\]
Substituting these bounds into \eqref{eqn::step3_start} yields
\begin{align}
\label{eqn::step3_simplified}
\kappa\|\hat f_n-\hat f_0\|^2
&\lesssim
\delta_{n,\mathcal F}\,\|\hat f_n-\hat f_0\|
+\delta_{n,\mathcal F}^2
\nonumber\\
&\quad+
\|\mathcal G\|_\infty\,
\delta_{n,\mathcal F}\,
\Bigl(\|\hat f_n-\hat f_0\|\vee \delta_{n,\mathcal F}\Bigr)
\nonumber\\
&\quad+
c_\infty\,
\|\hat f_n-\hat f_0\|^{\alpha}\,
\delta_{n,\mathcal G}\,
\Bigl(\varepsilon_{\mathrm{nuis}}(n,\eta)\vee \delta_{n,\mathcal G}\Bigr).
\end{align}

We now use the above fixed-point inequality to extract a rate bound. Let $\Delta:=\|\hat f_n-\hat f_0\|$, $\delta_F:=\delta_{n,\mathcal F}$, and $\delta_G:=\delta_{n,\mathcal G}$, and define $\Delta_G:=\varepsilon_{\mathrm{nuis}}(n,\eta)\vee \delta_G$.
From \eqref{eqn::step3_simplified} and the bound $\Delta\vee \delta_F\le \Delta+\delta_F$, we obtain
\[
\kappa \Delta^2
\ \lesssim\
(1+\|\mathcal G\|_\infty)\bigl(\delta_F\Delta+\delta_F^2\bigr)
+c_\infty\,\Delta^\alpha\,\delta_G\,\Delta_G.
\]
Apply Young's inequality to the linear term:
\[
\delta_F\Delta \ \le\ \tfrac{\kappa}{4}\Delta^2 + C\,\delta_F^2,
\]
and absorb $\tfrac{\kappa}{4}\Delta^2$ into the left-hand side. This yields
\begin{equation}
\label{eq:delta_basic_short}
\Delta^2 \ \lesssim\ \delta_F^2 + \Delta^\alpha\,\delta_G\,\Delta_G .
\end{equation}
If $\alpha\in(0,1]$, apply Young's inequality with exponents $p=2/\alpha$ and $q=2/(2-\alpha)$ to the product $\Delta^\alpha\cdot(\delta_G\Delta_G)$:
\[
\Delta^\alpha(\delta_G\Delta_G)
\ \le\
\varepsilon\,\Delta^2
+ C_\alpha\,\varepsilon^{-\alpha/(2-\alpha)}(\delta_G\Delta_G)^{2/(2-\alpha)}.
\]
Choosing $\varepsilon>0$ sufficiently small and absorbing $\varepsilon\Delta^2$ into the left-hand side of \eqref{eq:delta_basic_short} gives
\[
\Delta^2
\ \lesssim\
\delta_F^2
+\bigl(\delta_G\Delta_G\bigr)^{2/(2-\alpha)},
\qquad \alpha\in(0,1).
\]
The endpoint $\alpha=0$ follows directly from \eqref{eq:delta_basic_short}, yielding
\[
\Delta^2\ \lesssim\ \delta_F^2+\delta_G\Delta_G,
\qquad \alpha=0.
\]
The result now follows recalling that $\alpha := 1 - 1/(2\beta)$.

\end{proof}

\end{document}